\documentclass{article}

 \usepackage[final]{neurips_2019}

\usepackage[utf8]{inputenc} 
\usepackage[T1]{fontenc}    
\usepackage{hyperref}       
\usepackage{url}            
\usepackage{booktabs}       
\usepackage{amsfonts}       
\usepackage{nicefrac}       
\usepackage{microtype}      

\usepackage{graphicx} 
\usepackage{caption}
\usepackage{subcaption}
\usepackage{slashbox}
\usepackage{enumerate}
\usepackage{enumitem}

\usepackage{natbib}
\usepackage{amsmath}
\usepackage{amsthm}
\usepackage{amssymb}
\usepackage{tikz}
\usepackage{xcolor}
\usetikzlibrary{arrows}
\usepackage{hyperref}
\hypersetup{colorlinks,
	linkcolor=blue,
	citecolor=blue,
	urlcolor=magenta,
	linktocpage,
	plainpages=false}
\usepackage{booktabs} 
\allowdisplaybreaks

\usepackage{mathrsfs}

\usepackage{algorithm}
\usepackage[noend]{algpseudocode}
\usepackage{hyperref}
\usepackage{bm,todonotes}

\usepackage{comment}

\usepackage{thmtools}
\usepackage{thm-restate}
\declaretheorem[name=Theorem,numberwithin=section]{thm}

\newcommand{\argmin}{\mathrm{argmin}}

\newcommand{\poly}{\mathrm{poly}}
\newcommand{\rank}{\mathrm{rank}}
\newcommand{\conv}{*}

\newcommand{\diag}{\mathrm{diag}}

\newcommand{\twotwomat}{\mat{\Lambda}}
\newcommand{\twotwog}{\mat{G}}

\renewcommand{\d}{\mathrm{d}}

\newcommand{\tr}{\mathrm{tr}}
\newcommand{\trace}[1]{\mathrm{tr}\left(#1\right)}

\def\R{\mathbb{R}}

\def\cA{\mathcal{A}}
\def\cB{\mathcal{B}}

\newcommand{\mat}[1]{\bm{#1}}
\newcommand{\vect}[1]{\bm{#1}}
\newcommand{\norm}[1]{\left\|#1\right\|}
\newcommand{\abs}[1]{\left|#1\right|}
\DeclareMathOperator*{\expect}{\mathbb{E}}

\newcommand{\pr}[1]{\mathbb{P}\left[#1\right]}

\newcommand{\Exp}[1]{\expect\left[#1\right]}
\newcommand{\Ex}[2]{\expect_{#1}\left[#2\right]}

\newcommand{\diff}{\triangle}

\newcommand{\indict}{\mathbf{1}}

\newcommand{\eps}{\epsilon}

\newcommand{\params}{\vect{\theta}}
\newcommand{\simiid}{\overset{\text{i.i.d.}}{\sim}}
\newcommand{\f}[2]{\vect f^{(#1)}( #2)}
\newcommand{\g}[2]{\vect g^{(#1)}( #2)}

\newcommand{\relu}[1]{\sigma\left(#1\right)}
\newcommand{\reluder}[1]{\sigma'\left(#1\right)}
\newcommand{\act}[1]{\sigma\left(#1\right)}
\newcommand{\deract}[1]{\dot{\sigma}\left(#1\right)}

\newcommand{\linop}{\mathcal{L}}
\newcommand{\trainker}{\mat{H}}
\newcommand{\paramsker}{\vect{\beta}}

\newcommand{\nnw}{{P}}
\newcommand{\nnh}{{Q}}
\newcommand{\nnc}{{C}}
\newcommand{\gauss}{\mathcal{N}}
\newcommand{\indset}{\mathcal{D}}

\newcommand{\back}{\mathbf{b}}
\newcommand{\T}[1]{\mathcal{T}_{#1}}
\newcommand{\K}[1]{\left.\mat K\right|_{\vect{#1}}}
\newcommand{\kernel}{\mathrm{ker}}
\newcommand{\wdist}{\mathcal{W}}

\newcommand{\GoodA}[4]{\overline{\mathcal{A}}^{#4}\left( #3\right)}
\newcommand{\GoodB}[4]{\overline{\mathcal{B}}^{#4}\left(  #3\right)}
\newcommand{\GoodC}[4]{\overline{\mathcal{C}}^{#4}\left( #3\right)}
\newcommand{\GoodD}[4]{\overline{\mathcal{D}}^{#4}\left(  #3\right)}
\newcommand{\fullgoodA}[4]{\mathcal{A}^{#4}\left( #1,  #2, #3\right)}
\newcommand{\fullgoodB}[4]{\mathcal{B}^{#4}\left( #1,  #2, #3\right)}

\newcommand{\fullgoodD}[4]{\mathcal{D}^{#4}\left( #1,  #2, #3\right)}
\newcommand{\fullGoodA}[4]{\overline{\mathcal{A}}^{#4}\left( #1,  #2, #3\right)}
\newcommand{\fullGoodB}[4]{\overline{\mathcal{B}}^{#4}\left( #1,  #2, #3\right)}
\newcommand{\fullGoodC}[4]{\overline{\mathcal{C}}^{#4}\left( #1,  #2, #3\right)}
\newcommand{\fullGoodD}[4]{\overline{\mathcal{D}}^{#4}\left( #1,  #2, #3\right)}
\newcommand{\xa}{\vect x}
\newcommand{\xb}{\vect x'}
\newcommand{\xone}{\vect x^{(1)}}
\newcommand{\xtwo}{\vect x^{(2)}}
\newcommand{\vx}{\vect x}
\newcommand{\vxx}{\vect x'}
\newcommand{\htt}{\hat{t}}
\newcommand{\htr}[2]{\htt_{#1}\left(#2\right)}
\newcommand{\trho}[2]{t_{#1}\left(#2\right)}
\newcommand{\id}{\mat{I}}
\newcommand{\Sum}[1]{\textrm{Sum}\left(#1\right)}
\newtheorem{lem}{Lemma}[section]

\newtheorem{defn}{Definition}[section]
\newtheorem{fact}{Fact}[section]

\newtheorem{rem}{Remark}[section]

\newtheorem{clm}{Claim}[section]

\DeclareMathOperator{\deq}{\overset{\textrm{d}}{=\joinrel=}}

\newenvironment{itemize*}%
{\begin{itemize}[leftmargin=*,topsep=0pt]%
		\setlength{\itemsep}{0pt}%
		\setlength{\parskip}{0pt}}%
	{\end{itemize}}
\newenvironment{enumerate*}%
{\begin{enumerate}[leftmargin=*,topsep=0pt]%
		\setlength{\itemsep}{0pt}%
		\setlength{\parskip}{0pt}}%
	{\end{enumerate}}

\title{
	On Exact Computation with an Infinitely Wide Neural Net\thanks{Code available: \texttt{https://github.com/ruosongwang/cntk}}}

\author{Sanjeev Arora\thanks{Princeton University and Institute for Advanced Study. Email: \texttt{arora@cs.princeton.edu}}
	\And 
	Simon S. Du\thanks{Institute for Advanced Study. Email: \texttt{ssdu@ias.edu}}
	\And
	Wei Hu\thanks{Princeton University. Email: \texttt{huwei@cs.princeton.edu}}
	\And
	Zhiyuan Li\thanks{Princeton University. Email: \texttt{zhiyuanli@cs.princeton.edu}}
	\AND
	Ruslan Salakhutdinov\thanks{Carnegie Mellon University. Email:\texttt{rsalakhu@cs.cmu.edu}}
	\And	Ruosong Wang\thanks{Carnegie Mellon University. Email: \texttt{ruosongw@andrew.cmu.edu}}
}

\begin{document}

\maketitle
\begin{abstract}
	How well does a classic deep net architecture like AlexNet or VGG19 classify on a standard dataset such as CIFAR-10 when its \textquotedblleft width\textquotedblright --- namely, number of channels in convolutional layers, and number of nodes in fully-connected internal layers --- is allowed to increase to infinity? 
Such questions have come to the forefront in the quest to theoretically understand deep learning and its mysteries about optimization and generalization. They also connect deep learning to notions such as {\em Gaussian processes} and \emph{kernels}. 
A recent paper \citep{jacot2018neural} introduced the {\em Neural Tangent Kernel (NTK)} which captures the behavior of fully-connected deep nets in the infinite width limit trained by gradient descent; this object was implicit in some other recent papers. 
An attraction of such ideas is that  a pure kernel-based method is used to capture the power of a fully-trained deep net of infinite width. 

The current paper gives the first efficient exact algorithm for computing the extension of NTK to convolutional neural nets, which we call \emph{Convolutional NTK (CNTK)}, as well as an efficient GPU implementation of this algorithm.  
This results in a significant new benchmark for performance of a pure kernel-based method on CIFAR-10, being  $10\%$ higher than the methods reported in \citep{novak2019bayesian}, and  only $6\%$ lower than the performance of the corresponding finite deep net architecture (once batch normalization etc. are turned off).
Theoretically, we also give the first \emph{non-asymptotic} proof
 showing that a fully-trained sufficiently wide net is indeed equivalent to the kernel regression predictor using NTK.

\end{abstract}

\section{Introduction}
\label{sec:intro}

How well does a classic deep net architecture like AlexNet or VGG19 perform on a standard dataset such as CIFAR-10 when its \textquotedblleft width\textquotedblright --- namely, number of channels in convolutional layers, and number of nodes in fully-connected internal layers --- is allowed to increase to infinity? Questions about these ``infinite limits''  of deep nets have naturally emerged in the ongoing effort to understand the power of deep learning. In mathematics it is often easier to study objects in the infinite limit. Furthermore, the infinite limit could conceivably make sense in deep learning, since  {\em over-parametrization} seems to help optimization a lot and doesn't hurt generalization much~\citep{zhang2016understanding}: deep neural nets  with millions of parameters  work well even for datasets with $50$k training examples. 
So why not imagine nets whose width goes to infinity?

 Allowing width to go to infinity also connects deep learning in an interesting way with other areas of machine learning. 
  A single hidden-layer neural network with i.i.d. random parameters, in the limit of infinite width, is a function drawn from a \emph{Gaussian process (GP)}~\citep{neal1996priors}. 
  This model as well as analogous ones with multiple layers~\citep{lee2018deep,matthews2018gaussian}, convolutional filters~\citep{novak2019bayesian,garriga-alonso2018deep} and other architectures~\citep{yang2019scaling} make up the GP view of deep learning. These correspond to infinitely wide deep nets whose all parameters are chosen randomly (with careful scaling),  and only the top (classification) layer is optimized. 
 
From now on we will use {\em weakly-trained nets} to  refer to nets whose layers receive random initialization and only the top layer is trained by gradient descent. 
We use {\em fully-trained} to refer to nets whose all parameters are trained by gradient descent.
It has long been known that weakly-trained convolutional nets have reasonable performance on MNIST and CIFAR-10. 
Weakly-trained nets that are fully-connected instead of convolutional, can also be thought of as ``multi-layer  random kitchen sinks,'' which also have a long history. 

Weakly-trained nets --- whether of finite or infinite width --- also define interesting kernels. Specifically, if $f(\vect{\params},\vect{x}) \in \R$ denotes the output of the network on input $\vect{x}$ where
$\vect{\params}$ denotes the parameters in the network,  and $\wdist$ is an initialization distribution over $\vect{\params}$ (usually Gaussian), then training just the top layer with an $\ell_2$ loss is equivalent to kernel regression for the following kernel:
 \begin{align}
 \kernel \left(\vect{x},\vect{x}'\right) = \expect_{\vect{\params} \sim \wdist}
 [f\left(\vect{\params},\vect{x}\right)\cdot f\left(\vect{\params},\vect{x}'\right)],
 \label{eqn:covkernel}
 \end{align}
 where $\vect{x},\vect{x}'$ are two inputs. This kernel method makes sense when the width goes to infinity.

The objects of interest in this paper are not weakly-trained nets, but fully-trained nets. In the finite case, analysis of optimization and generalization of fully-trained nets is of course an open problem. One may also ask:
 \begin{center} 
{\em Can we understand the power of fully-trained nets whose width goes to infinity?}
\end{center} 
 {\em A priori} this question doesn't seem any easier than the finite case,  and  empirical evaluation seems computationally infeasible due to the infinite limit. They also do not correspond to a kernel method in any obvious way. 
 
   Recent papers  suggest that neural nets whose width greatly exceeds the number of training data points can rapidly reduce training error to $0$ via gradient descent, and under some conditions, the trained net also exhibits good generalization~\citep{du2018provably,du2018global,li2018learning,allen2018learning,allen2018convergence,zou2018stochastic,arora2019fine,cao2019generalization}. 
   Extra-wideness plays a crucial role in the proof: it is shown that as width increases, training causes increasingly smaller changes (in a proportionate sense) in the  parameters. This raises the possibility that as one increases the width to infinity, a certain limiting behavior can emerge even in the fully-trained net. A recent paper by \cite{jacot2018neural} isolated a notion implicit in the above papers, which they called the {\em Neural Tangent Kernel (NTK)}. They suggested --- via a proof that is slightly heuristic --- that this fixed kernel  characterizes the  behavior of fully-connected infinite width neural networks whose layers have been trained by gradient descent.     The NTK is different from the Gaussian process kernels discussed earlier, and is defined using the {\em gradient} of the output of the randomly initialized net with respect to its parameters, i.e.,
   \begin{align}
   \kernel\left(\vect{x},\vect{x}'\right) = \expect_{\params \sim \wdist}\left\langle \frac{\partial f(\params,\vect{x})}{\partial \params}, \frac{\partial f(\params,\vect{x}')}{\partial \params}\right\rangle . \label{eqn:ntk}
   \end{align}
   Here, the gradient $\frac{\partial f(\params, \vect{x})}{\partial \params}$ appears from considering gradient descent, as will be explained in Section~\ref{sec:ntk-recap}.
    One may also generalize the NTK to convolutional neural nets, and we call the corresponding kernel \emph{Convolutional Neural Tangent Kernel (CNTK)}.
    
   Though NTK and CNTK are defined by an infinite limit, a recent paper~\citep{lee2019wide} attempted to understand their properties via a finite approximation of the infinite limit kernel by Monte Carlo methods.
   However, as will be shown in Section~\ref{sec:exp_rf}, using random features generated from practically sized nets can degrade the performance a lot.
   It was still open what is the full power of \emph{exact}  CNTK on modern datasets.
   This is a challenging question especially for CNTK with pooling operations, since when convolution with pooling is involved, it was believed that exact computation of kernels (for either convolutional Gaussian process kernel or CNTK) is infeasible for large datasets like CIFAR-10~\citep{novak2019bayesian}.

\paragraph{Our contributions.} We give an exact and efficient dynamic programming algorithm to compute CNTKs for ReLU activation (namely, to compute $\kernel\left(\vect{x},\vect{x}'\right)$ given $\vect{x}$ and $\vect{x}'$). Using this algorithm --- as well as implementation tricks for GPUs --- we can settle the question of the performance of fully-trained infinitely wide nets with a variety of architectures. For instance, we find that their performance on CIFAR-10 is within $5\%$ of the performance of the same architectures in the finite case (note that the proper comparison in the finite case involves turning off batch norm, data augmentation, etc., in the optimization). 
In particular, the CNTK corresponding to a $11$-layer convolutional net with global average pooling achieves $77\%$ classification accuracy.
This is $10\%$ higher than the best reported performance of a Gaussian process with fixed kernel on CIFAR-10~\citep{novak2019bayesian}.\footnote{We only consider fixed kernels defined without using the training data. 
We do not compare to methods that tune the kernels using training data~\citep{van2017convolutional} or use a neural network to extract features and then applying a kernel method on top of them~\citep{mairal2014convolutional}.}

Furthermore,
we give a more rigorous, non-asymptotic proof that the NTK captures the behavior of a fully-trained wide neural net under weaker condition than previous proofs. 
We also experimentally show that the random feature methods for approximating CNTK in earlier work do not compute good approximations, which is clear from their much worse performance on CIFAR.

\subsection{Notation}
\label{sec:notations}
We use bold-faced letters for vectors, matrices and tensors. 
For a vector $\vect a$, let $[\vect a]_i$ be its $i$-th entry;
for a matrix $\mat A$, let $\left[\mat A\right]_{i,j}$ be its $(i, j)$-th entry;
for a 4th-order tensor $\mat{T}$, let $\mat \left[\mat A\right]_{ij,i'j'}$ be its $(i,j,i',j')$-th entry.
Let $\mat I$ be the identity matrix, and $[n]=\{1, 2, \ldots, n\}$.
Let $\vect{e}_{i}$ be an indicator vector with $i$-th entry being $1$ and other entries being $0$, and let $\vect{1}$ denote the all-one vector.
We use $\odot$ to denote the entry-wise product and $\otimes$ to denote the tensor product.
We use $\langle \cdot,\cdot \rangle$ to denote the standard inner product.
We use $\diag(\cdot)$ to transform a vector to a diagonal matrix.
We use $\relu{\cdot}$ to denote the activation function, such as the rectified linear unit (ReLU) function: $\relu{z} = \max\{z, 0\}$, and $\dot{\sigma}\left(\cdot\right)$ to denote the derivative of $\relu{\cdot}$.
Denote by $\gauss(\bm\mu, \mat\Sigma)$ the Gaussian distribution with mean $\bm\mu$ and covariance $\mat{\Sigma}$.

\section{Related Work}
\label{sec:rel}

From a Gaussian process (GP) viewpoint, the correspondence between infinite neural networks and kernel machines was first noted by~\citet{neal1996priors}.
Follow-up work extended this correspondence to more general shallow neural networks~\citep{williams1997computing,leroux07a,hazan2015steps}.
More recently, this was extended to deep and convolutional neural networks~\citep{lee2018deep,matthews2018gaussian,novak2019bayesian,garriga-alonso2018deep} and a variety of other architectures~\citep{yang2019scaling}.
However, these kernels, as we discussed in Section~\ref{sec:intro}, represent weakly-trained nets, instead of fully-trained nets.

Beyond GPs, the connection between neural networks and kernels is also studied in the compositional kernel literature.
\citet{cho2009kernel} derived a closed-form kernel formula for rectified polynomial activations, which include ReLU as a special case.
\citet{daniely2016toward} proposed a general framework to transform a neural network to a compositional kernel and later \citet{daniely2017sgd} showed for sufficiently wide neural networks, stochastic gradient descent can learn  functions that lie in the corresponding reproducing kernel Hilbert space.
However, the kernels studied in these works still correspond to weakly-trained neural networks.

This paper is inspired by a line of recent work on over-parameterized neural networks~\citep{du2018provably,du2018global,du2019width,li2018learning,allen2018convergence,allen2018learning,zou2018stochastic,cao2019generalization}.
These papers established that for (convolutional) neural networks with large but finite width, (stochastic) gradient descent can achieve zero training error.
A key component in these papers is showing that the weight matrix at each layer is close to its initialization.
This observation implies that the kernel defined in Equation~\eqref{eqn:ntk} is still close to its initialization.
\citet{arora2019fine} explicitly used this observation to derive generalization bounds for two-layer over-parameterized neural networks.
\citet{chizat2018note} argued that these results in the kernel regime may be too simple to be able to explain the success of deep learning, while on the other hand, out results show that CNTK is at least able to perform well on tasks like CIFAR-10 classification.
Also see the survey~\cite{fan2019selective} for recent advance in deep learning theory.

\citet{jacot2018neural} derived the exact same kernel from kernel gradient descent.
They showed that if the number of neurons per layer goes to infinity in a sequential order, then the kernel remains unchanged for a finite training time.
They termed the derived kernel \emph{Neural Tangent Kernel (NTK)}.
We follow the same naming convention and name its convolutional extension \emph{Convolutional Neural Tangent Kernel (CNTK)}.
Later, \citet{yang2019scaling} derived a formula of CNTK as well as a mechanistic way to derive NTK for different architectures.
Comparing with \citep{yang2019scaling}, our CNTK formula has a more explicit convolutional structure and results in an efficient GPU-friendly computation method.
Recently, \citet{lee2019wide} tried to empirically verify the theory in \citep{jacot2018neural} by studying the linearization of neural nets.
They observed that in the first few iterations, the linearization is close to the actual neural net.
However, as will be shown in Section~\ref{sec:exp_rf}, such linearization can decrease the classification accuracy by $5\%$ even on a  ``CIFAR-2" (airplane V.S. car) dataset.
Therefore, exact kernel evaluation is important to study the power of NTK and CNTK.

\section{Neural Tangent Kernel}
\label{sec:ntk-recap}
In this section we describe fully-connected deep neural net architecture and its infinite width limit, and how training it with respect to the $\ell_2$ loss gives rise to a kernel regression problem involving the neural tangent kernel (NTK). 
We denote by $f(\params, \vect x) \in \R$ the output of a neural network where $\params \in \R^N$ is all the parameters in the network and $\vect x \in \R^d$ is the input.\footnote{For simplicity, we only consider a single output here. The generalization to multiple outputs is straightforward.}
Given a training dataset $\left\{(\vect x_i, y_i)\right\}_{i=1}^n \subset \R^d \times \R$, consider training the neural network by minimizing the squared loss over training data:
$
\ell(\params) = \frac12 \sum_{i=1}^n \left( f(\params, \vect x_i) - y_i \right)^2.
$
The proof of the following lemma uses simple differentiation and appears in Section~\ref{appsec:omittedpfs}.

\begin{lem}\label{lem:dynamics}
	Consider minimizing the squared loss $\ell(\params)$ by gradient descent with infinitesimally small learning rate: $\frac{\d\params(t)}{\d t} = -\nabla\ell(\params(t)) $.
Let $\vect u(t) = \left( f(\params(t), \vect x_i) \right)_{i\in[n]} \in \R^n $ be the network outputs on all $\vect x_i$'s at time $t$, and $\vect y = (y_i)_{i\in[n]}$ be the desired outputs. Then $\vect u(t)$ follows the following evolution, where $\mat H(t)$ is an $n\times n$ positive semidefinite matrix whose $(i, j)$-th entry is $\left\langle \frac{\partial f(\params(t), \vect x_i)}{\partial \params}, \frac{\partial f(\params(t), \vect x_j)}{\partial \params} \right\rangle$: 
\begin{equation} \label{eqn:output-dynamics-lemma}
\frac{\d \vect u(t)}{\d t} = - \mat H(t) \cdot (\vect u(t) - \vect y).
\end{equation}
\end{lem} 

The statement of Lemma~\ref{lem:dynamics} involves a matrix $\mat H(t)$. Below we define a deep net architecture whose width is allowed to go to infinity, while fixing the training data as above. In the limit, it can be shown that the matrix $\mat H(t)$ remains \emph{constant} during training i.e., equal to $\mat H(0)$.
Moreover, under a random initialization of parameters, the random matrix $\mat H(0)$ converges in probability to a certain deterministic kernel matrix $\mat H^*$ as the width goes to infinity, which is the {\em Neural Tangent Kernel} $\kernel(\cdot, \cdot)$ (Equation~\eqref{eqn:ntk}) evaluated on the training data. 
If $\mat H(t) = \mat H^*$ for all $t$, then Equation~\eqref{eqn:output-dynamics-lemma} becomes 
\begin{equation} \label{eqn:linear-dynamics}
\frac{\d \vect u(t)}{\d t} = - \mat H^* \cdot (\vect u(t) - \vect y).
\end{equation}
Note that the above dynamics is identical to the dynamics of \emph{kernel regression} under gradient flow, for which at time $t\to\infty$ the final prediction function is (assuming $\vect u(0)=\vect0$)
\begin{equation} \label{eqn:kernel-solution}
f^*(\vect x) = \left( \kernel(\vect x, \vect x_1), \ldots, \kernel(\vect x, \vect x_n) \right) \cdot (\mat H^*)^{-1} \vect y.
\end{equation}
In Theorem~\ref{thm:ntk_main}, we rigorously prove that a fully-trained sufficiently wide ReLU neural network is equivalent to the kernel regression predictor~\eqref{eqn:kernel-solution} on any given data point.

\paragraph{Fully-connected deep neural net and its infinite width limit.}

Now we define a fully-connected neural net formally.
Let $\vect x \in \R^d$ be the input, and denote $\vect g^{(0)}(\vect x) = \vect x$ and $d_0=d$ for notational convenience. We define an $L$-hidden-layer fully-connected neural network recursively:
\begin{equation} \label{eqn:fc-nn}
\begin{aligned}
\vect f^{(h)}(\vect x) = \mat W^{(h)}  \vect g^{(h-1)}(\vect x) \in \R^{d_h}, \quad \vect g^{(h)}(\vect x) = \sqrt{\frac{c_\sigma}{d_h}} \sigma\left( \vect f^{(h)}(\vect x)  \right) \in \R^{d_h}, \qquad h=1, 2, \ldots, L,
\end{aligned}
\end{equation}
where $\mat W^{(h)} \in \R^{d_h \times d_{h-1}}$ is the weight matrix in the $h$-th layer ($h\in[L]$), $\sigma: \R\to\R$ is a coordinate-wise activation function, and $c_{\sigma} = \left(\expect_{z \sim \gauss(0,1)}\left[\relu{z}^2\right]\right)^{-1}$.
The last layer of the neural network is
\begin{align*}
f(\params, \vect x) &= f^{(L+1)}(\vect x) = \mat W^{(L+1)} \cdot \vect g^{(L)}(\vect x) \\
&= \mat{W}^{(L+1)} \cdot \sqrt{\frac{c_\sigma}{d_L}}  \relu{\mat{W}^{(L)}\cdot \sqrt{\frac{c_{\sigma}}{d_{L-1}}}\relu{\mat{W}^{(L-1)}\cdots\sqrt{\frac{c_{\sigma}}{d_1}}\relu{\mat{W}^{(1)}\vect{x}}}},
\end{align*}
where $\mat W^{(L+1)} \in \R^{1\times d_L}$ is the weights in the final layer, and $\params = \left( \mat W^{(1)}, \ldots, \mat W^{(L+1)} \right)$ represents all the parameters in the network.

We initialize all the weights to be i.i.d. $\gauss(0, 1)$ random variables, and consider the limit of large hidden widths: $d_1, d_2, \ldots, d_L \to \infty$.
The scaling factor $\sqrt{c_{\sigma}/d_h}$ in Equation~\eqref{eqn:fc-nn} ensures that the norm of $\vect g^{(h)}(\vect x)$ for each $h\in[L]$ is approximately preserved at initialization (see~\citep{du2018global}). In particular, for ReLU activation, we have $\expect\left[ \norm{\vect g^{(h)}(\vect x)}^2 \right] = \norm{\vect x}^2$ ($\forall h\in[L]$).

Recall from \citep{lee2018deep} that in the infinite width limit, the pre-activations $\vect f^{(h)}(\vect x)$ at every hidden layer $h\in[L]$ has all its coordinates tending to i.i.d. centered Gaussian processes of covariance $\Sigma^{(h-1)}: \R^d\times \R^d \to \R$ defined recursively as: for $h \in [L]$,
\begin{equation} \label{eqn:gp-cov-kernel}
\begin{aligned}
\Sigma^{(0)}(\vect{x},\vect{x}') &= \vect{x}^\top \vect{x}',  \\
\mat{\twotwomat}^{(h)}(\vect{x},\vect{x}') &=  \begin{pmatrix}
\Sigma^{(h-1)}(\vect{x},\vect{x}) & \Sigma^{(h-1)}(\vect{x},\vect{x}')\\
\Sigma^{(h-1)}(\vect{x}',\vect{x}) & \Sigma^{(h-1)}(\vect{x}',\vect{x}') 
\end{pmatrix}\in \R^{2\times2},  \\ 
\Sigma^{(h)}(\vect{x},\vect{x}') &=  c_{\sigma}\expect_{(u,v) \sim \gauss\left(\vect{0},\mat{\twotwomat}^{(h)}\right) }\left[\relu{u}\relu{v}\right].
\end{aligned}
\end{equation}
To give the formula of NTK, we also need to  define a derivative covariance:
\begin{align} \label{eqn:gradient_kernel}
\dot{\Sigma}^{(h)}(\vect{x},\vect{x}') =  c_{\sigma}\expect_{(u,v) \sim \gauss\left(\vect{0},\mat{\twotwomat}^{(h)}\right)}\left[\dot{\sigma}(u)\dot{\sigma}(v)\right]. 
\end{align}
The final NTK expression for the fully-connected neural network is 
 \begin{align}\label{eqn:fc_ntk}
 \Theta^{(L)}(\vect{x},\vect{x}') = \sum_{h=1}^{L+1} \left(
 \Sigma^{(h-1)}(\vect{x},\vect{x}') \cdot \prod_{h'=h}^{L+1} \dot{\Sigma}^{(h')}(\vect{x},\vect{x}')
 \right), 
 \end{align}
 where we let $\dot{\Sigma}^{(L+1)}(\vect{x},\vect{x}')=1$ for convenience.
We refer readers to Section~\ref{sec:ntk_der} for the derivation of this formula. 
Rigorously, for ReLU activation, we have the following theorem that gives a concrete bound on the hidden widths that is sufficient for convergence to the NTK at initialization:
\begin{restatable}[Convergence to the NTK at initializatoin]{thm}{ntkinit}
	\label{thm:ntk_init}
	Fix $\eps > 0$ and $\delta \in (0, 1)$.
	Suppose $\relu{z} = \max(0,z)$ and $\min_{h\in[L]} d_h \ge\Omega(\frac{L^{6}}{\eps^4} \log({L}/{\delta}))$. 
	Then for any inputs $\xa,\xb\in \R^{d_0}$ such that $\norm{\xa}\le 1, \norm{\xb}\le 1$, with probability at least $1-\delta$ we have:
	\begin{equation*}
	\left| 
	\left\langle \frac{\partial f(\params, \vect x)}{\partial \params}, \frac{\partial f(\params, \vect x')}{\partial \params} \right\rangle
	-\Theta^{(L)}(\vect{x},\vect{x}') \right| \le (L+1)\eps.
	\end{equation*}
\end{restatable}

The proof of Theorem~\ref{thm:ntk_init} is given in Section~\ref{appsec:ntk_init}.
	Theorem~\ref{thm:ntk_init} improves upon previous results \citep{jacot2018neural, yang2019scaling} that also established similar convergence in the following sense:
	\begin{enumerate*}
		\item Previous results are asymptotic, i.e., they require the widths to go to infinity, while Theorem~\ref{thm:ntk_init} gives a non-asymptotic bound on the required layer widths.
		\item \cite{jacot2018neural} required sequential limit, i.e., $d_1, \ldots, d_L$ go to infinity one by one, and \cite{yang2019scaling} let $d_1, \ldots, d_L$ go to infinity at the same rate. On the other hand, Theorem~\ref{thm:ntk_init} only requires $\min_{h\in[L]} d_h$ to be sufficiently large, which is the weakest notion of limit.
	\end{enumerate*}

\paragraph{Equivalence between wide neural net and kernel regression with NTK.}

Built on Theorem~\ref{thm:ntk_init}, we can further incorporate the training process and show the equivalence between a fully-trained sufficiently wide neural net and the kernel regression solution using the NTK, as described in Lemma~\ref{lem:dynamics} and the discussion after it.

Recall that the training data are $\left\{(\vect x_i, y_i)\right\}_{i=1}^n \subset \R^d \times \R$, and $\trainker^* \in \mathbb{R}^{n \times n}$ is the NTK evaluated on these training data, i.e., $[\trainker^*]_{i, j} = \Theta^{(L)}(\vx_i, \vx_j)$. Denote $\lambda_0 = \lambda_{\min}\left(\trainker^*\right)$.
For a testing point $\vect{x}_{te} \in \mathbb{R}^d$, we let $\kernel_{ntk}(\vect{x}_{te},\mat{X}) \in \mathbb{R}^n$ be the kernel evaluated between the testing point and $n$ training points, i.e., $\left[\kernel_{ntk}(\vect{x}_{te},\mat{X})\right]_i = \Theta^{(L)}(\vx_{te}, \vx_i)$.
The prediction of kernel regression using NTK on this testing point is
$
f_{ntk}\left(\vect{x}_{te}\right) = \left( \kernel_{ntk}\left(\vect{x}_{te},\mat{X}\right) \right)^{\top} \left(\trainker^*\right)^{-1} \vect{y}.
$

Since the above solution corresponds to the linear dynamics in Equation~\eqref{eqn:linear-dynamics} with zero initialization,
in order to establish equivalence between neural network and kernel regression, we would like the initial output of the neural network to be small.
Therefore, we apply a small multiplier $\kappa>0$, and let the final output of the neural network be $f_{nn}(\params,\vect{x})= \kappa f\left(\params,\vect{x}\right).$
We let $f_{nn}(\vect{x}_{te}) = \lim_{t \rightarrow \infty} f_{nn}(\params(t),\vect{x}_{te})$ be the prediction of the neural network at the end of training.

The following theorem 
 establishes the equivalence between the fully-trained wide neural network $f_{nn}$ and the kernel regression predictor $f_{ntk}$ using the NTK.
\begin{thm}[Equivalence between trained net and kernel regression]
	\label{thm:ntk_main}
Suppose $\relu{z} = \max(0,z)$, $1/\kappa =  \poly(1/\epsilon, \log(n/\delta)) $ and $d_1=d_2=\cdots=d_L=m$ with $m \ge \poly(1/\kappa,L,1/\lambda_0,n,\log(1/\delta))$.
	Then for any $\vx_{te} \in \R^d$ with $\norm{\vx_{te}}=1$, with probability at least $1-\delta$ over the random initialization, we have \begin{align*}
	\abs{f_{nn}(\vect{x}_{te})-f_{ntk}(\vect{x}_{te})}\le \epsilon.
	\end{align*}
\end{thm}

Several comments are in sequel.
Theorem~\ref{thm:ntk_main} is, to our knowledge, the first result that rigorously shows the equivalence between a fully-trained neural net and a kernel predictor.
Comparing with \citep{jacot2018neural}, our bound is non-asymptotic whereas \citep{jacot2018neural} only has an asymptotic result;
furthermore, \citet{jacot2018neural} required the width of every layer to go to infinity in a sequential order, while we can have the same number of neurons per layer, which is closer to practice.
Comparing with recent results on over-parameterized neural nets~\citep{arora2019fine,allen2018convergence,allen2018learning,du2018provably,du2018global,li2018learning,zou2018stochastic}, our theorem is a more precise characterization of the learned neural network.
That is, the prediction is essentially a kernel predictor.
Therefore, to study the properties of these over-parameterized nets, such as their generalization power, it is sufficient to study the corresponding NTK.

While this theorem only gives guarantee for a single point, using a union bound, we can show that this guarantee holds for (exponentially many) finite testing points.
Combing this with the standard analysis of hold-out validation set, we can conclude that a fully-trained wide neural net enjoys the same generalization ability as its corresponding NTK.

For the proof of Theorem~\ref{thm:ntk_main}, we first use a generic argument to show that the perturbation on the prediction can be reduced to the perturbation on kernel value at the initialization and during training.  Theorem~\ref{thm:ntk_init} guarantees a small perturbation on kernel value at initialization.
For the perturbation during training, we use high level proof idea from \cite{du2018global,arora2019fine} to reduce the perturbation on the kernel value to the perturbation on the gradient of each prediction with respect to weight matrices.
Then we adopt technical lemmas from \cite{allen2018convergence} to obtain bounds on the perturbation of the gradient.
The proof of Theorem~\ref{thm:ntk_main} is given in Section~\ref{appsec:main_proof}.
We remark that \citet{jacot2018neural,lee2019wide} provided proofs for the training part.
However, both are asymptotic results and only apply to \emph{finite} training time.
In contrast, we give a finite-width perturbation bound and our result applies to \emph{infinite} training time.

\section{Convolutional Neural Tangent Kernel}
\label{sec:kernel}
In this section we study convolutional neural nets (CNNs) and their corresponding CNTKs.
We study two architectures, vanilla CNN and CNN with global average pooling (GAP).
In this section we define vanilla CNN and present its corresponding CNTK formula.
The derivation of this formula is deferred to Section~\ref{sec:cntk_derivation}.
We present the definition of CNN with GAP and its CNTK in Section~\ref{sec:cntk_gap}.

To formally define CNNs, we first introduce some notation.
We let $\nnw$ be the width and $\nnh$ be the height of the image.
We use $q \in \mathbb{Z}_+$ to denote the filter size. In practice, $q=1,3,$ or $5$. 
We use standard zero padding and set stride size to be $1$ to make sure the input of each layer has the same size.
For a convolutional filter $\vect{w} \in \mathbb{R}^{q \times q}$ and an image $\vect{x} \in \mathbb{R}^{\nnw \times \nnh}$, the convolution operator is defined as
\begin{align}
[\vect{w}\conv \vect{x}]_{ij} = \sum_{a = -\frac{q-1}{2}}^{\frac{q-1}{2}} \sum_{b = -\frac{q-1}{2}}^{\frac{q-1}{2}}  \left[\mat{w}\right]_{a+\frac{q+1}{2},b+\frac{q+1}{2}} [\vect{x}]_{a+i,b+j}\text{ for }i \in [\nnw], j \in [\nnh]. \label{eqn:conv}
\end{align}
Equation~\eqref{eqn:conv} shows that patch $[\vect{w}\conv \vect{x}]_{ij}$ depends on $[\vect{x}]_{i-\frac{q-1}{2}:i+\frac{q-1}{2},  j-\frac{q-1}{2}: j+\frac{q-1}{2}}$.
Our CNTK formula  also relies on this dependency.
For $(i,j,i',j') \in [\nnw]\times [\nnh] \times [\nnw] \times [\nnh]$, define 
\begin{align*}
&\indset_{ij,i'j'}\\
 = &\left\{(i+a,j+b,i'+a',j'+b') \in [\nnw]\times [\nnh] \times [\nnw] \times [\nnh] \mid -(q - 1) / 2 \le a, b, a', b' \le (q - 1) /2
\right\}.
\end{align*}
Lastly, for a tensor $\mat{T} \in \mathbb{R}^{\nnw \times \nnh \times \nnw \times \nnh}$, we denote by $\left[\mat{T}\right]_{\indset_{ij,i'j'}} \in \mathbb{R}^{q \times q \times q \times q}$ a sub-tensor and we let $\tr\left(\mat{T}\right) = \sum_{i,j}\mat{T}_{ij,ij}$.

A vanilla CNN consisting of $L$ convolution layers and one fully-connected layer is formally defined as follows:
\begin{itemize*}
	\item Let  $\vect{x}^{(0)} =\vect{x} \in \mathbb{R}^{\nnw\times \nnh \times \nnc^{(0)}}$ be the input image where $\nnc^{(0)}$ is the number of channels. 
	\item For $h=1,\ldots,L$, $\beta = 1,\ldots,\nnc^{(h)}$, the intermediate outputs are defined as \begin{align*}
	\tilde{\vect{x}}_{(\beta)}^{(h)} = \sum_{\alpha=1}^{\nnc^{(h-1)}} \mat{W}_{(\alpha),(\beta)}^{(h)} \conv \vect{x}_{(\alpha)}^{(h-1)} ,\quad
	\vect{x}^{(h)}_{(\beta)} = \sqrt{\frac{c_{\sigma}}{\nnc^{(h)} \times q \times q}}\act{\tilde{\vect{x}}_{(\beta)}^{(h)}},
	\end{align*}
	where each $\mat{W}_{(\alpha),(\beta)}^{(h)} \in \mathbb{R}^{q \times q}$ is a filter with standard Gaussian initialization.
	\item The final output is defined as $
	f(\params,\vect{x}) = \sum_{\alpha=1}^{\nnc^{(L)}} \left\langle \mat{W}_{(\alpha)}^{(L+1)},\vect{x}_{(\alpha)}^{(L)}\right\rangle
	$
	where $\mat{W}_{(\alpha)}^{(L+1)} \in \mathbb{R}^{\nnw \times \nnh}$ is a weight matrix with standard Gaussian initialization.
\end{itemize*}
For this architecture, using the same reasoning as in Section~\ref{sec:ntk_der}, we obtain the following convolutional neural tangent kernel formula.
The details are provided in Section~\ref{sec:cntk_derivation}.

\paragraph{CNTK formula.}
We let $\vect{x},\vect{x}'$ be two input images.

\begin{itemize*}
	\item For $\alpha = 1,\ldots,\nnc^{(0)}$, $(i,j,i',j') \in [\nnw] \times [\nnh] \times [\nnw] \times [\nnh]$, define \begin{align*}
	\mat{K}^{(0)}_{(\alpha)}\left(\vect{x},\vect{x}'\right) = \vect{x}_{(\alpha)} \otimes \vect{x}_{(\alpha)}' \text{ and }\left[\mat{\Sigma}^{(0)}(\vect{x},\vect{x}')\right]_{ij,i'j'} =  \sum_{\alpha=1}^{\nnc^{(0)}}\tr\left(\left[\mat{K}^{(0)}_{(\alpha)}(\vect{x},\vect{x}')\right]_{ \indset_{ij,i'j'}}\right).
	\end{align*}

	\item For $h \in [L]$, \begin{itemize*}
		\item For $(i,j,i',j') \in [\nnw] \times [\nnh] \times [\nnw] \times [\nnh]$, define \begin{align*}
		\mat{\twotwomat}_{ij,i'j'}^{(h)}(\vect{x},\vect{x}') = \begin{pmatrix}
		\left[\mat{\Sigma}^{(h-1)}(\vect{x},\vect{x})\right]_{ij,ij} & \left[\mat{\Sigma}^{(h-1)}(\vect{x},\vect{x}')\right]_{ij,i'j'} \\
		\left[\mat{\Sigma}^{(h-1)}\left(\vect{x}',\vect{x}\right)\right]_{i'j',ij}&
		\left[\mat{\Sigma}^{(h-1)}\left(\vect{x}',\vect{x}'\right)\right]_{i'j',i'j'}
		\end{pmatrix} \in \mathbb{R}^{2 \times 2}. 
		\end{align*}
		\item Define $\mat{K}^{(h)}(\vect{x},\vect{x}'),\dot{\mat{K}}^{(h)}(\vect{x},\vect{x}') \in \mathbb{R}^{\nnw \times \nnh \times \nnw \times \nnh}$: for  $(i,j,i',j') \in [\nnw] \times [\nnh] \times [\nnw] \times [\nnh]$,
		\begin{align}
		\left[\mat{K}^{(h)}(\vect{x},\vect{x}')\right]_{ij,i'j'} = & \frac{c_{\sigma}}{q^2} \cdot\expect_{(u,v)\sim \gauss\left(\vect{0},\mat{\twotwomat}_{ij,i'j'}^{(h)}(\vect{x},\vect{x}')\right)}\left[\act{u}\act{v}\right], \label{eqn:vanila_cnn_exp}\\
		\left[\dot{\mat{K}}^{(h)}(\vect{x},\vect{x}')\right]_{ij,i'j'} = & \frac{c_{\sigma}}{q^2} \cdot\expect_{(u,v)\sim \gauss\left(\vect{0},\mat{\twotwomat}_{ij,i'j'}^{(h)}(\vect{x},\vect{x}')\right)}\left[\deract{u}\deract{v}\right]. \label{eqn:vanila_cnn_exp_d}
		\end{align}
		\item Define $\mat{\Sigma}^{(h)}(\vect{x},\vect{x}') \in \mathbb{R}^{\nnw \times \nnh \times \nnw \times \nnh}$: for $(i,j,i',j') \in [\nnw] \times [\nnh] \times [\nnw] \times [\nnh]$,
		\begin{align*}
		\left[\mat{\Sigma}^{(h)}(\vect{x},\vect{x}')\right]_{ij,i'j'} = & \tr\left(\left[\mat{K}^{(h)}(\vect{x},\vect{x}')\right]_{D_{ij,i'j'}}\right). 
		\end{align*}
	\end{itemize*}
\end{itemize*}
Note that $\mat{\Sigma}(\vect{x},\vect{x}')$ and $\dot{\mat{\Sigma}}(\vect{x},\vect{x}')$ share similar structures as their NTK counterparts in Equations~\eqref{eqn:gp-cov-kernel} and \eqref{eqn:gradient_kernel}.
The only difference is that we have one more step, taking the trace over patches.
This step represents the convolution operation in the corresponding CNN.
Next, we can use a recursion to compute the CNTK:
\begin{enumerate*}
\item First, we define  $\mat{\Theta}^{(0)}(\vect{x},\vect{x}') = \mat{\Sigma}^{(0)}(\vect{x},\vect{x}')$.

\item For $h=1,\ldots,L-1$ and $(i,j,i',j') \in [\nnw] \times [\nnh] \times [\nnw] \times [\nnh]$, we define
\begin{align*}\label{eqn:vanila_cnn_gradient_kernel}
\left[\mat{\Theta}^{(h)}(\vect{x},\vect{x}')\right]_{ij,i'j'} = \tr\left(\left[\dot{\mat{K}}^{(h)}(\vect{x},\vect{x}')\odot\mat{\Theta}^{(h-1)}(\vect{x},\vect{x}')+\mat{K}^{(h)}(\vect{x},\vect{x}')\right]_{D_{ij,i'j'}}\right) .
\end{align*}

\item For $h=L$ , we define
$\mat{\Theta}^{(L)}(\vect{x},\vect{x}')=\dot{\mat{K}}^{(L)}(\vect{x},\vect{x}')\odot\mat{\Theta}^{(L-1)}(\vect{x},\vect{x}')+\mat{K}^{(L)}(\vect{x},\vect{x}').$

\item The final CNTK value is defined as 
	$
	\tr\left( \mat{\Theta}^{(L)}(\vect{x},\vect{x}')\right).
	$

\end{enumerate*}

In Section~\ref{sec:cntk_gap} we give the CNTK formula for CNNs with GAP, which is similar to vanilla CNNs.
To compute the CNTK matrix corresponding to a CNN with GAP that has $L$ convolution layers and one fully-connected layer on $n$ samples, the time complexity is $O(n^2\nnw^2\nnh^2  L)$.
Previous work assumed that directly computing convolutional kernel (with pooling) exactly is computationally infeasible, and thus resorted to approximations like Monte Carlo sampling~\citep{novak2019bayesian}.
We are able to scale the exact CNTK computation to the full CIFAR-10 dataset and 20-layer CNN with GAP.
We present our efficient computation approach in Section~\ref{sec:fast}.

\section{Experiments}
\label{sec:exp}

We evaluate the performances of CNNs and their corresponding CNTKs on the CIFAR-10 dataset.
The implementation details are in Section~\ref{sec:exp_details}.
We also compare the performances between CNTKs and their corresponding random feat
Due to space limit, we defer these results on random features to Section~\ref{sec:exp_rf}.

\paragraph{Results.} We test two types of architectures, vanilla CNN and CNN with global average pooling (GAP), 
as described in Sections~\ref{sec:kernel} and~\ref{sec:cntk_gap}.
We also test CNTKs with only 2,000 training data to see whether their performances are consistent with CNTKs and CNNs using the full training set.
The results are summarized in Table~\ref{tab:cnn_cntk}.
Notice that in Table~\ref{tab:cnn_cntk}, depth is the total number of layers (including both convolution layers and fully-connected layers).
\begin{table*}
	\centering
	\resizebox{\columnwidth}{!}{%
	\renewcommand{\arraystretch}{1.5}
	\begin{tabular}{ |c|c|c|c|c|c|c|}
		\hline
		Depth & CNN-V & CNTK-V & CNTK-V-2K & CNN-GAP & CNTK-GAP & CNTK-GAP-2K \\
		\hline
		3 & 59.97\% & 64.47\% & 40.94\% & 63.81\% & 70.47\% & 49.71\%\\
		\hline
		4 & 60.20\% & 65.52\% & 42.54\% & 80.93\% & 75.93\% & 51.06\%\\
		\hline
		6 & 64.11\%& 66.03\% & 43.43\% & 83.75\% & 76.73\% & 51.73\%\\
		\hline
		11 & 69.48\% & 65.90\% & 43.42\% & 82.92\% & \textbf{77.43\%} & 51.92\%\\
		\hline
		21 & 75.57\% & 64.09\% & 42.53\% & 83.30\% & 77.08\% & 52.22\%\\
		\hline
	\end{tabular}
	}
	\caption{Classification accuracies of CNNs and CNTKs on the CIFAR-10 dataset.
CNN-V represents vanilla CNN and CNTK-V represents the kernel corresponding to CNN-V.
CNN-GAP represents CNN with GAP and CNTK-GAP represents the kernel correspondong to CNN-GAP.
CNTK-V-2K and CNTK-GAP-2K represent training CNTKs with only 2,000 training data. 
		\label{tab:cnn_cntk}
	}
\end{table*}

Several comments are in sequel.
First, CNTKs are very powerful kernels.
The best kernel, 11-layer CNTK with GAP, achieves 77.43\% classification accuracy on CIFAR-10.
This results in a significant new benchmark for performance of a pure kernel-based method on CIFAR-10, being  $10\%$ higher than methods reported in \citep{novak2019bayesian}.

Second, we find that for both CNN and CNTK, depth can affect the classification accuracy.
This observation demonstrates that depth not only matters in deep neural networks but can also affect the performance of CNTKs.

Third, the global average pooling operation can significantly increase the classification accuracy by 8\% - 10\% for both CNN and CNTK.
Based on this finding, we expect that many techniques that improve the performance of neural networks are in some sense universal, i.e., these techniques can also benefit kernel methods.

Fourth, we find that there is still a 5\% - 6\% performance gap between CNTKs and CNNs.
Since CNTKs exactly correspond to infinitely wide CNNs, this performance gap implies that finite width has its benefits.
Therefore, it is likely that recent theoretical work on over-parameterization that operates in the NTK regime cannot fully explain the success of neural networks yet, and we believe it is an interesting open problem to characterize this gap.

\paragraph{Potential application in neural architecture search.} Finally, we find that performances of CNTK-V-2Ks and CNTK-GAP-2Ks are highly correlated to their CNN-V, CNTK-V, CNN-GAP and CNTK-GAP counterparts.
Again we see CNTK-GAP-2Ks outperform CNTK-V-2Ks by a large margin (about $8\%$ - $9\%$).
One potential application of this observation is to guide neural architecture search.
We can compute the kernel on a small training data, test it on a validation set, and choose neural network architectures based on the performance of this small kernel on the validation set.
We leave large scale experiments of this idea for future work.

\section{Conclusion}
\label{sec:diss}

By giving the first practical algorithm for computing CNTKs exactly, this paper allows investigation of the behavior of infinitely wide (hence infinitely over-parametrized) deep nets, which turns out to not be much worse than that of their finite counterparts.  We also give a fully rigorous proof that a sufficiently wide net is approximately equivalent to the kernel regression predictor, thus yielding a powerful new off-the-shelf kernel. 
We leave it as an open problem to understand the behavior of infinitely wide nets with features such as Batch Normalization or Residual Layers. Of course, one can also hope that the analysis of infinite nets provides rigorous insight into  finite ones.

\section*{Acknowledgments}
\label{sec:ack}
We thank Jason D. Lee, Haochuan Li and Xiyu Zhai for useful discussions.
S. Arora, W. Hu and Z. Li are supported by NSF, ONR, Simons Foundation, Schmidt Foundation, Mozilla Research, Amazon Research, DARPA and SRC.
R. Salakhutdinov and R. Wang are supported in part by NSF IIS-1763562, Office of Naval Research grant N000141812861,
and Nvidia NVAIL award. 
We thank Amazon Web Services for providing compute time for the experiments in this paper.
This work was done while S. S. Du was a Ph.D. student at Carnegie Mellon University.

\bibliography{simonduref}

\begin{thebibliography}{30}
\providecommand{\natexlab}[1]{#1}
\providecommand{\url}[1]{\texttt{#1}}
\expandafter\ifx\csname urlstyle\endcsname\relax
  \providecommand{\doi}[1]{doi: #1}\else
  \providecommand{\doi}{doi: \begingroup \urlstyle{rm}\Url}\fi

\bibitem[Allen-Zhu et~al.(2018{\natexlab{a}})Allen-Zhu, Li, and
  Liang]{allen2018learning}
Zeyuan Allen-Zhu, Yuanzhi Li, and Yingyu Liang.
\newblock Learning and generalization in overparameterized neural networks,
  going beyond two layers.
\newblock \emph{arXiv preprint arXiv:1811.04918}, 2018{\natexlab{a}}.

\bibitem[Allen-Zhu et~al.(2018{\natexlab{b}})Allen-Zhu, Li, and
  Song]{allen2018convergence}
Zeyuan Allen-Zhu, Yuanzhi Li, and Zhao Song.
\newblock A convergence theory for deep learning via over-parameterization.
\newblock \emph{arXiv preprint arXiv:1811.03962}, 2018{\natexlab{b}}.

\bibitem[Arora et~al.(2019)Arora, Du, Hu, Li, and Wang]{arora2019fine}
Sanjeev Arora, Simon~S Du, Wei Hu, Zhiyuan Li, and Ruosong Wang.
\newblock Fine-grained analysis of optimization and generalization for
  overparameterized two-layer neural networks.
\newblock \emph{arXiv preprint arXiv:1901.08584}, 2019.

\bibitem[Boucheron et~al.(2013)Boucheron, Lugosi, and
  Massart]{boucheron2013concentration}
St{\'e}phane Boucheron, G{\'a}bor Lugosi, and Pascal Massart.
\newblock \emph{Concentration inequalities: A nonasymptotic theory of
  independence}.
\newblock Oxford university press, 2013.

\bibitem[Cao and Gu(2019)]{cao2019generalization}
Yuan Cao and Quanquan Gu.
\newblock A generalization theory of gradient descent for learning
  over-parameterized deep relu networks.
\newblock \emph{arXiv preprint arXiv:1902.01384}, 2019.

\bibitem[Chizat and Bach(2018)]{chizat2018note}
Lenaic Chizat and Francis Bach.
\newblock A note on lazy training in supervised differentiable programming.
\newblock \emph{arXiv preprint arXiv:1812.07956}, 2018.

\bibitem[Cho and Saul(2009)]{cho2009kernel}
Youngmin Cho and Lawrence~K Saul.
\newblock Kernel methods for deep learning.
\newblock In \emph{Advances in neural information processing systems}, pages
  342--350, 2009.

\bibitem[Daniely(2017)]{daniely2017sgd}
Amit Daniely.
\newblock {SGD} learns the conjugate kernel class of the network.
\newblock In \emph{Advances in Neural Information Processing Systems}, pages
  2422--2430, 2017.

\bibitem[Daniely et~al.(2016)Daniely, Frostig, and Singer]{daniely2016toward}
Amit Daniely, Roy Frostig, and Yoram Singer.
\newblock Toward deeper understanding of neural networks: The power of
  initialization and a dual view on expressivity.
\newblock In \emph{Advances In Neural Information Processing Systems}, pages
  2253--2261, 2016.

\bibitem[Du and Hu(2019)]{du2019width}
Simon~S Du and Wei Hu.
\newblock Width provably matters in optimization for deep linear neural
  networks.
\newblock \emph{arXiv preprint arXiv:1901.08572}, 2019.

\bibitem[Du et~al.(2018{\natexlab{a}})Du, Hu, and Lee]{du2018algorithmic}
Simon~S Du, Wei Hu, and Jason~D Lee.
\newblock Algorithmic regularization in learning deep homogeneous models:
  Layers are automatically balanced.
\newblock In \emph{Advances in Neural Information Processing Systems 31}, pages
  382--393. 2018{\natexlab{a}}.

\bibitem[Du et~al.(2018{\natexlab{b}})Du, Lee, Li, Wang, and
  Zhai]{du2018global}
Simon~S Du, Jason~D Lee, Haochuan Li, Liwei Wang, and Xiyu Zhai.
\newblock Gradient descent finds global minima of deep neural networks.
\newblock \emph{arXiv preprint arXiv:1811.03804}, 2018{\natexlab{b}}.

\bibitem[Du et~al.(2019)Du, Zhai, Poczos, and Singh]{du2018provably}
Simon~S. Du, Xiyu Zhai, Barnabas Poczos, and Aarti Singh.
\newblock Gradient descent provably optimizes over-parameterized neural
  networks.
\newblock In \emph{International Conference on Learning Representations}, 2019.

\bibitem[Fan et~al.(2019)Fan, Ma, and Zhong]{fan2019selective}
Jianqing Fan, Cong Ma, and Yiqiao Zhong.
\newblock A selective overview of deep learning.
\newblock \emph{arXiv preprint arXiv:1904.05526}, 2019.

\bibitem[Garriga-Alonso et~al.(2019)Garriga-Alonso, Rasmussen, and
  Aitchison]{garriga-alonso2018deep}
Adrià Garriga-Alonso, Carl~Edward Rasmussen, and Laurence Aitchison.
\newblock Deep convolutional networks as shallow gaussian processes.
\newblock In \emph{International Conference on Learning Representations}, 2019.
\newblock URL \url{https://openreview.net/forum?id=Bklfsi0cKm}.

\bibitem[Hazan and Jaakkola(2015)]{hazan2015steps}
Tamir Hazan and Tommi Jaakkola.
\newblock Steps toward deep kernel methods from infinite neural networks.
\newblock \emph{arXiv preprint arXiv:1508.05133}, 2015.

\bibitem[Jacot et~al.(2018)Jacot, Gabriel, and Hongler]{jacot2018neural}
Arthur Jacot, Franck Gabriel, and Cl{\'e}ment Hongler.
\newblock Neural tangent kernel: Convergence and generalization in neural
  networks.
\newblock \emph{arXiv preprint arXiv:1806.07572}, 2018.

\bibitem[Lee et~al.(2018)Lee, Sohl-dickstein, Pennington, Novak, Schoenholz,
  and Bahri]{lee2018deep}
Jaehoon Lee, Jascha Sohl-dickstein, Jeffrey Pennington, Roman Novak, Sam
  Schoenholz, and Yasaman Bahri.
\newblock Deep neural networks as gaussian processes.
\newblock In \emph{International Conference on Learning Representations}, 2018.
\newblock URL \url{https://openreview.net/forum?id=B1EA-M-0Z}.

\bibitem[Lee et~al.(2019)Lee, Xiao, Schoenholz, Bahri, Sohl-Dickstein, and
  Pennington]{lee2019wide}
Jaehoon Lee, Lechao Xiao, Samuel~S Schoenholz, Yasaman Bahri, Jascha
  Sohl-Dickstein, and Jeffrey Pennington.
\newblock Wide neural networks of any depth evolve as linear models under
  gradient descent.
\newblock \emph{arXiv preprint arXiv:1902.06720}, 2019.

\bibitem[Li and Liang(2018)]{li2018learning}
Yuanzhi Li and Yingyu Liang.
\newblock Learning overparameterized neural networks via stochastic gradient
  descent on structured data.
\newblock \emph{arXiv preprint arXiv:1808.01204}, 2018.

\bibitem[Mairal et~al.(2014)Mairal, Koniusz, Harchaoui, and
  Schmid]{mairal2014convolutional}
Julien Mairal, Piotr Koniusz, Zaid Harchaoui, and Cordelia Schmid.
\newblock Convolutional kernel networks.
\newblock In \emph{Advances in neural information processing systems}, pages
  2627--2635, 2014.

\bibitem[Matthews et~al.(2018)Matthews, Rowland, Hron, Turner, and
  Ghahramani]{matthews2018gaussian}
Alexander G de~G Matthews, Mark Rowland, Jiri Hron, Richard~E Turner, and
  Zoubin Ghahramani.
\newblock Gaussian process behaviour in wide deep neural networks.
\newblock \emph{arXiv preprint arXiv:1804.11271}, 2018.

\bibitem[Neal(1996)]{neal1996priors}
Radford~M Neal.
\newblock Priors for infinite networks.
\newblock In \emph{Bayesian Learning for Neural Networks}, pages 29--53.
  Springer, 1996.

\bibitem[Novak et~al.(2019)Novak, Xiao, Bahri, Lee, Yang, Abolafia, Pennington,
  and Sohl-dickstein]{novak2019bayesian}
Roman Novak, Lechao Xiao, Yasaman Bahri, Jaehoon Lee, Greg Yang, Daniel~A.
  Abolafia, Jeffrey Pennington, and Jascha Sohl-dickstein.
\newblock Bayesian deep convolutional networks with many channels are gaussian
  processes.
\newblock In \emph{International Conference on Learning Representations}, 2019.
\newblock URL \url{https://openreview.net/forum?id=B1g30j0qF7}.

\bibitem[Roux and Bengio(2007)]{leroux07a}
Nicolas~Le Roux and Yoshua Bengio.
\newblock Continuous neural networks.
\newblock In \emph{Proceedings of the Eleventh International Conference on
  Artificial Intelligence and Statistics}, volume~2 of \emph{Proceedings of
  Machine Learning Research}, pages 404--411, San Juan, Puerto Rico, 2007.

\bibitem[Van~der Wilk et~al.(2017)Van~der Wilk, Rasmussen, and
  Hensman]{van2017convolutional}
Mark Van~der Wilk, Carl~Edward Rasmussen, and James Hensman.
\newblock Convolutional gaussian processes.
\newblock In \emph{Advances in Neural Information Processing Systems}, pages
  2849--2858, 2017.

\bibitem[Williams(1997)]{williams1997computing}
Christopher~KI Williams.
\newblock Computing with infinite networks.
\newblock In \emph{Advances in neural information processing systems}, pages
  295--301, 1997.

\bibitem[Yang(2019)]{yang2019scaling}
Greg Yang.
\newblock Scaling limits of wide neural networks with weight sharing: Gaussian
  process behavior, gradient independence, and neural tangent kernel
  derivation.
\newblock \emph{arXiv preprint arXiv:1902.04760}, 2019.

\bibitem[Zhang et~al.(2017)Zhang, Bengio, Hardt, Recht, and
  Vinyals]{zhang2016understanding}
Chiyuan Zhang, Samy Bengio, Moritz Hardt, Benjamin Recht, and Oriol Vinyals.
\newblock Understanding deep learning requires rethinking generalization.
\newblock In \emph{Proceedings of the International Conference on Learning
  Representations (ICLR), 2017}, 2017.

\bibitem[Zou et~al.(2018)Zou, Cao, Zhou, and Gu]{zou2018stochastic}
Difan Zou, Yuan Cao, Dongruo Zhou, and Quanquan Gu.
\newblock Stochastic gradient descent optimizes over-parameterized deep {ReLU}
  networks.
\newblock \emph{arXiv preprint arXiv:1811.08888}, 2018.

\end{thebibliography}
\bibliographystyle{plainnat}

\newpage
\appendix

\section{Experiment Details}
\label{sec:exp_details}
\paragraph{Setup.} Due to efficiency considerations, for all experiments, we use no data augmentation. 
Tricks like batch normalization, dropout, weight decay, etc. are not used for proper comparison.
We fix the filter $q$ to be $3$ and stride to be $1$.
We use zero padding to make sure the number of patches keeps unchanged after each convolutional layer. 
We set the number of convolution layers to be $2$, $3$, $5$, $10$, or $20$.
For both CNNs and CNTKs, we use the quadratic loss as the objective function.

Following \cite{novak2019bayesian}, for a label $c \in \{1,\ldots,10\}$, we use $-0.1 \cdot \vect{1} +  \vect{e}_{c}$ as its encoding. 
For example, if the class label is $3$, we use $\left(-0.1,-0.1,0.9,-0.1,\ldots,-0.1\right)$ as its encoding.
During training time, we calculate $(\mat H^*)^{-1} \vect Y$, where $\mat H^*$ is the CNTK matrix on inputs, and the $i$-th row of ${\mat Y} \in \mathbb{R}^{n \times 10}$ is the encoding of the label of the $i$-th data.
During testing time, for a test data point $\vect{x}_{te}$, we calculate 
$$
\vect f^*(\vect x_{te}) = \left( \kernel(\vect{x}_{te}, \vect x_1), \ldots, \kernel(\vect{x}_{te}, \vect x_n) \right) \cdot (\mat H^*)^{-1} \vect Y
$$
and choose the class with largest value as the prediction. 

The architecture of CNNs is as described in Section~\ref{sec:kernel} and Section~\ref{sec:cntk_gap}.
 We set the number of the channels of the network as 1024 and $\kappa$ as $0.05$. To train CNNs, we use stochastic gradient descent (SGD) with fixed learning rate. We report the best average performance over 3 trials among the different learning rate chosen from $\{0.1,1,10\}$. The test accuracy is measured by taking average of the  10 epochs after reaching full training accuracy except the depth-3 vanilla CNN, which couldn't attain full training accuracy within 3000 epochs for all learning rates

Our neural networks are trained using the PyTorch package, using (possibly multiple) NVIDIA Tesla V100 GPUs.
We calculate the kernel values using the CuPy\footnote{\url{https://cupy.chainer.org}.} package. 
For time-consuming operations, we write native CUDA codes to speed up the calculation. 
All experiments are performed on Amazon Web Services (AWS).

\section{Additional Experiments on Random Features}
\label{sec:exp_rf}
We verify the importance of using the exact kernels instead of the approximated ones from random features (as done in \citep{lee2019wide}).
The random features are generated by taking the gradient of the randomly initialized CNNs with respect to the weight matrices.
For all CNNs we set the number of channels to be $128$.
We compare the performances of the exact kernels and the random kernels on a CIFAR-2 dataset, i.e., the first two class in CIFAR-10.
For each kernel generated by random features, we test $10$ times and report the median.
The results are summarized in Table~\ref{tab:rf_cntk}.
\begin{table*}[h]
	\centering
	\resizebox{\columnwidth}{!}{%
	\renewcommand{\arraystretch}{1.5}
	\begin{tabular}{ |c|c|c|c|c|}
		\hline
		Depth & RF from Vanilla CNTK & Vanilla CNTK & RF for CNTK-GAP & CNTK-GAP\\
		\hline
		3 & 87.25\% & 92.15\% & 51.10\% & 71.05\%\\
				\hline
		4 & 87.78\% & 92.80\% & 52.85\% & 94.50\%\\
		\hline
		6 & 88.73\% & 93.10\% & 53.98\% & 95.25\%\\
		\hline
		11 & 87.80\% & 93.05\% & 56.55\% & 95.40\%\\
		\hline
		21 & 85.35\% & 91.95\% & 90.65\% & 95.70\%\\
		\hline
	\end{tabular}
	}
	\caption{Classification accuracies of random kernels generated from random features and exact CNTKs on CIFAR-2.
		\label{tab:rf_cntk}
	}
\end{table*}

Note that even on the simple CIFAR-2 dataset, random features have much worse accuracies than exact kernels by a large margin.
This experiment demonstrates the importance of using the exact kernels instead of approximated ones.

\section{Proof of Lemma~\ref{lem:dynamics}}
\label{appsec:omittedpfs}
\begin{proof}[Proof of Lemma~\ref{lem:dynamics}]
The parameters $\params$ evolve according to the differential equation
\begin{equation}\label{eqn:gf}
\frac{\d\params(t)}{\d t} = -\nabla\ell(\params(t)) = - \sum_{i=1}^n \left( f(\params(t), \vect x_i) - y_i \right) \frac{\partial f(\params(t), \vect x_i)}{\partial \params},
\end{equation}
where $t\ge 0$ is a continuous time index. 
Under Equation~\eqref{eqn:gf}, the evolution of the network output $f(\params(t), \vect x_i)$ can be written as
\begin{equation} \label{eqn:output-dynamics}
\frac{\d f(\params(t), \vect x_i)}{\d t} = - \sum_{j=1}^n (f(\params(t), \vect x_j) - y_j) \left\langle \frac{\partial f(\params(t), \vect x_i)}{\partial \params}, \frac{\partial f(\params(t), \vect x_j)}{\partial \params} \right\rangle, \qquad \forall i\in[n].
\end{equation}
Since $\vect u(t) = \left( f(\params(t), \vect x_i) \right)_{i\in[n]} \in \R^n $ is the network outputs on all $\vect x_i$'s at time $t$, and $\vect y = (y_i)_{i\in[n]}$ is the desired outputs, Equation~\eqref{eqn:output-dynamics} can be written more compactly as
\begin{equation} \label{eqn:output-dynamics-2}
\frac{\d \vect u(t)}{\d t} = - \mat H(t) \cdot (\vect u(t) - \vect y),
\end{equation}
where $\mat H(t) \in \R^{n\times n}$ is a kernel matrix defined as $[\mat H(t)]_{i, j} = \left\langle \frac{\partial f(\params(t), \vect x_i)}{\partial \params}, \frac{\partial f(\params(t), \vect x_j)}{\partial \params} \right\rangle$ ($\forall i, j\in[n]$).
\end{proof}

\section{NTK Derivation}
\label{sec:ntk_der}
In this section we derive NTK  for the fully-connected neural net defined in Section~\ref{sec:ntk-recap}. 

First we explain how the Gaussian process covariance in Equation~\eqref{eqn:gp-cov-kernel} is obtained.
The intuition is that $\left[ \vect f^{(h+1)}(\vect x) \right]_i = \sum_{j=1}^{d_{h}} \left[ \mat W^{(h+1)} \right]_{i, j} \left[ \vect g^{(h)}(\vect x) \right]_j $ is a centered Gaussian process conditioned on $\vect f^{(h)}$ ($\forall i\in[d_{h+1}]$), with covariance
\begin{equation} \label{eqn:covariance-calculation}
\begin{aligned}
\expect\left[\left[ \vect f^{(h+1)}(\vect x) \right]_i \cdot \left[ \vect f^{(h+1)}(\vect x') \right]_i \Big| \vect f^{(h)} \right]
&= \left\langle \vect g^{(h)}(\vect x), \vect g^{(h)}(\vect x') \right\rangle \\
&= \frac{c_\sigma}{d_{h}} \sum_{j=1}^{d_{h}} \sigma\left( \left[\vect f^{(h)}(\vect x)\right]_j  \right) \sigma\left( \left[\vect f^{(h)}(\vect x')\right]_j  \right),
\end{aligned}
\end{equation}
which converges to $\Sigma^{(h)}(\vect x, \vect x')$ as $d_h\to\infty$ given that each $\left[\vect f^{(h)}\right]_j$ is a centered Gaussian process with covariance $\Sigma^{(h-1)}$.
This yields the inductive definition in Equation~\eqref{eqn:gp-cov-kernel}.

Recall that we need to compute the value that $\left\langle \frac{\partial f(\params, \vect x)}{\partial \params}, \frac{\partial f(\params, \vect x')}{\partial \params} \right\rangle$ converges to at random initialization in the infinite width limit.
We can write the partial derivative with respect to a particular weight matrix $\mat W^{(h)}$ in a compact form:
\begin{align*}
\frac{\partial f(\params,\vect{x})}{\partial \mat{W}^{(h)}} = \back^{(h)}(\vect{x}) \cdot \left(\vect{g}^{(h-1)}(\vect x)\right)^\top , \qquad h=1, 2, \ldots, L+1,
\end{align*}
where \begin{equation}
\begin{aligned}
&\back^{(h)}(\vect{x}) = 
\begin{cases}
1 \in \mathbb{R}, \quad & h=L+1,\\
\sqrt{\frac{c_\sigma}{d_h}} \mat{D}^{(h)}(\vect{x}) \left(\mat W^{(h+1)} \right)^\top \back^{(h+1)}(\vect{x}) \in \mathbb{R}^{d_{h}}, \quad & h=1,\ldots,L,\\
\end{cases}
\end{aligned}
\label{eqn:back}
\end{equation}
\begin{align}
\mat{D}^{(h)}(\vect{x}) = \diag\left( \dot{\sigma}\left( \vect f^{(h)}(\vect x) \right) \right) \in \mathbb{R}^{d_h \times d_h}, \qquad h=1, \ldots, L. \label{eqn:D}
\end{align}
Then, for any $h \in [L+1]$, we can compute
\begin{align*}
\left\langle\frac{\partial f(\params,\vect{x})}{\partial \mat{W}^{(h)}}, \frac{\partial f(\params,\vect{x}')}{\partial \mat{W}^{(h)}}\right\rangle 
& = \left\langle \back^{(h)}(\vect{x}) \cdot \left(\vect{g}^{(h-1)}(\vect x)\right)^\top, \back^{(h)}(\vect{x}') \cdot \left(\vect{g}^{(h-1)}(\vect x')\right)^\top \right\rangle\\
& = \left\langle \vect{g}^{(h-1)}(\vect x), \vect{g}^{(h-1)}(\vect x')  \right\rangle \cdot \left\langle  \back^{(h)}(\vect{x}),  \back^{(h)}(\vect{x}') \right\rangle.
\end{align*}
Note that we have established in Equation~\eqref{eqn:covariance-calculation} that
\begin{align*}
\left\langle \vect{g}^{(h-1)}(\vect x), \vect{g}^{(h-1)}(\vect x')  \right\rangle  \to \Sigma^{(h-1)}\left(\vect{x},\vect{x}'\right).
\end{align*}
For the other factor $\left\langle  \back^{(h)}(\vect{x}),  \back^{(h)}(\vect{x}') \right\rangle$, from Equation~\eqref{eqn:back} we get
\begin{equation} \label{eqn:back-inner-prod}
\begin{aligned}
&\left\langle  \back^{(h)}(\vect{x}),  \back^{(h)}(\vect{x}') \right\rangle \\
=\, &\left\langle  \sqrt{\frac{c_\sigma}{d_h}} \mat{D}^{(h)}(\vect{x}) \left(\mat W^{(h+1)} \right)^\top \back^{(h+1)}(\vect{x}),\sqrt{\frac{c_\sigma}{d_h}} \mat{D}^{(h)}(\vect{x}') \left(\mat W^{(h+1)} \right)^\top \back^{(h+1)}(\vect{x}')\right\rangle.
\end{aligned}	
\end{equation}

Although $\mat{W}^{(h+1)}$ and $\back_{h+1}(\vect{x})$ are dependent,
the Gaussian initialization of $\mat{W}^{(h+1)}$ allows us to replace $\mat{W}^{(h+1)}$ with a fresh new sample $\widetilde{\mat{W}}^{(h+1)}$ without changing its limit: (This is made rigorous for ReLU activation in Theorem~\ref{thm:ntk_init}.)
\begin{align*}
&\left\langle  \sqrt{\frac{c_\sigma}{d_h}} \mat{D}^{(h)}(\vect{x}) \left(\mat W^{(h+1)} \right)^\top \back^{(h+1)}(\vect{x}),\sqrt{\frac{c_\sigma}{d_h}} \mat{D}^{(h)}(\vect{x'}) \left(\mat W^{(h+1)} \right)^\top \back^{(h+1)}(\vect{x}')\right\rangle \\
\approx & \left\langle  \sqrt{\frac{c_\sigma}{d_h}} \mat{D}^{(h)}(\vect{x}) \left(\widetilde{\mat W}^{(h+1)} \right)^\top \back^{(h+1)}(\vect{x}),\sqrt{\frac{c_\sigma}{d_h}} \mat{D}^{(h)}(\vect{x'}) \left(\widetilde{\mat W}^{(h+1) }\right)^\top \back^{(h+1)}(\vect{x}')\right\rangle \\
\to& \frac{c_\sigma}{d_h} \trace{ \mat{D}^{(h)} (\vect x) \mat{D}^{(h)}(\vect x')} \left\langle  \back^{(h+1)}(\vect{x}),  \back^{(h+1)}(\vect{x}') \right\rangle\\
\to & \dot{\Sigma}^{(h)}\left(\vect{x},\vect{x}'\right) \left\langle  \back^{(h+1)}(\vect{x}),  \back^{(h+1)}(\vect{x}') \right\rangle.
\end{align*}
Applying this approximation inductively in Equation~\eqref{eqn:back-inner-prod}, we get
\begin{align*}
\left\langle  \back^{(h)}(\vect{x}),  \back^{(h)}(\vect{x}') \right\rangle  \to \prod_{h'=h}^{L} \dot{\Sigma}^{(h')}(\vect{x},\vect{x}').
\end{align*}

Finally, since $\left\langle \frac{\partial f(\params, \vect x)}{\partial \params}, \frac{\partial f(\params, \vect x')}{\partial \params} \right\rangle = \sum_{h=1}^{L+1} \left\langle\frac{\partial f(\params,\vect{x})}{\partial \mat{W}^{(h)}}, \frac{\partial f(\params,\vect{x}')}{\partial \mat{W}^{(h)}}\right\rangle $, we obtain the final NTK expression for the fully-connected neural network:
\begin{align*}
\Theta^{(L)}(\vect{x},\vect{x}') = \sum_{h=1}^{L+1} \left(
\Sigma^{(h-1)}(\vect{x},\vect{x}') \cdot \prod_{h'=h}^{L+1} \dot{\Sigma}^{(h')}(\vect{x},\vect{x}')
\right).
\end{align*}

\section{Proof of Theorem~\ref{thm:ntk_init}}\label{appsec:ntk_init}

\subsection{Notation and Some Properties of ReLU}

\begin{defn}[$k$-homogeneous function]
A function $f:\R\to \R$ is said to be $k$-homogeneous, if 
$f(\lambda x) = \lambda^k f(x)$  for all $x\in\R, \lambda >0$.
\end{defn}
\begin{defn}
Let $\mathcal{S}^+$ be the set of \emph{positive semi-definite kernels} over $\R^d$, that is 

\[\mathcal{S}^+  = \left\{K:\R^d\times \R^d\to \R \bigg| \forall N\in \mathbb{N},  \vect x_1,\ldots \vect x_N\in \R^d, c_1,\ldots,c_N \in \R,\ \sum\nolimits_{i=1}^N \sum\nolimits_{j=1}^N c_ic_jK(x_i,x_j)\ge 0. \right\}\] 
\end{defn}

Let $\sigma:\R \to \R$ be the activation function, and $\T{\sigma} : \mathcal{S}^+ \to \mathcal{S}^+$ be the operator induced by $\sigma$, 

\[\forall \vect x,\vect x' \in \R^d, \quad \T{\sigma}(K)(\vect x,\vect x') = c_{\sigma}\Ex{(u,v) \sim \gauss\left(\vect{0}, \K{x,x'}\right) }{\relu{u}\relu{v}},\]

where $\K{x,x'} \in \R^{2\times 2}$, $\K{x,x'} = \begin{bmatrix} K(\vect x,\vect x) & K(\vect x,\vect x') \\ K(\vect x',\vect x) & K(\vect x',\vect x')\end{bmatrix}$. 

For convenience, we use $t_\sigma(\mat \Sigma)$ to denote  $ c_{\sigma}\Ex{(u,v) \sim \gauss\left(\vect{0}, \mat \Sigma \right) }{\relu{u}\relu{v}}$, and define $\htt_\sigma(\rho) $ as 
\[\htt_\sigma(\rho) = c_{\sigma}\Ex{(u,v) \sim \mat \Sigma' }{\relu{u}\relu{v}}, \textrm{ with } \mat \Sigma' =  \begin{bmatrix} 1 & \rho \\ \rho & 1\end{bmatrix}\]

When $\sigma$ is k-homogeneous function, we have 
\[t_\sigma(\mat \Sigma) =c_{\sigma}\left(\Sigma_{11} \Sigma_{22}\right)^\frac{k}{2}\Ex{(u,v) \sim \gauss\left(\vect{0}, \mat \Sigma' \right) }{\relu{u}\relu{v}} \textrm{\quad with\quad} \mat\Sigma' = \begin{bmatrix} 1 & \frac{\Sigma_{12}}{\sqrt{\Sigma_{11}\Sigma_{22}}} \\ \frac{\Sigma_{12}}{\sqrt{\Sigma_{11}\Sigma_{22}}} & 1\end{bmatrix}.\] 

Thus $t_\sigma (\mat \Sigma)$ can be written as $c_\sigma\left(\Sigma_{11} \Sigma_{22}\right)^\frac{k}{2} \htt({\frac{\Sigma_{12}}{\sqrt{\Sigma_{11}\Sigma_{22}}}})$,

\begin{fact}[Some facts about $\relu{z}=\max(0,z)$ and $\T{\sigma}$]\ 
\begin{enumerate}
\item For all activation function $\sigma$, $t_\sigma\left(\begin{bmatrix} 1 &1 \\ 1& 1\end{bmatrix}\right) =1$.
\item For all 1-homogeneous activation $\sigma$,  $\htt_\sigma(1) =1$  and $t_\sigma\left(\begin{bmatrix} a &a \\ a& a\end{bmatrix}\right) =a^k$   .
\item For $\relu{z}=\max(0,z)$, $\htt_\sigma(\rho) =\frac{\sqrt{1-\rho^2}+\rho\arcsin \rho}{\pi}+\frac{x}{2}$, $\htt_{\dot{\sigma}}(\rho) = \frac{1}{2}+ \frac{\arcsin \rho}{\pi} $ and $c_\sigma = c_{\dot{\sigma}} =2 $. 
\end{enumerate}
\end{fact}

\begin{lem}[Uniform Continuity of $\arcsin z$]\label{lem:uniform_continuity}
\ 

\begin{enumerate}
\item For any $-\frac{\pi}{2}\le y'\le y \le \frac{\pi}{2}$, $\sin y-\sin y' \ge 2 \sin^2\frac{y-y'}{2}$.
\item $\sin y\ge \frac{2y}{\pi}, \ \forall y\in [0,\frac{\pi}{2}]$.
\item  $\arcsin $ is uniform continuous: for every $ \eps\in \R^+$, $| z- z'|< \frac{2\eps^2}{\pi^2} \Rightarrow |\arcsin z-\arcsin z'|<\eps$.  
\item For $\relu{z} = \max(0,z)$,  $\htt_{\dot{\sigma}}$ is uniform continuous: for every $ \eps\in \R^+$, $| z- z'|< 2\eps^2 \Rightarrow |\htr{\dot{\sigma}}{z}-\htr{\dot{\sigma}}{z'}|<\eps$.  

\end{enumerate}
\end{lem}
\begin{proof}[Proof of Lemma~\ref{lem:uniform_continuity}]

(1). From $-\frac{\pi}{2}\le y'\le y' \le \frac{\pi}{2}$ we know $\frac{-\pi}{2}+\frac{y-y'}{2} \le \frac{y+y'}{2} \le \frac{\pi}{2} - \frac{y-y'}{2}$, which implies that $\cos(\frac{y+y'}{2}) \ge \sin(\frac{y-y'}{2})$. Thus,

\[ \sin y\sin y' = 2\cos\frac{y+y'}{2}\sin \frac{y-y'}{2}\ge 2\sin^2\frac{y-y'}{2}.\]

(2). Note that $\left(\frac{\sin y}{y}\right)'  = \frac{y \cos y-\sin y}{y^2}=\frac{\cos y}{y^2}(y-\tan y)<0$, $\frac{\sin y}{y}$ is decreasing on $[0,\frac{\pi}{2}]$. Thus $\frac{\sin y}{y} \ge \frac{1}{\frac{\pi}{2}} = \frac{2}{\pi}, \ \forall y\in[0,\frac{\pi}{2}]$.

(3). Let $y,y'\in[-\frac{\pi}{2},\frac{\pi}{2}]$, such that $\sin y = z,\sin y' =z'$. W.l.o.g., we assume $y'<y$, $z'<z$. Combing (1) and (2), we have $z-z' =\sin y-\sin y' \ge 2\sin^2\frac{y-y'}{2}\ge \frac{2(y-y')^2}{\pi^2}$. Thus $z-z'\le \frac{2\eps^2}{\pi^2} \Longrightarrow \arcsin z-\arcsin z' = y-y'\le \eps.$
\end{proof}

Recall the definition in Equation~\eqref{eqn:gp-cov-kernel} and~\eqref{eqn:gradient_kernel}, we have
\begin{equation*}
\begin{aligned}
\Sigma^{(0)}(\vect{x},\vect{x}') &= \vect{x}^\top \vect{x'},  \\
\mat{\twotwomat}^{(h)}(\vect{x},\vect{x}') &= \left. \mat \Sigma^{(h-1)}\right|_{\vect x,\vect x'} = \begin{pmatrix}
\Sigma^{(h-1)}(\vect{x},\vect{x}) & \Sigma^{(h-1)}(\vect{x},\vect{x}')\\
\Sigma^{(h-1)}(\vect{x}',\vect{x}) & \Sigma^{(h-1)}(\vect{x}',\vect{x}') 
\end{pmatrix}\in \R^{2\times2},  \\ 
\Sigma^{(h)}(\vect{x},\vect{x}') &=  c_{\sigma}\expect_{(u,v) \sim \gauss\left(\vect{0},\mat{\twotwomat}^{(h)}\right) }\left[\relu{u}\relu{v}\right],\\
\dot{\Sigma}^{(h)}(\vect{x},\vect{x}') & =  c_{\sigma}\expect_{(u,v) \sim \gauss\left(\vect{0},\mat{\twotwomat}^{(h)}\right)}\left[\dot{\sigma}(u)\dot{\sigma}(v)\right]
\end{aligned}
\end{equation*}
for $h=1, \ldots, L$.

For $\relu{z} = \max(z,0)$, we have 
\begin{equation*}
\Sigma^{(h)}(\vect{x},\vect{x}) = \norm{\vect x}^2,\quad \forall 0\le h \le L.
\end{equation*}

Let $\mat D = \mat D(\xa,\xb) = \mat{D}^{(h)}(\xa)\mat{D}^{(h)}(\xb)$ is a 0-1 diagonal matrix. 
We define the following events:
\begin{itemize}
\item $\fullgoodA{\xa}{\xb}{\eps_1}{h}:=\left\{\left|\g{h}{x^{(0)}}^\top \g{h}{\xa} - \mat \Sigma^{(h)}(\vect x^{(0)},\xa)\right| \le  \eps_1\right\}$,  $\forall 0\le h\le L$
\item $\fullGoodA{\xa}{\xb}{\eps_1}{h} =  \fullgoodA{\xa}{\xa}{\eps_1}{h}\cap \fullgoodA{\xa}{\xb}{\eps_1}{h}\cap \fullgoodA{\xb}{\xb}{\eps_1}{h}$;
\item $\fullGoodA{\xa}{\xb}{\eps_1}{} = \bigcup_{h=0}^L \GoodA{\xa}{\xb}{\eps_1}{h}$.
\item $\fullgoodB{\xa}{\xb}{\eps_2}{h} = \left\{\left|\left\langle  \back^{(h)}(\xa),  \back^{(h)}(\xb) \right\rangle  - \prod_{h=h}^{L} \dot{\Sigma}^{(h)}(\xa,\xb)\right| <\eps_2\right\}$;
\item $\fullGoodB{\xa}{\xb}{\eps_2}{h} =  \fullgoodB{\xa}{\xa}{\eps_2}{h}\cap \fullgoodB{\xa}{\xb}{\eps_2}{h}\cap \fullgoodB{\xb}{\xb}{\eps_2}{h}$;
\item $\fullGoodB{\xa}{\xb}{\eps_2}{} = \bigcup_{h=1}^{L+1} \GoodB{\xa}{\xb}{\eps_2}{h}$;
\item $\fullGoodC{\xa}{\xb}{\eps_3}{} = \left\{ |f(\vect \theta,\xa)| \le\eps_3,  |f(\vect \theta,\xb)| \le\eps_3 \right\}$;
\item $\fullgoodD{\xa}{\xb}{\eps_4}{h} = \left\{ |2\frac{\trace{\mat D (\xa,\xb)}}{d_h} -  \dot{\Sigma}^{(h)}(\xa,\xb)| < \eps_4\right\}$;
\item $\fullGoodD{\xa}{\xb}{\eps_4}{h} =  \fullgoodD{\xa}{\xa}{\eps_4}{h}\cap \fullgoodD{\xa}{\xb}{\eps_4}{h}\cap \fullgoodD{\xb}{\xb}{\eps_1}{h}$;
\item $\fullGoodD{\xa}{\xb}{\eps_4}{} = \bigcup_{h=1}^{L+1} \GoodD{\xa}{\xb}{\eps_4}{h}$.
\end{itemize}

For simplicity, we will omit $\xa,\xb$ when there's no ambiguity.
For events $\cA,\cB$, we define the event $\cA\Rightarrow \cB$ as $\neg\cA\wedge\cB$.

\begin{lem}\label{lem:cond_pr}
$\pr{\cA\Rightarrow \cB} \ge \pr{\cB\mid \cA}.$
\end{lem}

\begin{proof}
$\pr{\cA\Rightarrow \cB} = \pr{\neg\cA\wedge\cB} = 1 - \pr{\cA\vee\neg\cB} = 1- \pr{\neg\cB\mid\cA} \pr{\cA} \ge  1- \pr{\neg\cB\mid\cA}  = \pr{\cB\mid \cA}$.
\end{proof}

For matrix $\vect A$, define the projection matrix for the column space of $\vect A$, $\vect \Pi_{\vect A}:=  \vect A \vect A^\dag$ and the orthogonal projection matrix $\vect \Pi_{\vect A}^\perp = I - \vect A \vect A^\dag$. For two random variables $X$ and $Y$, $X\deq_\mathcal{A} Y$ means $X$ is equal to $Y$ in distribution conditioned on the $\sigma$-algebra generated by $\mathcal{A}$.

\begin{lem}\label{lem:gaussian_cond}
Let $\vect w\sim \gauss(0,\mat I_d)$, $\mat G \in \R^{d\times k}$ be some fixed matrix, and  random vector $\mat F =  \vect w^\top \mat G$, then conditioned on the value of $\mat F$, $\vect w$ remains gaussian in the null space of the row space of $\mat G$. 
Mathematically, it means
\begin{equation*}
\Pi^\perp_{\mat G}\vect w \deq_{\mat F = \vect w^\top \mat G } \Pi^\perp_{\mat G}\widetilde{\vect w},
\end{equation*}
where $\widetilde{\vect w}\sim \gauss(0,\mat I_d)$ is a fresh i.i.d. copy of $\vect w$.
\end{lem}

\begin{proof}
This lemma is straightforward when $\Pi^\perp_{\mat G}$ is a diagonal matrix. 

In general, let $\mat G = \mat U \mat G'$, where $\mat U\in \R^{d\times d}$ is orthogonal and $\Pi_{G'}^\perp$ is diagonal. Now we have
\begin{equation*}
 \Pi^\perp_{\mat G}\vect w 
 = \mat U  \Pi^\perp_{\mat G'} \mat U^\top \vect w 
\deq_{\mat F = ({\mat U}^\top \vect w)^\top \mat G'} \mat U  \Pi^\perp_{\mat G'} \mat U^\top \widetilde{\vect w },
= \Pi^\perp_{\mat G}\widetilde{\vect w}
\end{equation*}
where we used the fact that if $\vect w\sim \gauss(0,\mat I_d)$, then for any orthogonal  $\mat U$, $\mat U\vect w\sim \gauss(0,\mat I_d)$ twice.
\end{proof}

\subsection{Proof Sketch}

\ntkinit*

\begin{proof}
Recall that $\Theta^{(L)}(\vect{x},\vect{x}') = \sum_{h=1}^{L+1} \left(
\Sigma^{(h-1)}(\vect{x},\vect{x}') \cdot \prod_{h'=h}^{L+1} \dot{\Sigma}^{(h')}(\vect{x},\vect{x}')
\right)$, thus it suffices to show that if $\min_{h\in[L]} d_h \ge\Omega(\frac{L^{6}}{\eps^4} \log({L}/{\delta}))$, then w.p. $1-\delta$, for every $0\le h\le L$, it holds that 

\begin{equation*}
\left|\left\langle\frac{\partial f(\params,\vect{x})}{\partial \mat{W}^{(h)}}, \frac{\partial f(\params,\vect{x}')}{\partial \mat{W}^{(h)}}\right\rangle 
 -\Sigma^{(h-1)}(\vect{x},\vect{x}') \cdot \prod_{h'=h}^{L+1} \dot{\Sigma}^{(h')}(\vect{x},\vect{x}')\right| \le {\eps}. 
\end{equation*}
which is a direct consequence of Theorem~\ref{thm:GoodAB}
\end{proof}
\begin{thm}[Corollary 16 in~\citep{daniely2016toward}]\label{thm:GoodA}
Let $\relu{z} = \max(0,z), z\in\R$ and  $[\vect W^{(h)}]_{ij}\simiid \gauss(0,1)$, $\forall h\in [L], i\in [d^{h+1}], j\in [d^h]$,  there exist constants $c_1$,$c_2$, such that if $ c_1\frac{L^2\log\left(\frac{8L}{\delta}\right)}{\eps^2}\le \min\limits_{1\le h \le L} d_h $ and $\eps \le \min(c_2,\frac{1}{L})$, then for any fixed $\xa,\xb \in \R^{d_0}$, $\norm{\xa},\norm{\xb}\le 1$,  we have  w.p. $\ge 1-\delta$ , $\forall 0\le h \le L$, $ \forall (\xone,\xtwo)\in\{(\xa,\xa),(\xa,\xb),(\xb,\xb)\}$,

\begin{equation*}
\left|\g{h}{\xtwo}^\top \g{h}{\xone} - \mat \Sigma^{(h)}(\xtwo,\xone)\right| \le  \eps.
\end{equation*}

In other words, if $\min_{h\in[L]} d_h \geq c_1\frac{L^2 \log(\frac{L}{\delta_1})}{\eps_1^2}$, $\eps_1\le\min(c_2,\frac{1}{L})$, then for fixed $\xa,\xb$, 
\begin{equation*}\pr{\GoodA{\xa}{\xb}{\eps_1}{}} \ge 1-\delta_1.\end{equation*}
\end{thm}

\begin{thm}\label{thm:GoodAB}

Let  $\relu{z} = \max(0,z), z\in\R$, if $[\vect W^{(h)}]_{ij}\simiid N(0,1)$, $\forall h\in [L+1], i\in [d^{h+1}], j\in [d^h]$,  there exist constants $c_1$,$c_2$, such that if $\min_{h\in[L]} d_h \geq c_1\frac{L^2 \log(\frac{L}{\delta})}{\eps^4}, \eps \le \frac{c_2}{L}$, then  for any fixed $\xa,\xb\in \R^{d_0}$, $\norm{\xa},\norm{\xb}\le 1$,we have w.p. $1- \delta$, $\forall 0\le h \le L$, $\forall (\xone,\xtwo)\in\{(\xa,\xa),(\xa,\xb),(\xb,\xb)\}$, 
\begin{equation*}
\left|\g{h}{\xtwo}^\top \g{h}{\xone} - \mat \Sigma^{(h)}(\xtwo,\xone)\right| \le \frac{ \eps^2}{2},
\end{equation*}
and 
\begin{equation*}
\left|\left\langle  \back^{(h)}(\xone),  \back^{(h)}(\xtwo) \right\rangle  - \prod_{h'=h}^{L} \dot{\Sigma}^{(h')}(\xone,\xtwo)\right| < 3L\eps.
\end{equation*}

In other words, if $\min_{h\in[L]} d_h \geq c_1\frac{L^2 \log(\frac{L}{\delta_1})}{\eps^4_1}$, $\eps_1\le \frac{c_2}{L}$, then for fixed $\xa,\xb$, 
\begin{equation*}
\begin{split}
\pr{  \GoodA{\xa}{\xb}{\frac{\eps_1^2}{8 }}{} \bigwedge \GoodB{\xa}{\xb}{3L\eps_1 }{} } \geq &1-\delta
\end{split}
\end{equation*}

\end{thm}

Note that for $c_{\sigma} =2$ for $\relu{z} = \max(0,z)$, by definition of $\back^{(h)}$,  we have 
\begin{equation*}
\left\langle  \back^{(h)}(\xa),  \back^{(h)}(\xb) \right\rangle  
= \frac{2}{d_h} \back^{(h+1)}(\xa) ^\top\mat W^{(h+1)} \mat{D}^{(h)}(\xa)\mat{D}^{(h)}(\xb) \left(\mat W^{(h+1)} \right)^\top \back^{(h+1)}(\xb).
\end{equation*}
Intuitively, when $d_h$ is large, we can replace $\mat W^{(h+1)}$ by a fresh i.i.d copy $\widetilde{\mat W}$ with a small difference by $\tilde{O}(\frac{1}{\sqrt{d_h}})$ as below. Similar  techniques are used in \citep{yang2019scaling}. 
\begin{equation}
\begin{split}
\left\langle  \back^{(h)}(\xa),  \back^{(h)}(\xb) \right\rangle  
= &\frac{2}{d_h} \back^{(h+1)}(\xa) ^\top\mat W^{(h+1)} \mat{D}^{(h)}(\xa)\mat{D}^{(h)}(\xb) \left(\mat W^{(h+1)} \right)^\top \back^{(h+1)}(\xb)\\
\approx & \frac{2}{d_h} \back^{(h+1)}(\xa) ^\top \widetilde{\mat W}\mat{D}^{(h)}(\xa)\mat{D}^{(h)}(\xb) \widetilde{\mat W} ^\top \back^{(h+1)}(\xb)\\
\approx & \trace{\frac{2}{d_h} \mat{D}^{(h)}(\xa)\mat{D}^{(h)}(\xb)} \back^{(h+1)}(\xa) ^\top \back^{(h+1)}(\xb)\\
\approx & \dot{\Sigma}^{(h)}(\xone,\xtwo)  \prod_{h'=h+1}^{L} \dot{\Sigma}^{(h')}(\xone,\xtwo)
\end{split}
\end{equation}
The proof is based on a careful control of the following events.

\begin{restatable}{lem}{lemGoodC}
\label{lem:GoodC}
\[\pr{\GoodA{\xa}{\xb}{\eps_1^2/2}{L}\Longrightarrow \GoodC{\xa}{\xb}{2\sqrt{\log\frac{4}{\delta_3}}}{} } \ge 1-\delta_3,\quad  \forall \eps_1\in [0,1], \delta_3\in (0,1). \]
\end{restatable}

\begin{restatable}{lem}{lemgoodD}
\label{lem:goodD}
\[\pr{\GoodA{\xa}{\xb}{\eps_1^2/2}{h+1}\Longrightarrow \GoodD{\xa}{\xb}{\eps_1 + \sqrt{\frac{2\log\frac{6}{\delta_4}}{d_h}}}{h} } \ge 1-\delta_4,\quad  \forall \eps_1\in [0,1], \delta_4\in (0,1). \]
\end{restatable}

\begin{restatable}{lem}{lemGoodD}\label{lem:GoodD}
\[\pr{\GoodA{\xa}{\xb}{\eps_1^2/2}{}\Longrightarrow \GoodD{\xa}{\xb}{\eps_1 + \sqrt{\frac{2\log\frac{6L}{\delta_4}}{\min_h d_h}}}{} } \ge 1-\delta_4,\quad  \forall \eps_1\in [0,1], \delta_4\in (0,1). \]
\end{restatable}
\begin{proof}
Apply union bound on Lemma~\ref{lem:goodD}.
\end{proof}

\begin{restatable}{lem}{lemgoodB}\label{lem:goodB}
There exists constant $C,C'\in \R$, for any $\eps_2,\eps_3,\eps_4\in [0,1]$, we have
\begin{equation*}
\begin{split}
\pr{ \GoodA{\xa}{\xb}{\eps_1^2/2}{L}\bigwedge \GoodB{\xa}{\xb}{\eps_2 }{h+1} \bigwedge \GoodC{\xa}{\xb}{\eps_3}{} \bigwedge  \GoodD{\xa}{\xb}{\eps_4 }{h}  
\Longrightarrow  \GoodB{\xa}{\xb}{\eps_2 + \frac{C' \eps_3}{\sqrt{d_h}}+2\eps_4+C\sqrt{\frac{\log\frac{1}{\delta_2}}{d_h}}}{h}}\ge  1-\delta_2 
\end{split}
\end{equation*}
\end{restatable}


\begin{proof}[Proof of Theorem~\ref{thm:GoodAB}]
We will use induction on Lemma~\ref{lem:goodB} to prove Theorem~\ref{thm:GoodAB}.
In the statement of Theorem~\ref{thm:GoodA}, we set $\delta_1 =\frac{\delta}{4}$, $\eps_1 = \frac{\eps^2}{8}$, for some $c_1,c_2$, we have 
\begin{equation}\label{eq:GoodA}
 \pr{\GoodA{\xa}{\xb}{\eps^2/8}{L} }
 \ge 1-\delta/4 
\end{equation}

In the statement of Lemma~\ref{lem:GoodD}, we  set $\delta_4=\frac{\delta_2}{4}$, and $\eps_1 = \frac{\eps}{2}$. Note that for $c_1$ large enough $ \sqrt{\frac{2\log\frac{24L}{\delta}}{\min_h d_h}} \le \frac{\eps}{2}$  and thus we have 
\begin{equation}\label{eq:GoodD}
 \pr{\GoodA{\xa}{\xb}{\eps^2/8}{}\Rightarrow \GoodD{\xa}{\xb}{\eps}{} } 
\ge\pr{\GoodA{\xa}{\xb}{\eps^2/8}{}\Rightarrow \GoodD{\xa}{\xb}{\eps/2 + \sqrt{\frac{2\log\frac{24L}{\delta}}{\min_h d_h}}}{} } 
\ge 1-\delta/4
\end{equation}

In the statement of Lemma~\ref{lem:GoodC}, we set $\delta_3 = \frac{\delta}{4}$, and $\eps_1 = \frac{\eps^2}{8}$, we have 
\begin{equation}\label{eq:GoodC}
 \pr{\GoodA{\xa}{\xb}{\eps^2/8}{L}\Rightarrow \GoodC{\xa}{\xb}{2\sqrt{\log\frac{16}{\delta}}}{} } 
\ge 1-\delta/4
\end{equation}

Using union bound on Equation~\eqref{eq:GoodA},\eqref{eq:GoodD},\eqref{eq:GoodC}, we have 
\begin{equation}\label{eq:GoodACD}
 \pr{\GoodA{\xa}{\xb}{\eps^2/8}{L}\bigwedge \GoodC{\xa}{\xb}{2\sqrt{\log\frac{16}{\delta}}}{} \bigwedge  \GoodD{\xa}{\xb}{\eps}{}}
 \ge 1-\frac{3\delta}{4}
\end{equation}

Now we will begin the induction argument. First of all, note that $\pr{\GoodB{\xa}{\xb}{0}{L+1}}=1$ by definition.

For $1\le h\le L$ in the statement of Lemma~\ref{lem:goodB}, we set $\eps_2 = 3(L+1-h)\eps$, $\eps_3 = 3\sqrt{\log\frac{16}{\delta}}$, $\eps_2 = \eps$, $\delta_4 = \frac{\delta}{4L}$. Note that for $c_1$ large enough, $C\sqrt{\frac{\log\frac{1}{\delta_2}}{d_h}} + C'\sqrt{\frac{\log\frac{L}{\delta_2}}{d_h}} < \eps$. Thus we have 
\begin{equation}\label{eq:goodB}
\begin{split}
&\pr{  \GoodB{\xa}{\xb}{(3L-3h)\eps }{h+1} \bigwedge \GoodC{\xa}{\xb}{3\sqrt{\log\frac{16}{\delta_2}} }{} \bigwedge  \GoodD{\xa}{\xb}{\eps }{h}  
\Rightarrow  \GoodB{\xa}{\xb}{(3L+2-3h)\eps+C\sqrt{\frac{\log\frac{1}{\delta}}{d_h}} + 3C'\sqrt{\frac{\log\frac{16}{\delta}}{d_h}} }{h}}\\
\ge &\pr{  \GoodB{\xa}{\xb}{(3L-3h)\eps }{h+1} \bigwedge \GoodC{\xa}{\xb}{3\sqrt{\log\frac{16}{\delta_2}} }{} \bigwedge  \GoodD{\xa}{\xb}{\eps }{h}  
\Rightarrow  \GoodB{\xa}{\xb}{(3L+3-3h)\eps}{h}}\\
\ge &1-\frac{\delta}{4L}
\end{split}
\end{equation}
Using union bound again on Equation~\eqref{eq:GoodACD} and Equation~\eqref{eq:goodB} for every $h$ in $\{1,2,\ldots,L\}$, we have 
\begin{equation}\label{eq:goodABCD}
\begin{split}
 &\pr{\GoodA{\xa}{\xb}{\eps^2/8}{L}\bigwedge \GoodC{\xa}{\xb}{\eps}{} \bigwedge  \GoodD{\xa}{\xb}{\eps}{} \bigwedge \GoodB{\xa}{\xb}{3L\eps}{}}\\
 \ge & \pr{\GoodA{\xa}{\xb}{\eps^2/8}{L}\bigwedge \GoodC{\xa}{\xb}{\eps}{} \bigwedge  \GoodD{\xa}{\xb}{\eps}{} \bigwedge_{h=1}^L \GoodB{\xa}{\xb}{3(L+1-h)\eps}{h}}\\
 \ge & 1-  \left(1- \pr{\GoodA{\xa}{\xb}{\eps^2/8}{L}\bigwedge \GoodC{\xa}{\xb}{\eps}{} \bigwedge  \GoodD{\xa}{\xb}{\eps}{}}\right)\\
 -&\sum_{h=1}^L \left(1-\pr{  \GoodB{\xa}{\xb}{(3L-3h)\eps }{h+1} \bigwedge \GoodC{\xa}{\xb}{\eps }{} \bigwedge  \GoodD{\xa}{\xb}{\eps }{h}  
\Rightarrow  \GoodB{\xa}{\xb}{(3L+3-3h)\eps}{h}}\right) \\
\ge &1- \delta
\end{split}
\end{equation}
\end{proof}

\subsection{Proof of Lemma~\ref{lem:GoodC}}
\lemGoodC*
\begin{proof}
For fixed $\g{L}{ \vx}$,  $f(\vect \theta,\vect x) = \mat W^{(L+1)} \g{L}{ \vx} \deq N(0, \norm{\g{L}{ \vx}}^2$. Thus by subgaussian concentration[cite], we know w.p. $\ge1-\delta$ over the randomness of $\vect W^{(L+1)}$,  $ \left| f(\vect \theta,\vect x)  \right|\le \sqrt{2\log \frac{2}{\delta}}\norm{ \g{L}{ \vx}}$.

For $\eps_1\le1$, we have $\eps_1^2/2<1$, which implies $\norm{\g{L}{\vx}}^2\le 1+ \frac{\eps_1^2}{2}\le 2$, and thus taking union bound over $\xa,\xb$, we have w.p. $\ge 1-\delta$, $|f(\vect \theta,\xa)|\le 2\sqrt{\log\frac{2}{\delta}}$,$|f(\vect \theta,\xb)|\le 2\sqrt{\log\frac{2}{\delta}}$.

\[\pr{\GoodA{\xa}{\xb}{\eps_1^2/2}{L}\Rightarrow \GoodC{\xa}{\xb}{2\sqrt{\log\frac{4}{\delta_3}}}{} } \ge  \pr{ \GoodC{\xa}{\xb}{2\sqrt{\log\frac{4}{\delta_3}}}{} \mid \GoodA{\xa}{\xb}{\eps_1^2/2}{L}} \ge 1-\delta \]

\end{proof}

\subsection{Proof of Lemma~\ref{lem:goodD}}


\lemgoodD*
\begin{lem}\label{lem:tan_exp}
Define $\twotwog^{(h)}(\xa,\xb) = \begin{bmatrix}  \g{h}{\xa}^\top  \g{h}{\xa} &  \g{h}{\xa}^\top  \g{h}{\xb} \\  \g{h}{\xb}^\top  \g{h}{\xa} &  \g{h}{\xb}^\top  \g{h}{\xb}\end{bmatrix}$, we have for every $1\le h \le L$,

\begin{equation*}
\norm{\twotwog^{(h)}(\xa,\xb) - \twotwomat^{(h)}(\xa,\xb) }_{\infty}\le \frac{\eps^2}{2}
\Rightarrow \left| \trho{\dot{\sigma}}{\twotwog^{(h)}(\xa,\xb)} - \trho{\dot{\sigma}}{\twotwomat^{(h)}(\xa,\xb)}\right| \le \eps,\forall 0\le \eps\le 1.
\end{equation*}
\end{lem}

\begin{proof}
For simplicity, we denote $\twotwog^{(h)}(\xa,\xb)$, $\twotwomat^{(h)}(\xa,\xb)$ by $\mat G$,$\mat \Lambda$ respectively. 

Since $\dot{\sigma}(z) = \bm{1}[z\ge 0]$ is 0-homogeneous, we have 
\begin{equation*}
\begin{split}
\trho{\dot{\sigma}}{\mat G} = &\htr{\dot{\sigma}}{\frac{G_{12}}{\sqrt{G_{11}G_{22}}}} = \frac{1}{2}+ \arcsin\frac{G_{12}}{\sqrt{G_{11}G_{22}}}\\
\trho{\dot{\sigma}}{\mat \Lambda} = &\htr{\dot{\sigma}}{\frac{\Lambda_{12}}{\sqrt{\Lambda_{11}\Lambda_{22}}}} = \frac{1}{2}+ \arcsin\frac{\Lambda_{12}}{\sqrt{\Lambda_{11}\Lambda_{22}}} = \frac{1}{2} + \arcsin \Lambda_{12}\\
\end{split}
\end{equation*}

It is easy to verify that $|\sqrt{G_{11}G_{22}}-1|\le \eps^2/2$, and thus 
\[ \left| \frac{G_{12}}{\sqrt{G_{11}G_{22}}} -\Lambda_{12} \right| \le \left| \frac{G_{12}}{\sqrt{G_{11}G_{22}}} - \frac{\Lambda_{12}}{\sqrt{G_{11}G_{22}}} \right| + |\Lambda_{12}| \left| 1-\frac{1}{\sqrt{G_{11}G_{22}}} \right| \le \frac{\eps^2/2}{1-\eps^2/2} +\frac{\eps^2/2}{1-\eps^2/2}\le 2\eps^2.\]

Thus, by Lemma~\ref{lem:uniform_continuity}
\begin{equation*}
\begin{split}
|\trho{\dot{\sigma}}{\mat G} - \trho{\dot{\sigma}}{\mat \Lambda} |
\le  \left| \frac{1}{2}+ \arcsin\frac{G_{12}}{\sqrt{G_{11}G_{22}}} -  \frac{1}{2} + \arcsin \Lambda_{12}\right|
\le \eps.
\end{split}
\end{equation*}
\end{proof}
\begin{lem}\label{lem:tan_var}
For any $0\le h \le L-1$, any fixed $\{\mat W^{(i)}\}_{i=1}^{h}$, w.p. $1-\delta$ over the randomness of $\mat W^{(h+1)}\in\R^{d^{h+1}\times d^h}$, we have 


\begin{equation*}
\left|2\frac{\trace{\mat D }}{d_h} - \htr{\dot{\sigma}}{ \twotwog^{(h)}(\xa,\xb)} \right| < \sqrt{\frac{2\log\frac{2}{\delta}}{d_h}}.
\end{equation*}
\end{lem} 

\begin{proof}
Notice that $ \Exp{2\frac{\trace{\mat D }}{d_h}}=  \htr{\dot{\sigma}}{ \twotwog^{(h)}(\xa,\xb)}$, the proof is completed by Chernoff Bound.
\end{proof}

\begin{proof}[Proof of Lemma~\ref{lem:goodD}]
Note that $\dot{\Sigma}^{(h)}(\xa,\xb) = \trho{\sigma'}{\left.\mat\Sigma^{(h)}\right|_{\vect \xa,\vect\xb} } = \htr{\sigma'}{\twotwomat^{(h)}(\xa,\xb)}$.

Combining Lemma~\ref{lem:tan_exp} and Lemma~\ref{lem:tan_var}, we have for any $(\vx,\vxx)$,
\[\pr{\fullgoodD{\vx}{\vxx}{\eps_1 + \sqrt{\frac{2\log\frac{6}{\delta}}{d_h}}}{h} \mid \fullgoodA{\vx}{\vxx}{\eps_1^2/2}{h+1}  } \ge 1-\frac{\delta}{3}. \]

Taking union bound over $(\xa,\xa),(\xa,\xb),(\xb,\xb)$ for the choice of $(\vx,\vxx)$, we have 

\[\pr{\GoodA{\vx}{\vxx}{\eps_1^2/2}{h+1} \Rightarrow \GoodD{\vx}{\vxx}{\eps_1 + \sqrt{\frac{2\log\frac{6}{\delta}}{d_h}}}{h}   }  
\ge \pr{\GoodD{\vx}{\vxx}{\eps_1 + \sqrt{\frac{2\log\frac{6}{\delta}}{d_h}}}{h} \mid \GoodA{\vx}{\vxx}{\eps_1^2/2}{h+1}    }  
\ge  1-\delta\]
\end{proof}

\subsection{Proof of Lemma~\ref{lem:goodB}}
\lemgoodB*

The proof of Lemma~\ref{lem:goodB} is based on the following 3 claims, Claim~\ref{clm:grad_idp_1}, \ref{clm:grad_idp_2} and \ref{clm:grad_idp_3}. 
\begin{clm}\label{clm:grad_idp_1}
If $\GoodA{\xa}{\xb}{\eps_1^2/2}{L}\bigwedge \GoodB{\xa}{\xb}{\eps_2 }{h+1} \bigwedge \GoodC{\xa}{\xb}{\eps_3}{} \bigwedge  \GoodD{\xa}{\xb}{\eps_4 }{h} $, then we have
\begin{equation*}
\left|\frac{2\trace{\mat D}}{d_h} \left\langle  \back^{(h)}(\xtwo),  \back^{(h)}(\xone) \right\rangle  - \prod_{h'=h}^{L} \dot{\Sigma}^{(h')}(\xone,\xtwo) \right|  \le \eps_2+2\eps_4.
\end{equation*}
\begin{proof}
\begin{equation*}
\begin{split}
&\left|\frac{2\trace{\mat D}}{d_h} \left\langle  \back^{(h)}(\xtwo),  \back^{(h)}(\xone) \right\rangle  - \prod_{h'=h}^{L} \dot{\Sigma}^{(h')}(\xone,\xtwo) \right|  \\
\le & \left| \frac{2\trace{\mat D}}{d_h} -\dot{\Sigma}^{(h)}(\xone,\xtwo)\right|\cdot \left| \left\langle  \back^{(h)}(\xtwo),  \back^{(h)}(\xone) \right\rangle\right| \\
+& \left| \dot{\Sigma}^{(h)}(\xone,\xtwo)\right| \cdot \left| \left\langle  \back^{(h)}(\xtwo),  \back^{(h)}(\xone) \right\rangle-\prod_{h'=h+1}^{L} \dot{\Sigma}^{(h')}(\xone,\xtwo) \right|\\
\le &  2\eps_4 + \eps_2
\end{split}
\end{equation*}

\end{proof}
\end{clm}

For any fixed $h$, let $\vect G = [\g{h}{\xa}\  \g{h}{\xb}]$, 
\begin{clm}\label{clm:grad_idp_2}
w.p. $\ge 1-\frac{\delta_2}{2}$,  if $\GoodA{\xa}{\xb}{\eps_1^2/2}{L}\bigwedge \GoodB{\xa}{\xb}{\eps_2 }{h+1} \bigwedge \GoodC{\xa}{\xb}{\eps_3}{} \bigwedge  \GoodD{\xa}{\xb}{\eps_4 }{h} $, then we have for  any $(\xone,\xtwo)\in \{(\xa,\xa), (\xa,\xb),(\xb,\xb)\}$,
\begin{align*}
&\left| \frac{2}{d_h} \back^{(h+1)}(\xone) ^\top \mat W^{(h+1)} \Pi^\perp_{\mat G}\mat D \Pi^\perp_{\mat G} \left(\mat W^{(h+1)} \right)^\top \back^{(h+1)}(\xtwo) -\frac{2\trace{\mat D}}{d_h} \left\langle  \back^{(h)}(\xtwo),  \back^{(h)}(\xone) \right\rangle \right|\\
  \le &16\sqrt{\frac{\log\frac{6}{\delta_2}}{d_h}}.
\end{align*}
As a by-product, for any $\xone\in\{\xa,\xb\}$, we have 
\begin{equation*}
\sqrt{\frac{2}{d_h}}\norm{\back^{(h+1)}(\xone) ^\top \mat W^{(h+1)} \Pi^\perp_{\mat G}\mat D}   \le 4\sqrt{\frac{\log\frac{6}{\delta_2}}{d_h}}.
\end{equation*}
\end{clm}

\begin{lem}[Gaussian chaos of order 2 \citep{boucheron2013concentration}]\label{lem:gauss_chaos}
 Let $\vect \xi \sim N(0, \mat{I}_n)$ be an n-dimensional unit gaussian random vector, $\mat A\in\R^{n\times n}$ be a symmetric matrix, then for any $t>0$, 

\begin{equation*}
\pr{\left| \vect \xi^\top \mat A\vect \xi- \Exp{\vect \xi^\top \mat A\vect \xi}\right| > 2\norm{\mat A}_F \sqrt{t} + 2\norm{\mat A}_2t}\le 2\exp(-t).
\end{equation*}
Or, 
\begin{equation*}
\pr{\left| \vect \xi^\top \mat A\vect \xi- \Exp{\vect \xi^\top \mat A\vect \xi}\right| > t}\le 2\exp\left(-\frac{t^2}{4(\norm{\mat A}_F^2)+\norm{\mat A}_2t}\right).
\end{equation*}
\end{lem}

\begin{proof}[Proof of \ref{clm:grad_idp_2}]

It suffices to prove this claim conditioned on every possible realization of 
\[
\{\back^{(h+1)}(\xone),\back^{(h+1)}(\xtwo), \f{h}{\xone}, \f{h}{\xtwo}\}.
\]
Recall that $\vect G = \left[\g{h}{\xone}\  \g{h}{\xtwo}\right]$, we further define $\mat F = \left[\f{h}{\xone}\  \f{h}{\xtwo}\right]$.
Applying Lemma~\ref{lem:gaussian_cond} on each row of $\mat W^{h+1}$, we have
\begin{equation}
\mat W^{(h+1)} \Pi^\perp_{\mat G} \ \deq_{\mat F = \mat W^{(h+1)} \mat G} \ \widetilde{\mat W} \Pi^\perp_{\mat G},
\end{equation}
where $\widetilde{\mat W} $ is an iid copy of $\mat W^{(h+1)} $.

Note that $[ \back^{(h+1)}(\vect{x}) ^\top  \widetilde{\mat W}  \quad \back^{(h+1)}(\xone) ^\top \widetilde{\mat W} ]^\top \in \R^{2d_h}$ follows a joint zero-mean gaussian distribution with covariance matrix $\mat \Sigma = \begin{bmatrix} 
 b_{11}\mat I_n& b_{12} \mat I_n \\
b_{21} \mat I_n& b_{22}\mat I_n
\end{bmatrix}$,
where $b_{ij} =  \back^{(h)}(\vect{x}^{(i)})^\top \back^{(h)}(\vect{x}^{(j)}) $,  for $i,j =1,2$.  In other words, there exists $\mat M\in \R^{2d_h \times 2d_h}$, s.t. $\mat M\mat M^\top = \mat \Sigma$, and 
\[[ \back^{(h+1)}(\vect{x}) ^\top  \widetilde{\mat W} \quad \back^{(h+1)}(\xone) ^\top \widetilde{\mat W} ]^\top \deq M\vect \xi ,\] where $\vect \xi \sim N(\vect 0, \mat I_{2d_h})$. 

Thus conditioned on $\{\back^{(h+1)}(\xone),\back^{(h+1)}(\xtwo), \g{h}{\xone}, \g{h}{\xtwo}\}$, we have 
\begin{equation*}
\begin{split}
& \back^{(h+1)}(\xone) ^\top \mat W^{(h+1)} \Pi^\perp_{\mat G}\mat D \Pi^\perp_{\mat G} \left(\mat W^{(h+1)} \right)^\top \back^{(h+1)}(\xtwo) \\
\deq\  &  \back^{(h+1)}(\xone) ^\top \widetilde{\mat W}  \Pi^\perp_{\mat G}\mat D \Pi^\perp_{\mat G} \left(\widetilde{\mat W} \right)^\top \back^{(h+1)}(\xtwo)\\
\deq\  &   \left( \begin{bmatrix}  \mat I_{d_h}& \mat 0 \end{bmatrix} \mat M \vect \xi \right)^\top \Pi^\perp_{\mat G} \mat D \Pi^\perp_{\mat G} \left( \begin{bmatrix}  \mat 0 & \mat I_{d_h} \end{bmatrix} \mat M \vect \xi \right)\\
\deq\  &  \frac{1}{2}  \vect \xi^\top \mat M^\top  \begin{bmatrix}  \mat 0 &  \Pi^\perp_{\mat G}\mat D \Pi^\perp_{\mat G} \\  \Pi^\perp_{\mat G}\mat D \Pi^\perp_{\mat G}  & \mat 0  \end{bmatrix}   \mat M \vect \xi. \\
\end{split}
\end{equation*}
Now we are ready to prove Claim~\ref{clm:grad_idp_1} by applying Lemma~\ref{lem:gauss_chaos}. 
Let $\mat A = \frac{1}{2} \mat M^\top  \begin{bmatrix}  \mat 0 &  \Pi^\perp_{\mat G}\mat D \Pi^\perp_{\mat G} \\  \Pi^\perp_{\mat G}\mat D \Pi^\perp_{\mat G}  & \mat 0  \end{bmatrix}   \mat M$, we have 
\[\Exp{\vect \xi^\top \mat A\vect \xi} 
= \trace{\mat A} = \frac{1}{2} \trace{\begin{bmatrix}  \mat 0 &  \Pi^\perp_{\mat G}\mat D \Pi^\perp_{\mat G} \\  \Pi^\perp_{\mat G}\mat D \Pi^\perp_{\mat G}  & \mat 0  \end{bmatrix} \mat \Sigma} 
= b_{12}\trace{\Pi^\perp_{\mat G}\mat D \Pi^\perp_{\mat G} \mat I_n } 
 = b_{12} \trace{\mat D \Pi^\perp_{\mat G}}.\]
Note that by definition $\Pi^\perp_{\mat G} = \mat I_{d_h} - \Pi_{\mat G}$, and $\rank(\Pi_{\mat G})\le 2$,  we have 
\[\trace{\mat D \Pi^\perp_{\mat G}} = \trace{\mat D (\mat I- \Pi_{\mat G})} = \trace{D}  -  \trace{\mat D \Pi_{\mat G}} =\trace{D}  - \trace{\Pi_{\mat G} \mat D \Pi_{\mat G}}.
\] 

Since $\mat 0\preceq \mat D \preceq \mat I_{d_h}$, we have $ 0 \le \trace{\Pi_{\mat G} \mat D \Pi_{\mat G}} \le 2$, and thus $b_{12}(\trace{\mat D}-2)\le \Exp{\vect \xi^\top \mat A\vect \xi} \le b_{12} \trace{\mat D}.$
For  the upper bound of spectrum, note that $\norm{\mat M}^2_2 = \norm{\mat \Sigma}_2 = \norm{\begin{bmatrix} b_{11} &b_{12} \\ b_{21} & b_{22} \end{bmatrix}}_2 \le b_{11}+b_{12}$, and $\mat 0\preceq \Pi^\perp_{\mat G}, \mat D\preceq \mat I_{d_h}$, we have 
\[\norm{\mat A}_2 \le \frac{1}{2} \norm{\mat M}_2^2 \norm{\Pi^\perp_{\mat G}\mat D \Pi^\perp_{\mat G}}_2 \le \frac{1}{2} \norm{\mat M}_2^2 \norm{\Pi^\perp_{\mat G}}_2 \norm{\mat D}_2 \norm{\Pi^\perp_{\mat G}}_2 \le \frac{b_{11}+b_{22}}{2}\le \sqrt{2},\] 
and 
\[\norm{\mat A}_F \le \sqrt{2d_h} \norm{\mat A}_2 = \frac{\sqrt{2 d_h} (b_{11}+b_{22})}{2}\le 2\sqrt{d_h}.  \]
Thus by Lemma~\ref{lem:gauss_chaos} with $t= \log \frac{6}{\delta_2}$ we have w.p. $1-\frac{\delta_2}{6}$,
\[
\frac{1}{d_h}\left| \vect \xi^\top \mat A\vect \xi- \Exp{\vect \xi^\top \mat A\vect \xi}\right| 
\le \frac{1}{d_h}\left( 2\norm{\mat A}_F \sqrt{t} + 2\norm{\mat A}_2t\right) 
= 4\sqrt{\frac{\log\frac{6}{\delta_2}}{d_h}} + 2\sqrt{2}\frac{\log\frac{6}{\delta_2}}{d_h}. .\]
Thus we have
\begin{equation*}
\begin{split}
&\left| \frac{2}{d_h} \back^{(h+1)}(\xone) ^\top \mat W^{(h+1)} \Pi^\perp_{\mat G}\mat D \Pi^\perp_{\mat G} \left(\mat W^{(h+1)} \right)^\top \back^{(h+1)}(\xtwo) 
- \frac{2\trace{\mat D}}{d_h} \left\langle  \back^{(h)}(\xtwo),  \back^{(h)}(\xone) \right\rangle \right|\\
\le & \frac{2}{d_h}\left| \vect \xi^\top \mat A\vect \xi- \Exp{\vect \xi^\top \mat A\vect \xi}\right| + \left|2 \Exp{\vect \xi^\top \mat A\vect \xi} - \frac{2\trace{\mat D}}{d_h}\left\langle  \back^{(h)}(\xtwo),  \back^{(h)}(\xone) \right\rangle  \right|\\
\le & 8\sqrt{\frac{\log\frac{6}{\delta_2}}{d_h}} + 4\sqrt{2}\frac{\log\frac{6}{\delta_2}}{d_h} + \frac{4b_{12}}{d_h} 
\le  14\sqrt{\frac{\log\frac{6}{\delta_2}}{d_h}} + \frac{4(1+\eps_2)}{d_h} \ (2\sqrt{2}\le 3 \wedge \log\frac{6}{\delta_2}\le d_h)\\
\le & 16 \sqrt{\frac{\log\frac{6}{\delta_2}}{d_h}} \ (\eps_2\le 1 \wedge \sqrt{d_h \log 6}\ge 4).
 \end{split}
\end{equation*}


The main part of the claim is completed by taking union bound over $(\xa,\xa),(\xa,\xb),(\xb,\xb)$.
For the by-product, let $\xtwo = \xone$, and we have 
\begin{equation*}
\begin{split}
&\sqrt{\frac{2}{d_h}}\norm{\back^{(h+1)}(\xone) ^\top \mat W^{(h+1)} \Pi^\perp_{\mat G}\mat D} \\
\le&\sqrt{\left| \frac{2}{d_h} \back^{(h+1)}(\xone) ^\top \mat W^{(h+1)} \Pi^\perp_{\mat G}\mat D \Pi^\perp_{\mat G} \left(\mat W^{(h+1)} \right)^\top \back^{(h+1)}(\xtwo) \right| }\\
\le & \sqrt{\frac{2\trace{\mat D}}{d_h} \left\langle  \back^{(h)}(\xtwo),  \back^{(h)}(\xone) \right\rangle 
+ \left( 16 \sqrt{\frac{\log\frac{6}{\delta_2}}{d_h}}\right)^2}\\
\le & \sqrt{4  + \left( 16 \sqrt{\frac{\log\frac{6}{\delta_2}}{d_h}}\right)^2}\\
\le & 2  +  4 \sqrt{\frac{\log\frac{6}{\delta_2}}{d_h}} \le 6\qquad  ( \log\frac{6}{\delta_2}\le d_h)\\
\end{split}
\end{equation*}

\begin{clm}\label{clm:grad_idp_3}
w.p. $\ge 1-\frac{\delta_2}{2}$, if $\GoodA{\xa}{\xb}{\eps_1^2/2}{L}\bigwedge \GoodB{\xa}{\xb}{\eps_2 }{h+1} \bigwedge \GoodC{\xa}{\xb}{\eps_3}{} \bigwedge  \GoodD{\xa}{\xb}{\eps_4 }{h}$, then 
\begin{equation*}
\norm{ \vect \Pi_{\vect G}\left(\mat W^{(h+1)} \right)^\top \back^{(h+1)}(\xa) } \le 2\sqrt{\log\frac{8}{\delta_2}}+ \sqrt{2}\eps_3,
\norm{ \vect \Pi_{\vect G}\left(\mat W^{(h+1)} \right)^\top \back^{(h+1)}(\xb) } \le 2\sqrt{\log\frac{8}{\delta_2}}+ \sqrt{2}\eps_3.
\end{equation*}
\end{clm}

\begin{proof}
It suffices to prove the claim for $\xa$. We will denote $\xa$ by $\vect{x}$, $\g{h}{\xa}$ by $\vect g^{(h)}$ and $\back^{(h+1)}(\xa)$ by $\back^{(h+1)}$. We also define $\Pi_{\vect g}$ as $\vect{g}\vect g^\top$, and $\Pi_{\mat G / \vect g} = \Pi_{\mat G} - \Pi_{\vect g}$. Clearly,  $\Pi_{\mat G / \vect g}$ is still a projection matrix of rank $0$ or $1$. 

Since $\norm{ \vect \Pi_{\vect G}\left(\mat W^{(h+1)} \right)^\top \back^{(h+1)}(\xa) } \le  \norm{ \vect \Pi_{\vect g}\left(\mat W^{(h+1)} \right)^\top \back^{(h+1)} } + \norm{ \vect \Pi_{\vect G/\vect g}\left(\mat W^{(h+1)} \right)^\top \back^{(h+1)} }$, it suffices to bound these two terms separately.

Recall $\back^{(h+1)}$ is defined as the gradient of $f(\vect \theta,\vect x)$ with respect to the pre-activation of layer $h+1$, $\vect f^{h+1}$, thus if we view $g$ as a function $\vect g^{(h)}, \mat W^{(h+1)},\ldots, \mat W^{(L+1)}$, by the rule of back propagation, we have 

\[\frac{\partial g(\vect g^{(h)},  \mat W^{(h+1)},\ldots, \mat W^{(L+1)})}{\partial \vect g^{(h)}} =(\back^{(h+1)})^\top\mat W^{(h+1)}.\]

Note that relu is 1-homogeneous, namely $\forall \lambda \in \R^+, \relu{\lambda z} = \lambda \relu{z}$, the whole network is also 1-homogeneous in $\vect g^{(h)}$. In other words, we have 

\begin{equation*}
\begin{split}
&f(\vect g^{(h)},  \mat W^{(h+1)},\ldots, \mat W^{(L+1)}) \\
= &\left. \frac{\partial f(\lambda \vect g^{(h)},  \mat W^{(h+1)},\ldots, \mat W^{(L+1)})}{\partial \lambda }\right|_{\lambda =1 } \\
= & \left\langle \left. \frac{\partial f(\lambda \vect g^{(h)},  \mat W^{(h+1)},\ldots, \mat W^{(L+1)})}{\partial \lambda\vect g^{(h)} }\right|_{\lambda =1 } , \left.\frac{\partial \lambda \vect g^{(h)}}{\partial \lambda} \right|_{\lambda =1 }\right\rangle\\
=& \left\langle\frac{\partial g(\vect g^{(h)},  \mat W^{(h+1)},\ldots, \mat W^{(L+1)})}{\partial \vect g^{(h)}}  , \vect g^{(h)} \right\rangle\\
= & (\vect g^{(h)})^\top \left(\mat W^{(h+1)} \right)^\top \back^{(h+1)}\end{split}
\end{equation*}

By definition of $\Pi_{\vect g}$, we have 
\[\norm{ \vect \Pi_{\vect g}\left(\mat W^{(h+1)} \right)^\top \back } = \norm{ \frac{\vect g^{(h)}(\vect g^{(h)})^\top}{\norm{\vect g^{(h)}}^2} \left(\mat W^{(h+1)} \right)^\top \back^{(h+1))} } = \left|\frac{(\vect g^{(h)})^\top}{\norm{\vect g^{(h)}}} \left(\mat W^{(h+1)} \right)^\top \back^{(h+1))} \right|  = \frac{|f(\vect \theta, \vect x)|}{\norm{\vect g^{(h)}}}.
\]

Note that $\g{h}{x^{(0)}}^\top \g{h}{\xa}\ge 1 - \eps_1^2/2 \ge \frac{1}{2}$, we have 
\[\norm{ \vect \Pi_{\vect g}\left(\mat W^{(h+1)} \right)^\top \back } = \frac{|f(\vect \theta, \vect x)| }{\norm{\vect g^{(h)}}} \le \sqrt{2}\eps_3.\]
For the second term $\vect \Pi_{\vect G/\vect g}\left(\mat W^{(h+1)} \right)^\top \vect b^{(h+1)}$, note that conditioned on $ \vect g^{(h)}, \vect f^{h} = \frac{1}{\sqrt{d_{h+1}}}\mat W^{(h+1)}\vect g^{(h)}$ and all $\{\mat W^{(h)}\})_{h'}^{L+1}$ (thus $\back^{(h+1)}$),  by Lemma~\ref{lem:gaussian_cond}, $\Pi_{\mat G/\vect g}(\mat W^{(h+1)}) \deq \Pi_{\mat G/\vect g} \widetilde{\mat W}$, where $\widetilde{\mat W} $ is an iid copy of $\mat W^{(h+1)} $.
Thus if $\rank(\vect \Pi_{\vect G/\vect g}) =1$, suppose $\vect \Pi_{\vect G/\vect g} =\vect u\vect u^\top$ for some unit vector $\vect u$, we have  
\[\norm{\vect \Pi_{\vect G/\vect g}\left(\mat W^{(h+1)} \right)^\top \vect b^{(h+1)}} = \left| \vect u^\top \left(\mat W^{(h+1)} \right)^\top \vect b^{(h+1)} \right|  \deq \left| \vect u^\top \left(\widetilde{\mat W}\right)^\top \vect b^{(h+1)} \right| \deq  |t|, \] 
where $t\sim N(0,\norm{\back^{(h+1)}})$. Hence w.p. $\ge 1-\delta_2/4$ over the randomness of $\vect W^{(L)}$,  $ \norm{\vect \Pi_{\vect G/\vect g}\left(\mat W^{(h+1)} \right)^\top \vect b^{(h+1)}} \le \sqrt{2\log \frac{8}{\delta_2}}\norm{ \back^{(h+1)}}\le \sqrt{2\log\frac{8}{\delta_2}}\le 2\sqrt{\log\frac{8}{\delta_2}}\ (\eps_2<1)$.

If $\rank(\vect \Pi_{\vect G/\vect g}) =0$,  then $\norm{\vect \Pi_{\vect G/\vect g}\left(\mat W^{(h+1)} \right)^\top \vect b^{(h+1)}} =0 < 2\sqrt{\log\frac{8}{\delta_2}}$.
Thus w.p. $\ge 1-\frac{\delta_2}{4}$, 
\[\norm{ \vect \Pi_{\vect G}\left(\mat W^{(h+1)} \right)^\top \back^{(h+1)}(\xa) } \le  \norm{ \vect \Pi_{\vect g}\left(\mat W^{(h+1)} \right)^\top \back^{(h+1)} } + \norm{ \vect \Pi_{\vect G/\vect g}\left(\mat W^{(h+1)} \right)^\top \back^{(h+1)} } \le 2\sqrt{\log\frac{8}{\delta_2}}+ \sqrt{2}\eps_3 .\]
Thus by assumption $\log\frac{8}{\delta_2} \le d_h$, we have $2\sqrt{\log\frac{8}{\delta_2}} + \sqrt{2}\eps_3 \le 2\sqrt{d_h} +\sqrt{2} \le 3\sqrt{2d_h}$.
\end{proof}
Wrapping things up, by combining Claim~\ref{clm:grad_idp_2} and Claim~\ref{clm:grad_idp_3}, we have w.p. $\ge 1-\delta_2$, for any pair of $(\xone,\xtwo)\in \{(\xa,\xa),(\xa,\xb),(\xb,\xb))\}$, 
\begin{equation}\label{eqn:clm3}
\begin{split}
&\left|\frac{2}{d_h} \back^{(h+1)}(\xone) ^\top\left(\mat W^{(h+1)} \right)\mat{D}^{(h)}(\xone)\mat{D}^{(h)}(\xtwo) \left(\mat W^{(h+1)} \right)^\top \back^{(h+1)}(\xtwo) -\right.\\
 & \left.\frac{2}{d_h} \back^{(h+1)}(\xone) ^\top \mat W^{(h+1)} \Pi^\perp_{\mat G}\mat D \Pi^\perp_{\mat G} \left(\mat W^{(h+1)} \right)^\top \back^{(h+1)}(\xtwo)  \right|\\
 \le & \norm{ \frac{2}{d_h} \back^{(h+1)}(\xone) ^\top \mat W^{(h+1)} \Pi_{\mat G}\mat D } \cdot \norm{\mat D \Pi^\perp_{\mat G} \left(\mat W^{(h+1)} \right)^\top \back^{(h+1)}(\xtwo) }\\
 + & \norm{ \frac{2}{d_h} \back^{(h+1)}(\xone) ^\top \mat W^{(h+1)} \Pi^\perp_{\mat G} \mat D } \cdot \norm{\mat D  \Pi_{\mat G} \left(\mat W^{(h+1)} \right)^\top \back^{(h+1)}(\xtwo) }\\
 +& \norm{ \frac{2}{d_h} \back^{(h+1)}(\xone) ^\top \mat W^{(h+1)} \Pi_{\mat G}} \cdot \norm{ \Pi_{\mat G} \left(\mat W^{(h+1)} \right)^\top \back^{(h+1)}(\xtwo)  }\\
 \le & \left(12\sqrt{2}\sqrt{\frac{\ln\frac{8}{\delta_2}}{d_h}} +12 \eps_3\right)
 + \left(12\sqrt{2}\sqrt{\frac{\ln\frac{8}{\delta_2}}{d_h}} +12 \eps_3\right)
 + \left(12\sqrt{2}\sqrt{\frac{\ln\frac{8}{\delta_2}}{d_h}} +12 \eps_3\right)\\
 =& 36\sqrt{\frac{2\ln\frac{8}{\delta_2}}{d_h}} +36 \eps_3.
\end{split}
\end{equation}

Using Equation~\eqref{eqn:clm3} together with Claim~\ref{clm:grad_idp_1} and Claim~\ref{clm:grad_idp_2}, we've finished the proof for Lemma~\ref{lem:goodB}.

\end{proof}

\section{Proof of Theorem~\ref{thm:ntk_main}}
\label{appsec:main_proof}
In this section, we prove Theorem~\ref{thm:ntk_main}.
At a high level, our proof first reduces the bounding the perturbation on the prediction to bounding perturbations on the kernel values between each pair of training points and between the testing point and each training point.
We use the following notations.
We let $\mat{X} \in \mathbb{R}^{n \times d}$ be the training data.
We define $\ker_t(\vect{x}_{te},\mat{X}) \in \mathbb{R}^{n}$ as \begin{align*}
\left[\ker_t\left(\vect{x}_{te},\mat{X}\right)\right]_i = \left\langle\frac{\partial f\left(\vect{\theta}(t),\vect{x}_{te}\right)}{\partial \vect{\theta}},\frac{\partial f\left(\vect{\theta}(t),\vect{x}_i\right)}{\partial \vect{\theta}}\right\rangle
\end{align*}
i.e., the kernel induced from the gradient of the prediction with respect to the  parameters of the neural network at time $t$.

We also use the following notations for NTK.
We let $\trainker^* \in \mathbb{R}^{n \times n}$ be the \emph{fixed} kernel matrix defined in Equation~\eqref{eqn:linear-dynamics}.  
We let $\ker_{ntk}\left(\vect{x}_{te},\mat{X}\right) \in \mathbb{R}^{n}$ be the kernel values between $\vect{x}_{te}$ and each training data.
Note with this notation, we can write \begin{align}
f_{ntk}(\vect{x}_{te}) = \ker_{ntk}(\vect{x}_{te},\mat{x})^\top \left(\trainker^*\right)^{-1} \vect{y}. \label{eqn:ntk_pred}
\end{align}

We prove a lemma to reduce the prediction perturbation bound to the kernel perturbation bound.

\begin{lem}[Kernel Value Perturbation $\Rightarrow$ Prediction Perturbation]
\label{lem:kernel_perb_pred_perb}
Fix $\epsilon_{\trainker} \le \frac{1}{2} \lambda_0$.
Suppose $\abs{f_{nn}(\params(0),\vect{x}_i)} \le \epsilon_{init}$ for $i=1,\ldots,n$ and $\abs{f_{nn}\left(\params(0),\vect{x}_{te}\right)} \le \epsilon_{init}$ and $\norm{\vect{u}_{nn}(0)-\vect{y}}_2 = O\left(\sqrt{n}\right)$.
Furthermore, if  for all $t \ge 0$ $\norm{\ker_{ntk}(\vect{x}_{te},\mat{X})-\ker_t(\vect{x}_{te},\mat{X})}_2 \le \epsilon_{test}$ and $\norm{\trainker^*-\trainker(t)}_2 \le \epsilon_{\trainker}$, then we have \begin{align*}
\abs{f_{ntk}(\vect{x}_{te}) - f_{nn}(\vect{x}_{te})} \le
O\left(\epsilon_{init} + \frac{\sqrt{n}}{\lambda_0} \epsilon_{test} + \frac{\sqrt{n}}{\lambda_0^2}\log\left(\frac{n}{\epsilon_{\trainker}\lambda_0\kappa}\right)\epsilon_{\trainker}\right).
\end{align*}
\end{lem}
\begin{proof}
Our proof relies a careful analysis on the trajectories induced by gradient flows for optimizing the neural network and the NTK predictor.

Note while Equation~\eqref{eqn:ntk_pred} is a closed-form formula, we can rewrite it in an integral form using the following observations.
For any $\vect{x} \in \mathbb{R}^{d}$, we let $\phi(\vect{x})$ be the feature map induced by NTK.
Note the expression in Equation~\eqref{eqn:ntk_pred} can be rewritten as $f_{ntk}(\vect{x}_{te}) = \kappa\phi(\vect{x}_{te})^\top \vect{\paramsker}_{ntk}$ where $\vect{\paramsker}_{ntk}$ satisfies \begin{align*}
&\min_{\vect{\paramsker}} \norm{\vect{\paramsker}}_2\\
\text{such that } &\kappa\phi(\vect{x}_i)^\top \paramsker = y_i \text{ for }i=1,\ldots,n.
\end{align*}
The solution to this program can written as applying gradient flow on \begin{align*}
\min_{\vect{\paramsker}} \sum_{i=1}^{n} \frac{1}{2n} \left(\kappa\phi(\vect{x}_i)^\top \vect{\paramsker}-y_i\right)^2
\end{align*}
with initialization $\vect{\paramsker}(0)=\vect{0}$.
We use $\vect{\paramsker}(t)$ to denote this parameter at time $t$ trained by gradient flow and $f_{ntk}\left(\vect{x}_{te},\vect{\paramsker}(t)\right)$ be the predictor for $\vect{x}_{te}$ at time $t$.
With these notations, we rewrite \begin{align*}
f_{ntk}(\vect{x}_{te}) = \int_{t=0}^{\infty}\frac{d f_{ntk}(\vect{\paramsker}(t),\vect{x}_{te})}{dt} dt
\end{align*}
where we have used the fact that the initial prediction is $0$.
Now we take a closer look at the time derivative:\begin{align*}
\frac{d f_{ntk}(\vect{\paramsker}(t),\vect{x}_{te})}{dt} = & \left\langle \frac{\partial f(\paramsker(t),\vect{x}_{te})}{\partial \paramsker(t)}, \frac{d \paramsker(t)}{dt}\right\rangle \\
= & \left\langle \frac{\partial f(\paramsker(t),\vect{x}_{te})}{\partial \paramsker(t)}, -\frac{\partial L(\paramsker(t),\{\vect{x}_i\}_{i=1}^n)}{\partial \paramsker(t)}\right\rangle \\
= & - \frac{1}n\left\langle \frac{\partial f(\paramsker(t),\vect{x}_{te})}{\partial \paramsker(t)}, \sum_{i=1}^{n}(u_{ntk,i}(t)-y_i)\frac{\partial f(\paramsker(t),\vect{x}_i)}{\paramsker(t)}\right\rangle \\
= & - \frac{1}{n}\left\langle\kappa \phi(\vect{x}_{te}), \sum_{i=1}^{n}(u_{ntk,i}(t)-y_i)\kappa\phi\left(\vect{x}_i\right)\right\rangle \\
= & - \frac{\kappa^2}{n} \kernel_{ntk}(\vect{x}_{te},\mat{X})^\top \left(\vect{u}_{ntk}(t)-\vect{y}\right)
\end{align*}
where $u_{ntk,i}(t)=f_{ntk}(\paramsker(t),\vect{x}_i)$ and $\vect{u}_{ntk}(t) \in \mathbb{R}^n$ with  $[\vect{u}_{ntk}(t)]_i = u_{ntk,i}(t)$.
Similarly, for the NN predictor, we can obtain a time derivative of the same form.
\begin{align*}
\frac{d f_{nn}(\vect{\params}(t),\vect{x}_{te})}{dt} = & \kappa \left\langle \frac{\partial f(\params(t),\vect{x}_{te})}{\partial \params(t)}, \frac{d \params(t)}{dt}\right\rangle \\
= & \kappa\left\langle \frac{\partial f(\params(t),\vect{x}_{te})}{\partial \params(t)}, -\frac{\partial L(\params(t),\{\vect{x}_i\}_{i=1}^n)}{\partial \params(t)}\right\rangle \\
= & - \frac{\kappa^2}{n}\left\langle \frac{\partial f(\params(t),\vect{x}_{te})}{\partial \params(t)}, \sum_{i=1}^{n}(u_{nn,i}(t)-y_i)\frac{\partial f(\params(t),\vect{x}_i)}{\params(t)}\right\rangle \\
= & - \frac{\kappa^2}{n} \kernel_{t}(\vect{x}_{te},\mat{X})^\top \left(\vect{u}_{nn}(t)-\vect{y}\right)
\end{align*}

We thus we analyze the difference between the NN predictor and NTK predictor via this integral form
\begin{align*}
&\abs{f_{nn}(\vect{x}_{te}) - f_{ntk}\left(\vect{x}_{te}\right)}\\
=  &\abs{f_{nn}(\params(0),\vect{x}_{te})+\int_{t =0}^{\infty} \left(\frac{d f_{nn}(\vect{\params}(t),\vect{x}_{te})}{dt} - \frac{d f_{ntk}(\paramsker(t),\vect{x}_{te})}{dt}\right) dt} \\
=&\abs{f_{nn}\left(\params(0),\vect{x}_{te}\right)}+\abs{-\frac{\kappa^2}{n}\int_{t=0}^{\infty} \left(\kernel_t(\vect{x}_{te},\mat{X})^\top (\vect{u}_{nn}(t)-\vect{y}) - \kernel_{ntk}(\vect{x}_{te},\mat{X})^\top (\vect{u}_{ntk}(t)-\vect{y})\right) dt}\\
\le & \epsilon_{init}+\frac{\kappa^2}{n}\abs{\int_{t=0}^\infty \left(\kernel_{t}(\vect{x}_{te},\mat{X})-\kernel_{ntk}(\vect{x}_{te},\mat{X})\right)^\top (\vect{u}_{nn}(t)-\vect{y})dt} \\
&+\frac{\kappa^2}{n} \abs{\int_{t=0}^\infty \kernel_{ntk}(\vect{x}_t,\mat{X})^\top (\vect{u}_{nn}(t)-\vect{u}_{ntk}(t)) dt } \\
\le & \epsilon_{init}+\kappa^2\max_{0\le t \le \infty} \norm{\kernel_{t}(\vect{x}_{te},\mat{X})-\kernel_{ntk}(\vect{x}_{te},\mat{X})}_2 \int_{t=0}^\infty \norm{\vect{u}_{nn}(t)-\vect{y}}_2 dt\\
& + \kappa^2\max_{0\le t \le \infty} \norm{\kernel_{ntk}(\vect{x}_{te},\mat{X})}_2 \int_{t=0}^\infty \norm{\vect{u}_{nn}(t)-\vect{u}_{ntk}(t)}_2 dt \\
\le & \epsilon_{init}+\kappa^2\epsilon_{test}  \int_{t=0}^\infty \norm{\vect{u}_{nn}(t)-\vect{y}}_2 dt+ \kappa^2 \max_{0\le t \le \infty} \norm{\kernel_{ntk}(\vect{x}_{te},\mat{X})}_2 \int_{t=0}^\infty \norm{\vect{u}_{nn}(t)-\vect{u}_{ntk}(t)}_2 dt 
\end{align*}

For the second term, recall $\norm{\trainker^*-\trainker(t)}_2 \le \epsilon_{\trainker}$ by our assumption so $\lambda_{\min}\left(\trainker(t)\right) \ge \frac{1}{2}\lambda_0$.
Using this fact we know $\norm{\vect{u}_{nn}(t)-\vect{y}}_2 \le \exp(-\frac{\kappa^2}{2}\lambda_0t)\norm{\vect{u}_{nn}(0)-\vect{y}}_2$.
Therefore, we can bound
\[\int_{0}^\infty\norm{\vect{u}_{nn}(t)-\vect{y}}_2 dt = \int_{t=0}^\infty \exp(-\frac{\kappa^2}{2}\lambda_0t)\norm{\vect{u}_{nn}(0)-\vect{y}}_2 dt = O\left(\frac{\sqrt{n}}{\kappa^2\lambda_0}\right).
\]

To bound $\int_{t=0}^\infty \norm{\vect{u}_{nn}(t)-\vect{u}_{ntk}(t)}_2$, we observe that $\vect{u}_{nn}(t) \rightarrow \vect{y}$ and $\vect{u}_{ntk}(t) \rightarrow \vect{y}$ with linear convergence rate.
Therefore, we can choose some  $t_0 = \frac{C}{\lambda_0\kappa^2}\log \left(\frac{n}{\epsilon_{\trainker} \lambda_0\kappa}\right)$ so that \begin{align*}
&\int_{t_0}^\infty\norm{\vect{u}_{nn}(t)-\vect{u}_{ntk}(t)}_2 dt \\
\le &\int_{t_0}^\infty\norm{\vect{u}_{nn}(t)-\vect{y}}_2dt + \int_{t_0}^\infty\norm{\vect{u}_{ntk}(t)-\vect{y}}_2 dt\\
\le & O\left(\frac{1}{\lambda_0\kappa^2} \left(\norm{\vect{u}_{nn}(t_0)-\vect{y}}_2+\norm{\vect{u}_{ntk}(t_0)-\vect{y}}_2\right)\right)\\
\le & O\left(\frac{\sqrt{n}}{\lambda_0\kappa} \exp\left(-\lambda_0\kappa^2t_0\right)\right)\\
\le & O(\epsilon_{\trainker}).
\end{align*}
Thus it suffices to bound \begin{align*}
	\int_{t=0}^{t_0}\norm{\vect{u}_{nn}(t)-\vect{u}_{ntk}(t)}_2 dt \le t_0 \max_{0\le t\le t_0} \norm{\vect{u}_{nn}(t)-\vect{u}_{ntk}(t)}_2.
\end{align*}
First observe that \begin{align*}
 &\norm{\vect{u}_{nn}(t)-\vect{u}_{ntk}(t)}_2 \\
 \le &\norm{\vect{u}_{nn}(0)}_2+\int_{\tau = 0}^t \norm{\frac{d\left(\vect{u}_{nn}(\tau)-\vect{u}_{ntk}(\tau)\right)}{d \tau}}_2 d\tau \\
 \le &\epsilon_{init}\sqrt{n} + \int_{\tau = 0}^t \norm{\frac{d\left(\vect{u}_{nn}(\tau)-\vect{u}_{ntk}(\tau)\right)}{d \tau}}_2 d\tau .
 \end{align*}
 Note \begin{align*}
 &\frac{d \left(\vect{u}_{nn}(t)-\vect{u}_{ntk}(t)\right)}{dt} \\
 = &-\kappa^2\trainker(t)\left(\vect{u}_{nn}(t)-\vect{y}\right) + \kappa^2 \trainker^*\left(\vect{u}_{ntk}(t)-\vect{y}\right) \\
 = &  -\kappa^2\trainker^*\left(\vect{u}_{nn}(t)-\vect{u}_{ntk}(t)\right) + \kappa^2\left(\trainker^*-\trainker(t)\right)\left(\vect{u}_{nn}(t)-\vect{y}\right)
 \end{align*}
Since $\trainker^*$ is positive semidefinite, $-\trainker^*\left(\tilde{\vect{u}}(t)-\vect{u}(t)\right) $ term only makes $\norm{\tilde{\vect{u}}(t)-\vect{u}(t)}_2$ smaller.
Therefore, we have \begin{align*}
\norm{\vect{u}_{nn}(t)-\vect{u}_{ntk}(t)}_2 \le &\kappa^2\int_{\tau =0}^t \norm{\vect{u}_{nn}(\tau)-\vect{y}}_2 \norm{\trainker(t)-\trainker^*}_2 \\
\le & t\kappa^2\norm{\vect{u}_{nn}(0)-\vect{y}}_2 \epsilon_{\trainker} \\
\le & O\left(t \kappa^2\sqrt{n}\epsilon_{\trainker}\right).
\end{align*}
Therefore, we have \begin{align*}
	\int_{t=0}^{t_0}\norm{\vect{u}_{nn}(t)-\vect{u}_{ntk}(t)}_2 dt
	\le O\left( t_0^2 \sqrt{n} \kappa^2\epsilon_{\trainker} \right)= O\left(\frac{\sqrt{n}}{\lambda_0^2\kappa^2}\log\left(\frac{n}{\epsilon_{\trainker}\lambda_0\kappa}\right)\epsilon_{\trainker}\right).
\end{align*}

Lastly, we put things together and get \begin{align*}
\abs{f_{ntk}(\vect{x}_{te}) - f_{nn}(\vect{x}_{te})} \le
O\left(\epsilon_{init} + \epsilon_{test} \frac{\sqrt{n}}{\lambda_0} + \frac{\sqrt{n}}{\lambda_0^2}\log\left(\frac{n}{\epsilon_{\trainker}\lambda_0\kappa}\right)\epsilon_{\trainker}\right). 
\end{align*}
\end{proof}

\begin{proof}[Proof of Theorem~\ref{thm:ntk_main}]
By Lemma~\ref{lem:kernel_perb_pred_perb}, the problem now reduces to (i) choose $\kappa$ small enough to make $\epsilon_{init} = O(\eps)$ and (ii) show when the width is large enough then $\epsilon_{\trainker}$ and $\epsilon_{test}$ are both $O(\epsilon)$.
For (i), based on Theorem~\ref{thm:GoodA} and the union bound,  we can just choose $\kappa = O\left(\frac{\epsilon}{\log(n/\delta)}\right)$ to make $\epsilon_{init} = O(\epsilon)$ with probability $1-\delta$.
For (ii), we will use Theorem~\ref{thm:ntk_init} and Lemma~\ref{lem:ker_perb_train} below, and then apply the union bound.
\end{proof}

\subsection{Kernel Perturbation During Training}

In this subsection we prove the following lemma.
\begin{lem}[Kernel Perturbation Bound During Training]
	\label{lem:ker_perb_train}
	Fix $\omega \le\poly(1/L,1/n,1/\log(1/\delta), \lambda_0)$. 	
	Suppose we set $m \ge \poly(1/\omega)$ and $\kappa \le 1$.
	Then with probability at least $1-\delta$ over random initialization, we have for all $t \ge 0$, for any $(\vect{x},\vect{x}') \in \left\{\vect{x}_1,\ldots,\vect{x}_n,\vect{x}_{te}\right\} \times \left\{\vect{x}_1,\ldots,\vect{x}_n,\vect{x}_{te}\right\}$\begin{equation*}
	\abs{
\kernel_{t}\left(\vect{x},\vect{x}'\right)- \kernel_{0}\left(\vect{x},\vect{x}'\right)
	} \le \omega
	\end{equation*}
\end{lem}


Recall for any fixed $\vect{x}$ and $\vect{x}'$, Theorem~\ref{thm:ntk_init} shows $\abs{\kernel_{0}(\vect{x},\vect{x}')-\kernel_{ntk}(\vect{x},\vect{x}')} \le \epsilon$ if $m$ is large enough.
The next lemma shows we can reduce the problem of bounding the perturbation on the kernel value to the perturbation on the gradient.

\begin{lem}[Gradient Perturbation $\Rightarrow$ Kernel Perturbation]
	\label{lem:grad_to_kernel}
	If $\norm{\frac{\partial f(\params(t),\vect{x})}{\partial \params} - \frac{\partial f(\params(0),\vect{x})}{\partial \params}} \le \epsilon$ and 
	$\norm{\frac{\partial f(\params(t),\vect{x}')}{\partial \params} - \frac{\partial f(\params(0),\vect{x}')}{\partial \params}} \le \epsilon$, we have \begin{align*}
		\abs{\kernel_{t}(\vect{x},\vect{x}') - \kernel_{0}(\vect{x},\vect{x}')} \le O\left(\epsilon\right)
	\end{align*}
\end{lem}
\begin{proof}
By the proof of Theorem~\ref{thm:ntk_init}, we know $\norm{\frac{\partial f(\params(0),\vect{x})}{\partial \params}}_2 = O\left(1\right)$.
Then we can just use triangle inequality.
\end{proof}

Now we proceed to analyze the perturbation on the gradient.
Note we can focus on the perturbation on a single sample $\vect{x}$ because we can later take a union bound.
Therefore, in the rest of this section, we drop the dependency on a specific sample.
We use the following notations in this section.
Recall $\mat{W}^{(1)},\ldots,\mat{W}^{(L+1)} \sim \gauss\left(\mat{0},\mat{I}\right)$ and we denote $\diff \mat{W}^{(1)},\ldots,\diff \mat{W}^{(L+1)}$ the perturbation matrices.
We let $\widetilde{\mat{W}}^{(h)} = \mat{W}^{(h)} + \diff\mat{W}^{(h)}$.
We let $\tilde{\vect{g}}^{(0)} = \vect{g}^{(0)}=\vect{x}$ and
for $h=1,\ldots,L$ we define\begin{align*}
\vect{z}^{(h)} = &\sqrt{\frac{2}{m}}\mat{W}^{(h)} 
\vect{g}^{(h-1)}, ~~~ \vect{g}^{(h)} = \relu{\vect{z}^{(h)}},\\
\tilde{\vect{z}}^{(h)} = &\sqrt{\frac{2}{m}}\widetilde{\mat{W}}^{(h)}
\tilde{\vect{g}}^{(h-1)}, ~~~ \tilde{\vect{g}}^{(h)} = \relu{\tilde{\vect{z}}^{(h)}}.
\end{align*}
For $h=1,\ldots,L$, $i=1,\ldots,m$, we denote \begin{align*}
[\mat{D}^{(h)}]_{ii} =
 &\indict\left\{\left[\mat{W}^{(h)}\right]_{i,:}\vect{g}^{(h-1)}\ge 0\right\}\\
[\widetilde{\mat{D}}^{(h)}]_{ii} =
 &\indict\left\{\left[\widetilde{\mat{W}}^{(h)}\right]_{i,:}\tilde{\vect{g}}^{(h-1)}\ge 0\right\}.
\end{align*}
\begin{rem}
Note $\vect{z}^{(h)} = \sqrt{\frac{2}{m}}\vect{f}^{(h)}$. Here we use $\vect{z}^{(h)}$ instead of  $\vect{f}^{(h)}$ for the ease of presentation.
\end{rem}

For convenience, we also define \begin{align*}
\diff \mat{D}^{(h)} = \widetilde{\mat{D}}^{(h)} - \mat{D}^{(h)}.
\end{align*}
Recall the gradient to $\mat{W}^{(h)}$ is: \begin{align*}
\frac{\partial f(\params,\vect{x})}{\partial \mat{W}^{(h)}} = \vect{b}^{(h)} \left(\vect{g}^{(h-1)}\right)^\top
\end{align*}
Similarly, we have \begin{align*}
\frac{\partial f(\params,\vect{x})}{\partial \widetilde{\mat{W}}^{(h)}} = \tilde{\vect{b}}^{(h)} \left(\tilde{\vect{g}}^{(h-1)}\right)^\top
\end{align*}
where
\begin{align*}
\tilde{\vect{b}}^{(h)} = \begin{cases}
1 &\text{ if } h = L+1\\
\sqrt{\frac{2}{m}}\widetilde{\mat{D}}^{(h)}\left(\widetilde{\mat{W}}^{(h+1)}\right)^\top\tilde{\vect{b}}^{(h+1)} &\text{ Otherwise}
\end{cases} .
\end{align*}
This gradient formula allows us to bound the perturbation on $\diff \vect{g}^{(h)}\triangleq\tilde{\vect{g}}^{(h)} - \vect{g}^{(h)}$ and $\diff \vect{b}^{(h)}\triangleq\tilde{\vect{b}}^{(h)}-\vect{b}^{(h)}$ separately.
The following lemmas adapted from \citep{allen2018convergence} show with high probability over the initialization, bounding the perturbation on $\diff \vect{g}^{(h)}$ and $\diff \vect{b}^{(h)}$ can be reduced to bounding the perturbation on weight matrices.
\begin{lem}[Adapted from Lemma 5.2 in \citep{allen2018convergence}]
\label{lem:forward_perturbation}
Suppose \[\omega\le \poly\left(1/n,\lambda_0,1/L,1/\log(m),\epsilon,1/\log(1/\delta)\right).\]
Then with probability at least $1-\delta$ over random initialization, if   $\norm{\diff \mat{W}^{(h)}}_2 \le \sqrt{m}\omega$ for all $h=1,\ldots,L$, we have $\norm{\diff \vect{g}^{(h)}}_2 = O(\omega L^{5/2}\sqrt{\log m})$ for all $h=1,\ldots,L$.

\end{lem}

\begin{rem}
While \citet{allen2018convergence} did not consider the perturbation on $\mat{W}^{(1)}$, by scrutinizing their proof, it is easy to see that the perturbation bounds still hold even if there is a small perturbation on $\mat{W}^{(1)}$.
\end{rem}

The next lemma bounds the backward vector, adapted from 
\begin{lem}[Adapted  from Lemma 5.7 in \citep{allen2018convergence}]
\label{lem:back_ward_perturbation}
Suppose \[\omega\le \poly\left(1/n,\lambda_0,1/L,1/\log(m),\epsilon,1/\log(1/\delta)\right).\]
Then with probability at least $1-\delta$ over random initialization,if  $\norm{\diff \mat{W}^{(h)}}_2 \le \sqrt{m}\omega$ for all $h=1,\ldots,L+1$, we have for all $h=1,\ldots,L+1$, $
	\norm{\tilde{\vect{b}}^{(h)}-\vect{b}^{(h)}}_2 = O\left(\omega^{1/3}L^2\sqrt{\log m}\right)$.
\end{lem}

\begin{rem}
	While \citet{allen2018convergence} did not consider the perturbation on $\mat{W}^{(L+1)}$, by scrutinizing their proof, it is easy to see that the perturbation bounds still hold even if there is a small perturbation on $\mat{W}^{(L+1)}$.
\end{rem}

Combing these two lemmas and the result for the initialization (Theorem~\ref{thm:ntk_init}), we have the following ``gradient-Lipschitz" lemma.
\begin{lem}
\label{lem:grad_lip}
Suppose $\omega \le \poly\left(1/n,\lambda_0,1/L,1/\log(m),\epsilon,1/\log(1/\delta)\right).$
Then with probability at least $1-\delta$ over random initialization, if  $\norm{\diff \mat{W}^{(h)}}_2 \le \sqrt{m}\omega$ for all $h=1,\ldots,L+1$, we have for all $h=1,\ldots,L+1$: \begin{align*}
	\norm{\tilde{\vect{b}}^{(h)}\left(\tilde{\vect{g}}^{(h-1)}\right)^\top-\vect{b}^{(h)}\left(\vect{g}^{(h-1)}\right)^\top}_F = O\left(\omega^{1/3}L^{5/2}\sqrt{\log m}\right)
\end{align*}
\end{lem}
\begin{proof}
We use the triangle inequality to bound the perturbation
\begin{align*}
	&\norm{\tilde{\vect{b}}^{(h)}\left(\tilde{\vect{g}}^{(h-1)}\right)^\top-\vect{b}^{(h)}\left(\vect{g}^{(h-1)}\right)^\top}_F \\
\le&  \norm{\tilde{\vect{b}}^{(h)}\left(\tilde{\vect{g}}^{(h-1)}\right)^\top-\vect{b}^{(h)}\left(\tilde{\vect{g}}^{(h-1)}\right)^\top}_F  + \norm{\vect{b}^{(h)}\left(\tilde{\vect{g}}^{(h-1)}\right)^\top-\vect{b}^{(h)}\left(\vect{g}^{(h-1)}\right)^\top}_F \\
\le & \norm{
\diff \vect{b}^{(h)}\left(\vect{g}^{(h-1)}+\diff \vect{g}^{(h-1)}\right)^\top
}_F + \norm{\vect{b}^{(h)}\left(\diff \vect{g}^{(h-1)}\right)^\top}_F\\
= & O\left(\omega^{1/3}L^{5/2}\sqrt{\log m}\right).
\end{align*}
\end{proof}

The following lemma shows for given weight matrix, if we have linear convergence and other weight matrices are only perturbed by a little, then the given matrix is only perturbed by a little as well.
\begin{lem}
\label{lem:lin_conv_w_perb_small}
Fix  $h \in [L+1]$ and a sufficiently small $\omega \le \poly\left(1/n,\lambda_0,1/L,1/\log(m),\epsilon,1/\log(1/\delta),\kappa\right).$
Suppose for all $t \ge 0$, $\norm{\vect{u}_{nn}(t)-\vect{y}}_2 \le \exp\left(-\frac12\kappa^2\lambda_0t\right)\norm{\vect{u}_{nn}(0)-\vect{y}}_2$ and $\norm{\mat{W}^{(h')}(t)-\mat{W}^{(h')}(0)}_F \le \omega \sqrt{m}$ for $h'\neq h$.
Then if $m \ge \poly\left(1/\omega\right)$ we have with probability at least $1-\delta$ over random initialization, for all $t \ge 0$ \begin{align*}
\norm{\mat{W}^{(h)}(t)-\mat{W}^{(h)}(0)}_F = O \left( \frac{\sqrt{n}}{\lambda_0}\right)\le\omega\sqrt{m}.
\end{align*}
\end{lem}
\begin{proof}
We let $C,C_0, C_1, C_2, C_3 > 0$ be some absolute constants.
\begin{align*}
&\norm{\mat{W}^{(h)}(t)-\mat{W}^{(h)}(0)}_F\\
 = &\norm{\int_{0}^{t}\frac{d \mat{W}^{(h)}(\tau)}{d\tau} d\tau}_F \\
=& \norm{\int_{0}^{t}\frac{\partial L(\params(\tau))}{\partial \mat{W}^{(h)}(\tau)} d\tau}_F \\
= & \norm{\int_{0}^{t}\frac{1}{n}\sum_{i=1}^n \left(u_i(\tau)-y_i\right) \frac{\partial f_{nn}(\params(\tau),\vect{x}_i)}{\partial \mat{W}^{(h)}} d\tau}_F \\
\le & \frac{1}{n}\max_{0\le \tau\le t}\sum_{i=1}^n\norm{ \frac{\partial f_{nn}(\params(\tau),\vect{x}_i)}{\partial \mat{W}^{(h)}}}_F\int_{0}^{t} \norm{\vect{u}_{nn}(\tau)-\vect{y}}_2 d\tau\\
\le & \frac{1}{n}\max_{0\le \tau\le t}\sum_{i=1}^n\norm{ \frac{\partial f_{nn}(\params(\tau),\vect{x}_i)}{\partial \mat{W}^{(h)}}}_F\int_{0}^{t} \exp\left(-\kappa^2\lambda_0\tau\right) d\tau\norm{\vect{u}_{nn}(0)-\vect{y}}_2 \\
\le & \frac{C_0}{\sqrt{n}\lambda_0}\max_{0\le \tau\le t}\sum_{i=1}^n\norm{ \frac{\partial f_{nn}(\params(\tau),\vect{x}_i)}{\partial \mat{W}^{(h)}}}_F &\\
\le &\frac{C_0}{\sqrt{n}\kappa^2\lambda_0}\max_{0\le \tau\le t}\sum_{i=1}^n\left(\norm{ \frac{\partial f_{nn}(\params(0),\vect{x}_i)}{\partial \mat{W}^{(h)}}}_F + \norm{\frac{\partial f_{nn}(\params(\tau),\vect{x}_i)}{\partial \mat{W}^{(h)}}- \frac{\partial f_{nn}(\params(0),\vect{x}_i)}{\partial \mat{W}^{(h)}}}_F\right)\\
\le &\frac{C_1}{\sqrt{n}\lambda_0}\max_{0\le \tau\le t}\sum_{i=1}^n\left(\norm{ \frac{\partial f_{nn}(\params(0),\vect{x}_i)}{\partial \mat{W}^{(h)}}}_F + \norm{\frac{\partial f_{nn}(\params(\tau),\vect{x}_i)}{\partial \mat{W}^{(h)}}- \frac{\partial f_{nn}(\params(0),\vect{x}_i)}{\partial \mat{W}^{(h)}}}_F\right) \\
\le &\frac{C_2\sqrt{n}}{\lambda_0}+\frac{C_1\sqrt{n}}{\lambda_0}\max_{0\le \tau\le t}\left( \norm{\frac{\partial f_{nn}(\params(\tau),\vect{x}_i)}{\partial \mat{W}^{(h)}}- \frac{\partial f_{nn}(\params(0),\vect{x}_i)}{\partial \mat{W}^{(h)}}}_F\right).
\end{align*}
The last step we used $\norm{ \frac{\partial f_{nn}(\params(0),\vect{x}_i)}{\partial \mat{W}^{(h)}}}_F = O(1)$.
Suppose there exists $t$ such that $\norm{\mat{W}^{(h)}(t)-\mat{W}^{(h)}(0)}_F > \omega\sqrt{m}.$
Denote \[t_0 = \argmin_{t\ge 0}\left\{\norm{\mat{
W
}^{(h)}(t)-\mat{W}^{(h)}(0)}_F > \omega\sqrt{m}.
\right\}.\]
For any $t < t_0$, we know for all $h' \in [L+1]$, 
$\norm{\mat{W}^{(h')}(t)-\mat{W}^{(h')}(0)}_2 \le \omega\sqrt{m}.$
Therefore, by Lemma~\ref{lem:grad_lip}, we know \begin{align*}
\norm{\frac{\partial f_{nn}(\params(t),\vect{x}_i)}{\partial \mat{W}^{(h)}}- \frac{\partial f_{nn}(\params(0),\vect{x}_i)}{\partial \mat{W}^{(h)}}}_F = C \omega^{1/3}L^{5/2}.
\end{align*}
Therefore, using the fact that $\omega$ is sufficiently small we can bound \begin{align*}
\norm{\mat{W}^{(h)}(t_0)-\mat{W}^{(h)}(0)}_F \le \frac{C_3\sqrt{n}}{\lambda_0}.
\end{align*}
Since we also know $m$ is sufficiently large to make $\omega\sqrt{m} > \frac{C_3\sqrt{n}}{\lambda_0}$, we have a contradiction.
\end{proof}

The next lemma shows if all weight matrices only have small perturbation, then we still have linear convergence.
\begin{lem}
\label{lem:w_perb_small_lin_conv}
Suppose $\omega = \poly\left(1/n,\lambda_0,1/L,1/\log(m),\epsilon,1/\log(1/\delta),\kappa\right).$
Suppose for all $t \ge 0$ $\norm{\mat{W}^{(h)}(t)-\mat{W}^{(h)}(0)}_F \le \omega \sqrt{m}$ for $h \in [L+1]$.
Then if $m = \poly\left(1/\omega\right)$, we have with probability at least $1-\delta$ over random initialization, for all $t \ge 0$ \begin{align*}
\norm{\vect{u}_{nn}(t)-\vect{y}}_2 \le \exp\left(-\frac{1}{2}\kappa^2\lambda_0t\right)\norm{\vect{u}_{nn}(0)-\vect{y}}_2.
\end{align*}
\end{lem}
\begin{proof}
Under this assumption and the result of initialization, we know for all $t \ge 0$, $\lambda_{\min}\left(\trainker(t)\right) \ge \frac{1}{2}\lambda_0$.
This in turn directly imply the linear convergence result we want.
\end{proof}

Lastly, with these lemmas at hand, using an argument similar to \citep{du2018provably}, we can show during training, weight matrices do not move by much.
\begin{lem}
\label{lem:cont_induction}
Let $\omega \le \poly(\eps,L,\lambda_0,1/\log(m),1/\log(1/\delta),\kappa, 1/n)$.
If $m \ge \poly(1/\omega)$, then with probability at least $1-\delta$ over random initialization, we have for all $t \ge 0$, for all $h \in [L+1]$ we have \begin{align*}
\norm{\mat{W}^{(h)}(t)-\mat{W}^{(h)}(0)}_F \le \omega\sqrt{m}
\end{align*}
and \begin{align*}
\norm{\vect{u}_{nn}(t)-\vect{y}}_2 \le \exp\left(-\frac{1}{2}\kappa^2\lambda_0t\right)\norm{\vect{u}_{nn}(0)-\vect{y}}_2.
\end{align*}
\end{lem}
\begin{proof}
Let \begin{align*}t_0 = \argmin_{t}\left\{\exists h \in [L+1], \norm{\mat{W}^{(h)}(t)-\mat{W}^{(h)}(0)}_F > \omega\sqrt{m} \right.\\
\left.\text{ or }\norm{\vect{u}_{nn}(t)-\vect{y}}_2 > \exp\left(-\frac{1}{2}\kappa^2\lambda_0t\right)\norm{\vect{u}_{nn}(0)-\vect{y}}_2\right\}.
\end{align*}
We analyze case by case.
Suppose at time $t_0$, $\norm{\mat{W}^{(h)}(t_0)-\mat{W}^{(h)}(0)}_F > \omega\sqrt{m}$.
By Lemma~\ref{lem:lin_conv_w_perb_small}, we know there exists some $0\le t_1 < t_0$ such that either there exists $h' \neq h$ such that \[\norm{\mat{W}^{(h')}(t_1)-\mat{W}^{(h')}(0)}_F > \omega\sqrt{m}\] or \[
\norm{\vect{u}_{nn}(t_1)-\vect{y}}_2 > \exp\left(-\frac{1}{2}\kappa^2\lambda_0t_1\right)\norm{\vect{u}_{nn}(0)-\vect{y}}_2.
\]
However, this violates the minimality of $t_0$.
For the other case, if \[
\norm{\vect{u}_{nn}(t_0)-\vect{y}}_2 > \exp\left(-\frac{1}{2}\kappa^2\lambda_0t_0\right)\norm{\vect{u}_{nn}(0)-\vect{y}}_2,
\]
By Lemma~\ref{lem:w_perb_small_lin_conv}, we know there exists $t_1 < t_0$ such that there exists $h \in [L+1]$, \[\norm{\mat{W}^{(h)}(t_1)-\mat{W}^{(h)}(0)}_F > \omega\sqrt{m}.\]
However, again this violates the minimality of $t_0$.
\end{proof}

Now we can finish the proof of Lemma~\ref{lem:ker_perb_train}.

\begin{proof}[Proof of Lemma~\ref{lem:ker_perb_train}]
	By Lemma~\ref{lem:cont_induction}, we know for $t \rightarrow \infty$, $\norm{\mat{W}^{(h)}(t)-\mat{W}^{(h)}(0)}_F \le O\left(\omega\sqrt{m}\right)$ for if $\omega$ is sufficiently.
	Applying Lemma~\ref{lem:grad_lip}, we know we only have a small perturbation on the gradient.
	Applying Lemma~\ref{lem:grad_to_kernel}, we know we only have small perturbation on  kernel values.
\end{proof}

\section{CNTK Derivation}
\label{sec:cntk_derivation}
In this section we derive CNTK for vanilla CNN.
Given $\vect{x} \in \mathbb{R}^{\nnw \times \nnh}$ and $(i,j) \in [\nnw]\times [\nnh]$, we define \[
\phi_{ij}(\vect{x}) = [\vect{x}]_{i-(q-1)/2:i+(q-1)/2,j-(q-1)/2:j+(q-1)/2}
\] i.e., this operator extracts the $(i,j)$-th patch.
By this definition, we can rewrite the CNN definition:
\begin{itemize*}
	\item Let  $\vect{x}^{(0)} =\vect{x} \in \mathbb{R}^{\nnw\times \nnh \times \nnc^{(0)}}$ be the input image where $\nnc^{(0)}$ is the number of channels in the input image.
	\item For $h=1,\ldots,H$, $\beta = 1,\ldots,\nnc^{(h)}$, the intermediate outputs are defined as \begin{align*}
	\left[\tilde{\vect{x}}_{(\beta)}^{(h)}\right]_{ij} = 
	\sum_{\alpha=1}^{\nnc^{(h-1)}}\left\langle\mat{W}_{(\alpha),(\beta)}^{(h)},\vect\phi_{ij}\left({x}_{(\alpha)}^{(h-1)}\right)\right\rangle, \quad
	\vect{x}^{(h)}_{(\beta)} = \sqrt{\frac{c_{\sigma}}{\nnc^{(h)} \times q \times q}}\act{\tilde{\vect{x}}_{(\beta)}^{(h)}}
	\end{align*}
	where each $\mat{W}_{(\alpha),(\beta)}^{(h)} \in \mathbb{R}^{q \times q}$ is a filter with Gaussian initialization.
	\item The final output is defined as \begin{align*}
	f(\params,\vect{x}) = \sum_{\alpha=1}^{\nnc^{(L)}} \left\langle \mat{W}_{(\alpha)}^{(L)},\vect{x}_{(\alpha)}^{(L)}\right\rangle
	\end{align*}
	where $\mat{W}_{(\alpha)}^{(L)} \in \mathbb{R}^{\nnw \times \nnh}$ is a weight matrix with Gaussian initialization.
\end{itemize*}

\subsection{Expansion of CNTK}
We expand $\Theta^{(L)}(\vect{x},\vect{x}')$ to show we can write it as the sum of $(L+1)$ terms with each term representing the inner product between the gradients with respect to the weight matrix of one layer.
We first define an linear operator\begin{align}
	\linop: \mathbb{R}^{\nnw\times\nnh\times\nnw\times\nnh} \rightarrow \mathbb{R}^{\nnw\times\nnh\times\nnw\times\nnh} \nonumber\\
	\left[\linop\left(\mat{M}\right)\right]_{k\ell,k'\ell'} = \frac{c_{\sigma}}{q^2}\tr\left(\left[\mat{M}\right]_{\indset_{k\ell,k'\ell'}}\right). \label{eqn:linop}
\end{align}
This linear operator is induced from convolutional operation. And here use zero padding, namely when the subscription exceeds the range of $[\nnw]\times [\nnh]\times[\nnw]\times [\nnh]$, the value of the element is zero.

We also define $\id\in \mathbb{R}^{\nnw\times\nnh\times\nnw\times\nnh} $ as the identity tensor, namely $\id_{i,j,i',j'}= \bm{1}\{i=i',j=j'\}.$
And \[\Sum{\mat{M}}=\sum_{(i,j,i',j') \in [\nnw] \times [\nnh]\times [\nnw]\times [\nnh] } \mat{M}_{i,j,i',j'}.\]

The following property of $\linop$ is immediate by definition: $\forall \mat{M},\mat{N} \in  \mathbb{R}^{\nnw\times\nnh\times\nnw\times\nnh}$, we have

%

\begin{equation}\label{eq:conjugate2}
\Sum{\mat{M}\odot \linop(\mat{N})}  = \Sum{\linop(\mat{M})\odot \mat{N}}.
\end{equation}

With this operator, we can expand CNTK as (for simplicity we drop  on $\vect{x}$ and $\vect{x}'$)\begin{align*}
&\Theta^{(L)}\\
= &\tr\left(\dot{\mat{K}}^{(L)}\odot\Theta^{(H-1)}+\mat{K}^{(L)}\right) \\
= & \tr\left(\mat{K}^{(L)}\right) + \tr\left(\dot{\mat{K}}^{(L)}\odot \linop\left(\mat{K}^{(H-1)}\right)\right) + \tr\left(\dot{\mat{K}}^{(L)}\odot\linop\left(
\dot{\mat{K}}^{(H-1)} \odot \mat{\Theta}^{(H-2)} 
\right)\right) \\
= & \ldots \\
= & \sum_{h=0}^{L} \tr\left(
\dot{\mat{K}}^{(L)} \odot \linop\left( \dot{\mat{K}}^{(H-1)} \linop \left(\cdots\dot{\mat{K}}^{(h+1)}\linop\left(\mat{K}^{(h)}\right)\cdots\right)\right) .
\right)
\end{align*}
Here for $h=H$, the term is just $\tr\left(\mat{K}^{(L)}\right) $.

In the following, we will show \begin{align*}
	\left\langle \frac{\partial f(\params,\vect{x})}{\partial \mat{W}^{(h)}} , \frac{\partial f(\params,\vect{x}')}{\partial \mat{W}^{(h)}} \right\rangle 
	\approx 
	&\tr\left(
\dot{\mat{K}}^{(L)} \odot \linop\left( \dot{\mat{K}}^{(H-1)} \odot\linop \left(\cdots\dot{\mat{K}}^{(h)} \odot\linop\left(\mat{K}^{(h-1)}\right)\cdots\right)\right) 
\right)\\
=& \Sum{ \id\odot \dot{\mat{K}}^{(L)} \odot \linop\left( \dot{\mat{K}}^{(H-1)} \odot\linop \left(\cdots\dot{\mat{K}}^{(h)} \odot\linop\left(\mat{K}^{(h-1)}\right)\cdots\right)\right) }.
\end{align*}

which could be rewritten as the following by Property~\ref{eq:conjugate2},

\begin{align*}
	\left\langle \frac{\partial f(\params,\vect{x})}{\partial \mat{W}^{(h)}} , \frac{\partial f(\params,\vect{x}')}{\partial \mat{W}^{(h)}} \right\rangle \approx 
	\Sum{
\linop\left(\mat{K}^{(h-1)}\right) \odot  \dot{\mat{K}}^{(h)} \odot\linop \left(\dot{\mat{K}}^{(h+1)} \cdots \odot\linop\left(\id \odot \dot{\mat{K}}^{(L)}\right)\cdots\right)
}.
\end{align*}

\subsection{Derivation}
We first compute the derivative of the prediction with respect to one single filter.
\begin{align*}
\frac{\partial f(\params,\vect{x})}{\partial \mat{W}_{(\alpha),(\beta)}^{(h)}} = &\left\langle\frac{\partial f(\params,\vect{x})}{\partial \vect{x}_{(\beta)}}, \frac{\partial \vect{x}_{(\beta)}^{(h)}}{\partial \vect{W}_{(\alpha),(\beta)}^{(h)}}\right\rangle \\
= & \sum_{(i,j) \in [\nnw] \times [\nnh]} \left\langle\frac{\partial f(\params,\vect{x})}{[\vect{x}_{(\beta)}^{(h)}]_{ij}},\frac{\partial [\vect{x}_{(\beta)}^{(h)}]_{ij}}{\partial \mat{W}_{(\alpha),(\beta)}^{(h)}}\right\rangle \\
= & \sum_{(i,j) \in [\nnw] \times [\nnh]} \frac{\partial f(\params,\vect{x})}{[\vect{x}_{(\beta)}^{(h)}]_{ij}} \sqrt{\frac{c_{\sigma}}{\nnc^{(h)}q^2}}\reluder{\left[
\tilde{\vect{x}}_{(\beta)}^{(h)}
\right]_{ij}}\phi_{ij}(\vect{x}_{(\alpha)}^{(h-1)}).
\end{align*}

With this expression, we proceed to we compute the inner product between gradients with respect to the $h$-th layer matrix 
\begin{align}
&\sum_{\alpha=1}^{\nnc^{(h-1)}}\sum_{\beta=1}^{\nnc^{(h)}}\left\langle 
\frac{\partial f(\params,\vect{x})}{\partial \mat{W}_{(\alpha),(\beta)}^{(h)}}, \frac{\partial f(\params,\vect{x}')}{\partial \mat{W}_{(\alpha),(\beta)}^{(h)}} \right\rangle \label{eqn:grad_prod}\\
= &\sum_{(i,j,i',j') \in [\nnw] \times [\nnh]\times [\nnw]\times [\nnh] }\frac{c_{\sigma}}{\nnc^{(h)}q^2}\sum_{\beta=1}^{\nnc^{(h)}}
\left(
\frac{\partial f(\params,\vect{x})}{\partial [\vect{x}_{(\beta)}^{(h)}]_{ij}}\cdot\frac{\partial f(\params,\vect{x}')}{\partial [\vect{x}'^{(h)}_{(\beta)}]_{i'j'}}
\right)
\left(
\reluder{\left[\tilde{\vect{x}}^{(h)}_{(\beta)}\right]_{ij}}\reluder{\left[\tilde{\vect{x}}'^{(h)}_{(\beta)}\right]_{i'j'}}
\right) \nonumber\\
&\cdot\left(\sum_{\alpha=1}^{\nnc^{(h-1)}}\left\langle\phi_{ij}(\vect{x}_{(\alpha)}^{(h-1)}),\phi_{i'j'}(\vect{x}'^{(h-1)}_{(\alpha)})\right\rangle\right).\nonumber
\end{align}

Similar to our derivation to NTK, we can use the following approximation\begin{align*}
\left(\sum_{\alpha=1}^{\nnc^{(h-1)}}\left\langle\phi_{ij}(\vect{x}_{(\alpha)}^{(h-1)}),\phi_{i'j'}(\vect{x}'^{(h-1)}_{(\alpha)})\right\rangle\right) \approx
\tr\left(\left[\mat{K}^{(h-1)}\right]_{\indset_{ij,i'j'}}\right) = \linop\left(\mat{K}^{(h-1)}\right).
\end{align*}

Thus it remains to show that $\forall (i,j,i',j') \in [\nnw] \times [\nnh]\times [\nnw]\times [\nnh] $, 
\begin{align*}
&\sum_{\beta=1}^{\nnc^{(h)}}\frac{c_{\sigma}}{\nnc^{(h)}q^2}
\left(
\frac{\partial f(\params,\vect{x})}{\partial [\vect{x}_{(\beta)}^{(h)}]_{ij}}\cdot\frac{\partial f(\params,\vect{x}')}{\partial [\vect{x}'^{(h)}_{(\beta)}]_{i'j'}}
\right)
\left(
\reluder{\left[\tilde{\vect{x}}^{(h)}_{(\beta)}\right]_{ij}}\reluder{\left[\tilde{\vect{x}}'^{(h)}_{(\beta)}\right]_{i'j'}}
\right)\\
\approx& 
\left[\linop \left(\dot{\mat{K}}^{(h+1)} \cdots \odot\linop\left(\id \odot \dot{\mat{K}}^{(L)}\right)\cdots\right)\odot \dot{\mat{K}}^{(h)} \right]_{i,j,i',j'}
\end{align*}

The key step of this derivation is the following approximation (Equation~\ref{eq:gaussian_conditioning2}), which assumes  for each $(i,j,i',j')$, $
\frac{\partial f(\params,\vect{x})}{\partial [\vect{x}_{(\beta)}^{(h)}]_{ij}}\cdot\frac{\partial f(\params,\vect{x}')}{\partial [\vect{x}'^{(h)}_{(\beta)}]_{i'j'}}
$ and $
\reluder{\left[\tilde{\vect{x}}^{(h)}_{(\beta)}\right]_{ij}}\reluder{\left[\tilde{\vect{x}}'^{(h)}_{(\beta)}\right]_{i'j'}}
$ are independent. This is used and made rigorous for ReLU activation and fully-connected networks in the proof of Theorem~\ref{thm:ntk_init}. \cite{yang2019scaling} gave a rigorous statement of this approximation in an asymptotic way for CNNs.

\begin{equation}
\begin{split}\label{eq:gaussian_conditioning2}
&\frac{1}{\nnc^{(h)}} \sum_{\beta=1}^{\nnc^{(h)}} \left(
\frac{\partial f(\params,\vect{x})}{\partial [\vect{x}_{(\beta)}^{(h)}]_{ij}}\cdot\frac{\partial f(\params,\vect{x}')}{\partial [\vect{x}'^{(h)}_{(\beta)}]_{i'j'}}
\right)
\left(
\reluder{\left[\tilde{\vect{x}}^{(h)}_{(\beta)}\right]_{ij}}\reluder{\left[\tilde{\vect{x}}'^{(h)}_{(\beta)}\right]_{i'j'}}
\right) \\
\approx& 
\left(\frac{1}{\nnc^{(h)}} \sum_{\beta=1}^{\nnc^{(h)}} 
\frac{\partial f(\params,\vect{x})}{\partial [\vect{x}_{(\beta)}^{(h)}]_{ij}}\cdot\frac{\partial f(\params,\vect{x}')}{\partial [\vect{x}'^{(h)}_{(\beta)}]_{i'j'}}
\right)
\left(\frac{1}{\nnc^{(h)}} \sum_{\beta=1}^{\nnc^{(h)}}
\reluder{\left[\tilde{\vect{x}}^{(h)}_{(\beta)}\right]_{ij}}\reluder{\left[\tilde{\vect{x}}'^{(h)}_{(\beta)}\right]_{i'j'}}
\right) 
\end{split}
\end{equation}

Note that 
\begin{align*}
	\frac{c_{\sigma}}{\nnc^{(h)}q^2}\sum_{\beta=1}^{\nnc^{(h)}}\reluder{\left[\tilde{\vect{x}}^{(h)}_{(\beta)}\right]_{ij}}\reluder{\left[\tilde{\vect{x}}'^{(h)}_{(\beta)}\right]_{i'j'}} \approx \left[\dot{\mat{K}}^{(h)}\left(\vect{x},\vect{x}'\right)\right]_{ij,i'j'},
\end{align*}

the derivation is complete once we show 

\begin{equation}\label{eq:defi_G}\mat{G}^{(h)}(\vect{x},\vect{x'},\vect{\theta}):=\frac{1}{\nnc^{(h)}} \sum_{\beta=1}^{\nnc^{(h)}} 
\frac{\partial f(\params,\vect{x})}{\partial \vect{x}_{(\beta)}^{(h)}}\otimes \frac{\partial f(\params,\vect{x}')}{\partial \vect{x}'^{(h)}_{(\beta)}}
\approx 
\linop \left(\dot{\mat{K}}^{(h+1)} \cdots \odot\linop\left(\id \odot \dot{\mat{K}}^{(L)}\right)\cdots\right).\end{equation}

Now, we tackle the term $\left(
\frac{\partial f(\params,\vect{x})}{\partial [\vect{x}_{(\beta)}^{(h)}]_{ij}}\cdot\frac{\partial f(\params,\vect{x}')}{\partial [\vect{x}'^{(h)}_{(\beta)}]_{i'j'}}
\right)$.
Notice that \begin{align*}
	\frac{\partial f(\params,\vect{x})}{\partial \left[\vect{x}_{(\beta)}^{(h)}\right]_{ij}} = & \sum_{(k,\ell) \in [\nnw]\times [\nnh] } 
	\frac{\partial f(\params,\vect{x})}{\partial \left[\vect{x}_{(\gamma
			)}^{(h+1)}\right]_{k\ell}} 
	\frac{\partial \left[\vect{x}_{(\gamma)}^{(h+1)}\right]_{k\ell}}{\partial \left[\vect{x}_{(\beta)}^{(h)}\right]_{ij}}.
\end{align*}
and for $\gamma \in [\nnc^{(h+1)}]$ and $(k,\ell) \in [\nnw] \times [\nnh]$\begin{align*}
\frac{\partial \left[\vect{x}_{(\gamma)}^{(h+1)}\right]_{k\ell}}{\partial \left[\vect{x}_{(\beta)}^{(h)}\right]_{ij}} = \begin{cases}
\sqrt{\frac{c_{\sigma}}{\nnc^{(h+1)q^2}}}\reluder{\left[\tilde{\vect{x}}_{(\gamma)}^{(h+1)}\right]_{k\ell}}  \left[\mat{W}_{(\beta),(\gamma)}^{(h+1)}\right]_{i-k+q-1,j-\ell+q-1}&\text{ if } (i,j) \in \indset_{k\ell} \\
0 &\text{ otherwise }
\end{cases}.
\end{align*}
We then have \begin{align}
&\left[\mat{G}^{(h)}(\vect{x},\vect{x'},\vect{\theta})\right]_{ij,i'j'} = \frac{1}{\nnc^{(h)}} \sum_{\beta=1}^{\nnc^{(h)}} \frac{\partial f(\params,\vect{x})}{\partial \left[\vect{x}_{(\beta)}^{(h)}\right]_{ij}} \frac{\partial f(\params,\vect{x}')}{\partial \left[\vect{x'}^{(h)}_{(\beta)}\right]_{i'j'}} \nonumber \\ 
=&\sum_{k,\ell,k',\ell'} \frac{c_{\sigma}}{\nnc^{(h+1)}q^2}\sum_{\gamma=1}^{\nnc^{(h+1)}}\left(
	\frac{\partial f(\params,\vect{x})}{\partial \left[\vect{x}_{(\gamma
			)}^{(h+1)}\right]_{k\ell}} 
	\frac{\partial f(\params,\vect{x}')}{\partial \left[\vect{x}_{(\gamma
			)}^{('h+1)}\right]_{k'\ell'}}\right) \left(\reluder{\left[\tilde{\vect{x}}^{(h+1)}_{(\gamma)}\right]_{k\ell}}\reluder{\left[\tilde{\vect{x}'}^{(h+1)}_{(\gamma)}\right]_{k'\ell'}}\right)\nonumber\\
	& \cdot \frac{1}{\nnc^{(h)}} \sum_{\beta=1}^{\nnc^{(h)}} \indict\left\{(i,j,i',j')\in \indset_{k\ell,k'\ell'}\right\} \left[\mat{W}_{(\beta),(\gamma)}^{(h+1)}\right]_{i-k+q-1,j-\ell+q-1}\left[\mat{W}_{(\beta),(\gamma)}^{(h+1)}\right]_{i'-k'+q-1,j'-\ell'+q-1}\nonumber\\
\approx & \sum_{k,\ell,k',\ell'} \frac{c_{\sigma}}{\nnc^{(h+1)}q^2}\sum_{\gamma=1}^{\nnc^{(h+1)}}\left(
\frac{\partial f(\params,\vect{x})}{\partial \left[\vect{x}_{(\gamma
		)}^{(h+1)}\right]_{k\ell}} 
\frac{\partial f(\params,\vect{x}')}{\partial \left[\vect{x}_{(\gamma
		)}^{('h+1)}\right]_{k'\ell'}}\right) \left(\reluder{\left[\tilde{\vect{x}}^{(h+1)}_{(\gamma)}\right]_{k\ell}}\reluder{\left[\tilde{\vect{x}'}^{(h+1)}_{(\gamma)}\right]_{k'\ell'}}\right)\nonumber\\
& \cdot \indict\left\{(i,j,i',j')\in \indset_{k\ell,k'\ell'}, i-k=i'-k',j-\ell=j'-\ell'\right\}\nonumber\\
\approx & \sum_{k,\ell,k',\ell'} \left(\frac{1}{\nnc^{(h+1)}}\sum_{\gamma=1}^{\nnc^{(h+1)}}
\frac{\partial f(\params,\vect{x})}{\partial \left[\vect{x}_{(\gamma
		)}^{(h+1)}\right]_{k\ell}} 
\frac{\partial f(\params,\vect{x}')}{\partial \left[\vect{x}_{(\gamma
		)}^{('h+1)}\right]_{k'\ell'}}\right) \left(\frac{c_\sigma}{q^2\nnc^{(h+1)}}\sum_{\gamma=1}^{{\nnc^{(h+1)}}}\reluder{\left[\tilde{\vect{x}}^{(h+1)}_{(\gamma)}\right]_{k\ell}}\reluder{\left[\tilde{\vect{x}'}^{(h+1)}_{(\gamma)}\right]_{k'\ell'}}\right)\nonumber\\
& \cdot \indict\left\{(i,j,i',j')\in \indset_{k\ell,k'\ell'}, i-k=i'-k',j-\ell=j'-\ell'\right\} \nonumber\\
\approx & \sum_{k,\ell,k',\ell'} \left(\frac{1}{\nnc^{(h+1)}}\sum_{\gamma=1}^{\nnc^{(h+1)}}
\frac{\partial f(\params,\vect{x})}{\partial \left[\vect{x}_{(\gamma
		)}^{(h+1)}\right]_{k\ell}} 
\frac{\partial f(\params,\vect{x}')}{\partial \left[\vect{x}_{(\gamma
		)}^{('h+1)}\right]_{k'\ell'}}\right) \left[\dot{\mat{K}}^{(h+1)}\left(\vect{x},\vect{x}'\right)\right]_{\ell k,\ell'k'}\nonumber\\
& \cdot \indict\left\{(i,j,i',j')\in \indset_{k\ell,k'\ell'}, i-k=i'-k',j-\ell=j'-\ell'\right\}\nonumber\\
\approx & \tr\left( \left[ \mat{G}^{(h+1)}(\vect{x},\vect{x'},\vect{\theta}) \odot\dot{\mat{K}}^{(h+1)}\left(\vect{x},\vect{x}'\right)\right]_{D_{ij,i'j'}}\right)
 \label{eqn:grad_fxh}
\end{align}
where the first approximation is due to our initialization of $\mat{W}^{(h+1)}$. In other words, we've shown  

\begin{equation}\label{eq:recursion_G} \mat{G}^{(h)}(\vect{x},\vect{x'},\vect{\theta}) = \linop\left(\mat{G}^{(h+1)}(\vect{x},\vect{x'},\vect{\theta})\odot \dot{\mat{K}}^{(h+1)}\left(\vect{x},\vect{x}'\right)\right).\end{equation}

Since we use a fully-connected weight matrix as the last layer, we have $ \mat{G}^{(L)}(\vect{x},\vect{x'},\vect{\theta}) \approx \id$.

Thus by induction with Equation~\ref{eq:recursion_G}, we have derived Equation~\ref{eq:defi_G}, which completes the derivation of CNTK.

For the derivation of CNTK-GAP, the only difference is due to the global average pooling layer(GAP), $ \mat{G}^{(L)}(\vect{x},\vect{x'},\vect{\theta}) \approx \frac{1}{\nnh^2\nnw^2} \bm{1}\otimes\bm{1}$, where $\bm{1}\otimes\bm{1}\in  \mathbb{R}^{\nnw\times\nnh\times\nnw\times\nnh} $ is the all one tensor.

\section{Formula of CNTK with Global Average Pooling}
\label{sec:cntk_gap}
In this section we define CNN with global average pooling considered in this paper and its corresponding CNTK formula.

\paragraph{CNN definition.}
\begin{itemize*}
	\item Let  $\vect{x} = \vect{x}^{(0)} \in \mathbb{R}^{\nnw\times \nnh \times \nnc^{(0)}}$ be the input image and  $\nnc^{(0)}$ is the number of initial channels.
	\item For $h=1,\ldots,L$, $\beta = 1,\ldots,\nnc^{(h)}$, the intermediate outputs are defined as \begin{align*}
	\tilde{\vect{x}}_{(\beta)}^{(h)} = \sum_{\alpha=1}^{\nnc^{(h-1)}} \mat{W}_{(\alpha),(\beta)}^{(h)} \conv \vect{x}_{(\alpha)}^{(h-1)} ,\quad
	\vect{x}^{(h)}_{(\beta)} = \sqrt{\frac{c_{\sigma}}{\nnc^{(h) \times q^{(h)} \times q^{(h)}}}}\act{\tilde{\vect{x}}_{(\beta)}^{(h)}}.
	\end{align*}
	\item The final output is defined as \begin{align*}
	f(\params,\vect{x}) =\sum_{\alpha=1}^{\nnc^{(L)}}  W_{(\alpha)}^{(L + 1)}\left(  \frac{1}{\nnw \nnh} \sum_{(i,j) \in [\nnw] \times [\nnh]}\left[\vect{x}_{(\alpha)}^{(L)}\right]_{ij}\right).
	\end{align*}
	where $W_{(\alpha)}^{(L + 1)} \in \mathbb{R}$ is a scalar with Gaussian initialization.
\end{itemize*}

Besides using global average pooling, another modification is that we do not train the first and the layer.
This is inspired by \cite{du2018algorithmic} in which authors  showed that if one applies gradient flow, then at any training time $t$, the difference between the squared Frobenius norm of the weight matrix at time $t$ and that at initialization is same for all layers.
However, note that $\mat{W}^{(1)}$ and $\mat{W}^{(L+1)}$ are special because they are smaller matrices compared with other intermediate weight matrices, so \emph{relatively}, these two weight matrices change more than the intermediate matrices during the training process, and this may dramatically change the kernel.
Therefore, we choose to fix $\mat{W}^{(1)}$ and $\mat{W}^{(L+1)}$ to the make over-parameterization theory closer to practice.


\paragraph{CNTK formula.}
We let $\vect{x},\vect{x}'$ be two input images.
Note because CNN with global average pooling and vanilla CNN shares the same architecture except the last layer, $\mat{\Sigma}^{(h)}(\vect{x},\vect{x}')$, $\dot{\mat{\Sigma}}^{(h)}(\vect{x},\vect{x}')$ and $\mat{K}^{(h)}(\vect{x},\vect{x}')$ are the same for these two architectures.
the only difference is in calculating the final kernel value.
To compute the final kernel value, we use the following procedure.
\begin{enumerate*}
\item First, we define  $\mat{\Theta}^{(0)}(\vect{x},\vect{x}') = \vect{0}$.
Note this is different from CNTK for vanilla CNN which uses $\mat{\Sigma}^{(0)}$ as the initial value because we do not train the first layer.
\item For $h=1,\ldots,L - 1$ and $(i,j,i',j') \in [\nnw] \times [\nnh] \times [\nnw] \times [\nnh]$, we define
\begin{align*}
\left[\mat{\Theta}^{(h)}(\vect{x},\vect{x}')\right]_{ij,i'j'} = \tr\left(\left[\dot{\mat{K}}^{(h)}(\vect{x},\vect{x}')\odot\mat{\Theta}^{(h-1)}(\vect{x},\vect{x}')+\mat{K}^{(h)}(\vect{x},\vect{x}')\right]_{D_{ij,i'j'}}\right).
\end{align*}
\item For  $h=L$, we define \quad 
$\mat{\Theta}^{(L)}(\vect{x},\vect{x}') = \dot{\mat{K}}^{(L)}(\vect{x},\vect{x}')\odot\mat{\Theta}^{(L-1)}(\vect{x},\vect{x}')$.

	\item Lastly, the final kernel value is defined as 	 \[
		\frac{1}{\nnw^2 \nnh^2} \sum_{(i,j,i',j') \in [\nnw] \times [\nnh] \times [\nnw] \times [\nnh]} \left[\mat{\Theta}^{(L)}(\vect{x},\vect{x}')\right]_{ij,i'j'} .
	\]
	Note that we ignore $\mat{K}^{(L)}$ comparing with the CNTK of CNN.
	This is because we do not train the last layer.
	The other difference is we calculate the mean over all entries, instead of calculating the summation over the diagonal ones.
	This is because we use global average pooling so the cross-variances between every two patches will contribute to the kernel.
\end{enumerate*}

\section{Fast Computation for ReLU-Activated CNTK}
\label{sec:fast}


In this section we present our approach to compute CNTK \emph{exactly}.
Notably, most computation required by our new approach can be described as entry-wise operations over matrices and tensors, which allows efficient implementations on GPUs.

Following the formulas in Sections~\ref{sec:kernel} and~\ref{sec:cntk_gap}, the trickiest part is computing the expectation of the post-activation output, i.e., Equations~\eqref{eqn:vanila_cnn_exp} and \eqref{eqn:vanila_cnn_exp_d}.
These two expectations depend on (the same) $2 \times 2$ matrices $\left[\twotwomat^{(h)}(\vect{x},\vect{x}')\right]_{ij,i'j'}$.
To obtain faster implementations, our key observation is that if the diagonal entries of $\left[\twotwomat^{(h)}(\vect{x},\vect{x}')\right]_{ij,i'j'}$ are all ones and the activation function is ReLU, 
there are closed-form formulas for the the corresponding expectations.
To see this, let us suppose for now that $\mat{\Lambda} = \begin{pmatrix}
1 & \lambda \\
\lambda & 1
\end{pmatrix}$ for some $\abs{\lambda} \le 1$.
When the activation function $\relu{\cdot}$ is ReLU, one can show that
\begin{equation}
\expect_{(u,v) \sim \gauss\left(\vect{0},\mat{\Lambda}\right) }\left[\relu{u}\relu{v}\right] = \frac{\lambda (\pi - \arccos(\lambda)) + \sqrt{1-\lambda^2}}{2\pi} \label{eqn:cov_formula}
\end{equation}
and
\begin{equation}
\expect_{(u,v) \sim \gauss\left(\vect{0},\mat{\Lambda}\right) }\left[\deract{u}\deract{v}\right] = \frac{\pi-\arccos\left(\lambda\right)}{2\pi}. \label{eqn:grad_formula}
\end{equation}

Now we let
$$
\mat{A}^{(h)} = \begin{pmatrix}
\mat{\Sigma}^{(h-1)}(\vect{x},\vect{x}) &\mat{\Sigma}^{(h-1)}(\vect{x},\vect{x}') \\
\mat{\Sigma}^{(h-1)}\left(\vect{x}',\vect{x}\right)& \mat{\Sigma}^{(h-1)}\left(\vect{x}',\vect{x}'\right)
\end{pmatrix} \in \mathbb{R}^{2\nnw\nnh \times 2 \nnw \nnh}.
$$
Here, we interpret $\mat{\Sigma}^{(h-1)}(\vect{x},\vect{x})$, $\mat{\Sigma}^{(h-1)}(\vect{x},\vect{x}')$, $\mat{\Sigma}^{(h-1)}\left(\vect{x}',\vect{x}\right)$ and $\mat{\Sigma}^{(h-1)}\left(\vect{x}',\vect{x}'\right)$ as matrices of size $\nnw\nnh \times \nnw\nnh$.
If the diagonal entries of $A^{(h)}$ are all ones, then the diagonal entries of $\left[\twotwomat^{(h)}(\vect{x},\vect{x}')\right]_{ij,i'j'}$ are all ones for all possible $(i, j, i', j') \in [\nnw] \times [\nnh] \times [\nnw] \times [\nnh]$, in which case we can calculate $\mat{K}^{(h)}(\vect{x},\vect{x}')$ and $\dot{\mat{K}}^{(h)}(\vect{x},\vect{x}')$ by simply applying the closed-form formulas described in \eqref{eqn:cov_formula} and \eqref{eqn:grad_formula} on $A^{(h)}$.

However, in general, the diagonal entries of $\mat{A}^{(h)}$ are not always all ones, in which case we resort to the homogeneity of the ReLU activation function.
Suppose
$
\mat{\Lambda} = \begin{pmatrix}
1 & \lambda \\
\lambda & 1
\end{pmatrix}$
for some $\abs{\lambda} \le 1$,
and
$\mat{D} = \begin{pmatrix}
c_1 & 0 \\
0 & c_2
\end{pmatrix}$
for some $c_1, c_2 \ge 0$, then one can show that
\begin{equation}
\expect_{(u,v) \sim \gauss\left(\vect{0},\mat{D \Lambda D} \right) }\left[\relu{u}\relu{v}\right] = \frac{\lambda (\pi - \arccos(\lambda)) + \sqrt{1-\lambda^2}}{2\pi} \cdot c_1 c_2\label{eqn:cov_formula_scale}
\end{equation}
and
\begin{equation}
\expect_{(u,v) \sim \gauss\left(\vect{0},\mat{D \Lambda D}\right) }\left[\deract{u}\deract{v}\right] = \frac{\pi-\arccos\left(\lambda\right)}{2\pi}. \label{eqn:grad_formula_scale}
\end{equation}
Inspired by this, our final approach is described as follows.

\begin{enumerate}
	\item Let $\mat{D} = \begin{pmatrix}
	\mat{D_x} & \mat{0} \\
	\mat{0} & \mat{D_{x'}}
	\end{pmatrix}$, where $\mat{D_x}$ and $\mat{D_{x'}}$  
	are diagonal matrices whose diagonal entries are square roots of the diagonal entries of 
	$\mat{\Sigma}^{(h-1)}(\vect{x},\vect{x})$ and 
	$\mat{\Sigma}^{(h-1)}(\vect{x'},\vect{x'})$, respectively. 
	\item Applying Equations \eqref{eqn:cov_formula_scale} and \eqref{eqn:grad_formula_scale} on $\mat{A}^{(h)} = \mat{D} \mat{\Lambda}^{(h)} \mat{D}$, where the diagonal entries of $\mat{\Lambda}^{(h)}$ are all ones. 
\end{enumerate}

Notice that the implementation above requires us to store the whole $\mat{A}^{(h)}$ matrix, which has size $2\nnw\nnh \times 2 \nnw \nnh$. 
To further optimize the efficiency, we notice that to implement the approach described above, we only need to store the diagonal entries of $\mat{\Sigma}^{(h-1)}(\vect{x},\vect{x})$ and $\mat{\Sigma}^{(h-1)}(\vect{x'},\vect{x'})$, together with the matrix $\mat{\Sigma}^{(h-1)}(\vect{x},\vect{x}') $, which has size $\nnw\nnh \times \nnw \nnh$.

\end{document}